\documentclass[11 pt]{article}
\usepackage{style}
\usetikzlibrary{arrows.meta}

\title{\LARGE\bfseries A Minimal-Assumption Analysis of Q-Learning with Time-Varying Policies}
\author{Phalguni Nanda and Zaiwei Chen\\
{\small \textit{
Edwardson School of Industrial Engineering, Purdue University}}\\
{\small \href{mailto:nanda14@purdue.edu}{\texttt{nanda14@purdue.edu}}, \href{mailto:chen5252@purdue.edu}{\texttt{chen5252@purdue.edu}}}}
\date{\vspace{-10 mm}}

\begin{document}
\maketitle
\begin{abstract}
In this work, we present the first \emph{finite-time analysis} of Q-learning with \emph{time-varying learning policies} (i.e., on-policy sampling) for discounted Markov decision processes under \emph{minimal assumptions}, requiring only the existence of a policy that induces an \emph{irreducible} Markov chain over the state space.
We establish a last-iterate convergence rate for $\mathbb{E}[\|Q_k - Q^*\|_\infty^2]$, which implies a sample complexity of order $\mathcal{O}(1/\xi^2)$ for achieving $\mathbb{E}[\|Q_k - Q^*\|_\infty]\le \xi$. This rate matches that of off-policy Q-learning, but with a worse dependence on exploration-related parameters. We also derive an explicit finite-time rate for $\mathbb{E}[\|Q^{\pi_k} - Q^*\|_\infty^2]$, where $\pi_k$ denotes the learning policy at iteration $k$. Together, these results highlight the exploration--exploitation trade-off in on-policy Q-learning. While exploration is weaker than in off-policy methods, on-policy learning enjoys an exploitation advantage since the learning policy itself converges to an optimal one. Numerical experiments corroborate our theoretical findings.

From a technical perspective, the combination of rapidly time-varying learning policies, which induce time-inhomogeneous Markovian noise, and minimal exploration assumptions presents significant analytical challenges. To address these challenges, we develop a Poisson-equation-based decomposition of the Markovian noise associated with a \emph{lazy} transition matrix, separating it into a martingale-difference term and residual terms. We then control the residual terms through a sensitivity analysis of the Poisson equation solution with respect to both the Q-function estimate and the learning policy. These techniques may facilitate the analysis of other reinforcement learning algorithms with rapidly time-varying learning policies, such as single-timescale actor--critic methods and learning-in-games algorithms.
\end{abstract}

\section{Introduction}\label{sec:intro}
Reinforcement learning (RL) provides a principled framework for sequential decision-making under uncertainty \cite{sutton2018reinforcement}, with broad applications in game playing \cite{silver2017mastering}, robotics \cite{levine2016end}, recommendation systems \cite{afsar2022reinforcement}, and large language models \cite{ouyang2022training}. Among the diverse algorithmic approaches in RL, Q-learning \cite{watkins1992q} stands out as one of the most fundamental and widely studied methods, owing to its simplicity, its natural interpretation as solving the Bellman equation via stochastic approximation \cite{robbins1951stochastic}, and its empirical success. In particular, a notable variant of Q-learning, known as the deep Q-network (DQN) \cite{mnih2015human}, achieved human-level performance on Atari games, which is widely regarded as a milestone in the modern development of RL.

Due to the popularity of Q-learning, substantial efforts have been devoted to establishing its theoretical foundations. As discussed, Q-learning can be viewed as a stochastic approximation algorithm for solving the Bellman equation \cite{tsitsiklis1994asynchronous,borkar2000ode}. The randomness arises from the agent’s interaction with the environment under a learning policy, during which it collects potentially noisy samples of state transitions and rewards. From this perspective, the literature has developed a broad range of theoretical results to deepen our understanding of Q-learning. Early work established asymptotic convergence \cite{tsitsiklis1994asynchronous,borkar2000ode,borkar2009stochastic,szepesvari1998asymptotic}, while more recent studies have provided non-asymptotic guarantees, including finite-time mean-square error bounds \cite{beck2012error,beck2013improved,chen2024lyapunov,lee2024final,wainwright2019stochastic,wainwright2019variance} and high-probability bounds \cite{even2003learning,qu2020finite,li2020sample,li2023statistical}. In particular, it has been shown that variance-reduced Q-learning \cite{wainwright2019variance,li2020sample} almost achieves the minimax lower bound \cite{azar2012dynamic}.

For most existing results—especially those concerning non-asymptotic analysis \cite{even2003learning, beck2012error, qu2020finite, beck2013improved,chen2024lyapunov,lee2024final}—the learning policy is typically assumed to be stationary, with a few exceptions \cite{jin2018q,liu2025linear}, which we discuss in more detail in Section~\ref{subsec:literature}. In practice, however, Q-learning is almost always implemented with time-varying policies, such as $\epsilon$-greedy \cite{mnih2015human}, Boltzmann (softmax) exploration \cite{leslie2005individual,liu2025linear,meyn2024projected}, or combinations and variants of these \cite{van2016deep,tokic2011value,haarnoja2017reinforcement}. This gap between theoretical assumptions and practical implementations motivates us to develop new theoretical insights into the non-asymptotic behavior of Q-learning under time-varying policies, with the aim of better guiding its use in modern applications.

From a stochastic approximation viewpoint, the time-varying nature of the learning policy implies that the noise sequence in Q-learning with on-policy sampling\footnote{Throughout this paper, we refer to Q-learning with time-varying learning policies (such as $\epsilon$-greedy, softmax, or their combinations and variants) as \emph{Q-learning with on-policy sampling}, in contrast to off-policy Q-learning where the learning policy is stationary.} forms a rapidly varying time-inhomogeneous Markov chain, which poses a fundamental analytical challenge. Existing analyses of RL algorithms under stationary learning policies typically rely on Markov chain mixing arguments \cite{srikant2019finite,bhandari2018finite}. However, when the policy is time-varying, it is unclear how to apply such techniques without imposing strong assumptions—such as requiring every policy encountered by the algorithm's trajectory to induce a uniformly ergodic Markov chain with mixing rates uniformly bounded from above and stationary distributions uniformly bounded away from zero \cite{liu2025linear,zou2019finite}. Moreover, under such assumptions, one cannot theoretically capture the exploration–exploitation trade-off inherent in Q-learning with on-policy sampling. We return to this issue in greater detail in Section~\ref{sec:main results}.
\begin{center}
    \textit{In this paper, we address these challenges by providing a principled non-asymptotic study of}\\
    \textit{Q-learning with time-varying learning policies under minimal assumptions.$\qquad\qquad\qquad\;\;\;$}
\end{center}
Specifically, we assume only the existence of a policy that induces an \emph{irreducible} Markov chain over states. This policy need not be encountered along the algorithm’s trajectory and can therefore be viewed as a mild, algorithm-independent assumption on the underlying stochastic model. Under this assumption, we establish last-iterate convergence rates for on-policy Q-learning that quantitatively characterize the exploration--exploitation trade-off. These results are further validated through numerical simulations. We next summarize the main contributions of this work in more detail.

\subsection{Main Contributions}\label{subsec:contributions}
For infinite-horizon discounted MDPs with finite state--action spaces, we study the classical tabular Q-learning algorithm implemented with a learning policy that is a convex combination (with parameter $\epsilon \in (0,1)$) of a uniform policy and a softmax policy (with temperature $\tau>0$) induced by the current Q-function estimate.

\begin{itemize}
    \item \textbf{Finite-time analysis under minimal assumptions.} We establish a last-iterate convergence rate for $\mathbb{E}[\|Q_k - Q^*\|_\infty^2]$, which implies a sample complexity of order $\widetilde{\mathcal{O}}(\xi^{-2})$ for achieving $\mathbb{E}[\|Q_k - Q^*\|_\infty] \le \xi$. We further characterize the dependence of this rate on the exploration parameters $\epsilon$ and $\tau$, as well as on other intrinsic quantities that capture the fundamental exploration properties of the underlying MDP. In addition, for the learning policy $\pi_k$ used at iteration $k$, we derive an explicit convergence rate for $\mathbb{E}[\|Q^{\pi_k} - Q^*\|_\infty^2]$. Together, these results quantitatively demonstrate that on-policy Q-learning exhibits weaker exploration than its off-policy counterpart but enjoys a distinct exploitation advantage, as the learning policy itself converges to an optimal one rather than remaining stationary. Importantly, our analysis is developed under the assumption that there exists a policy inducing an \emph{irreducible} Markov chain. This assumption is the weakest among those used in the literature, and we further show that it is necessary even for the asymptotic convergence of Q-learning. Our theoretical findings are corroborated by numerical simulations.

\vspace{1 em}

    \item \textbf{Handling rapidly varying time-inhomogeneous Markovian noise.} The combination of minimal assumptions (existence of a policy that induces an irreducible Markov chain) and the time-varying nature of the learning policy presents unique technical challenges that, to the best of our knowledge, have not been addressed before. Inspired by \cite{haque2024stochastic,chandak2022concentration}, we tackle this challenge by developing an approach based on using the \emph{Poisson equation} to decompose the Markov chain into a martingale-difference sequence and residual terms. To handle time inhomogeneity, we perform a sensitivity analysis and establish an almost-Lipschitz continuity property of the Poisson equation solution with respect to both the transition matrix and the forcing function (cf.\ Proposition~\ref{prop:MC}). To address the minimal assumption challenge, our analysis is built upon the \emph{lazy chain} associated with the original transition kernel. More details are presented in Section~\ref{subsec:Poisson}. The proposed approach for handling time-inhomogeneous Markovian noise is of independent interest and can potentially be applied to other RL algorithms, such as single-timescale actor--critic methods and multi-agent settings where learning policies are often rapidly time-varying.
\end{itemize}

\subsection{Related Literature}\label{subsec:literature}
The most closely related works are those that study Q-learning, SARSA, and general stochastic approximation algorithms with time-inhomogeneous Markovian noise. However, existing studies either do not employ on-policy sampling or require strong assumptions. We next discuss these works in more detail.

\textbf{Q-learning.} The Q-learning method was first introduced in \cite{watkins1992q} and later proven to converge asymptotically to the optimal Q-function \cite{tsitsiklis1994asynchronous,jaakkola1994convergence,borkar2000ode,lee2019unified}. 
Beyond asymptotic guarantees, non-asymptotic analyses have established an $\mathcal{O}(1/k)$ convergence rate of $\|Q_k - Q^*\|_\infty^2$ (both in expectation and with high probability), under the assumption that the learning policy is stationary and induces an irreducible and aperiodic Markov chain over the states
\cite{even2003learning,qu2020finite,beck2012error,beck2013improved,chen2024lyapunov,li2020sample,wainwright2019stochastic,wainwright2019variance,lee2024final}. 
In addition, several variants of Q-learning have been proposed and analyzed, including Zap Q-learning \cite{devraj2017zap}, Q-learning with variance reduction \cite{wainwright2019variance,xia2024instance}, Q-learning with Polyak--Ruppert averaging \cite{li2023statistical,zhang2024constant}, Q-learning with function approximation \cite{melo2008analysis,chen2022target,meyn2024projected}, federated Q-learning \cite{woo2025blessing,khodadadian2022federated}, etc.

For Q-learning with on-policy sampling, existing results are far more limited and rely on strong assumptions about the set of all policies or all learning policies encountered along the algorithm’s trajectory. In particular, the analysis in \cite{chandak2022concentration} can, in principle, be extended to this setting, but it requires the existence of a uniform lower bound on the stationary distribution over states for all policy-induced Markov chains, and the resulting bounds (i) hold only for sufficiently large $k$ (e.g., $k \geq N$ for some $N$), (ii) depend on a random quantity $Q_N$, and (iii) involve implicit problem-dependent constants. More recently, \cite{liu2025linear} studied on-policy Q-learning with linear function approximation, with the tabular case as a special instance. However, their analysis assumes that every policy induces a uniformly ergodic Markov chain whose mixing rate is uniformly bounded away from $1$ and whose stationary distribution is uniformly bounded away from $0$. Moreover, the problem-dependent constants are implicit, and as a result, the bound cannot quantitatively capture the exploration–exploitation trade-off in on-policy Q-learning. A related but distinct line of research studies online (and offline) Q-learning, primarily in the episodic setting, where performance is measured in terms of regret; see \cite{jin2018q,yang2021q} and references therein. Since the problem formulations (episodic vs.\ infinite-horizon) and performance criteria (regret vs.\ last-iterate convergence) differ fundamentally, the corresponding results and analytical techniques are not directly comparable.

\textbf{SARSA.} A closely related algorithmic framework to Q-learning is SARSA, proposed in \cite{rummery1994line}. Similar to Q-learning with on-policy sampling, the learning policy in SARSA is time-varying. The key distinction is that SARSA updates the Q-function using the actual action chosen by the learning policy, whereas Q-learning relies on a virtual action that maximizes the current Q-function. The asymptotic convergence of SARSA was established in \cite{singh2000convergence}. For finite-time analysis, SARSA with linear function approximation has been studied in \cite{zou2019finite,zhang2023convergence}, which also covers the tabular case as a special instance. However, in addition to requiring strong assumptions (uniform ergodicity under all policies), both \cite{zou2019finite,zhang2023convergence} assume that the policy is Lipschitz with a sufficiently small Lipschitz constant. In contrast, \cite{singh2000convergence} showed that SARSA converges to the optimal Q-function only if the policy can be arbitrarily close to the greedy policy with respect to the Q-function. Consequently, the guarantees in \cite{zou2019finite,zhang2023convergence} do not ensure convergence to the optimal Q-function.

\textbf{Stochastic approximation with time-inhomogeneous Markovian noise.} Mathematically, Q-learning with on-policy sampling can be modeled as a stochastic approximation method \cite{robbins1951stochastic} for solving the Bellman equation, where the noise sequence forms a time-inhomogeneous Markov chain due to the learning policy being time-varying. While finite-time analyses of stochastic approximation have been extensively studied (see \cite{srikant2019finite,bhandari2018finite,chen2024lyapunov} and the references therein), results for the case of time-inhomogeneous Markovian noise are relatively rare, with notable exceptions in specific settings such as actor–critic algorithms \cite{wu2020finite,khodadadian2021finite,zhang2023convergence} and learning in games \cite{chen2023finite,chen2023twotimescale}.
However, these results all rely on a timescale separation assumption, namely that the transition kernel of the Markovian noise evolves much more slowly (either orderwise or by a large multiplicative factor) than the main iterate. 
As a result, the Markovian noise in these works is not rapidly changing, which stands in sharp contrast to Q-learning with on-policy sampling. 

The rest of the paper is organized as follows. Section~\ref{sec:background} reviews the background on RL and Q-learning with on-policy sampling. Section~\ref{sec:main results} presents the main results, including convergence rates for $\mathbb{E}[\|Q_k - Q^*\|_\infty^2]$ (Theorem~\ref{thm1}) and $\mathbb{E}[\|Q^{\pi_k} - Q^*\|_\infty^2]$ (Theorem~\ref{thm2}), with proofs given in Sections~\ref{sec:proof1} and~\ref{sec:proof2}, respectively, and technical lemmas deferred to the appendix. Numerical experiments are presented in Section~\ref{sec:numerical}, and conclusions are drawn in Section~\ref{sec:conclusion}.

\section{Background}\label{sec:background}
In this section, we introduce the mathematical model of RL and the Q-learning algorithm with time-varying learning policies.

\subsection{Reinforcement Learning}\label{subsec:RL}
Consider an infinite-horizon discounted MDP defined by a finite set of states $\mathcal{S}$, a finite set of actions $\mathcal{A}$, a transition kernel $\{p(s'\mid s,a)\mid s,s'\in\mathcal{S},\, a\in\mathcal{A}\}$, a reward function $\mathcal{R}:\mathcal{S}\times \mathcal{A}\to \mathbb{R}$, and a discount factor $\gamma\in (0,1)$. We assume, without loss of generality, that $|\mathcal{R}(s,a)|\leq 1$ for all $(s,a)$. At each time step $k \geq 0$, let $S_k$ denote the current state of the environment. The agent selects an action $A_k$ according to a policy $\pi:\mathcal{S} \to \Delta(\mathcal{A})$, receives a stage-wise reward $\mathcal{R}(S_k, A_k)$, and the environment transitions to a new state $S_{k+1} \sim p(\cdot \mid S_k, A_k)$. This process then repeats. Importantly, the parameters of the stochastic model (e.g., the transition kernel and the reward function) are unknown to the agent, who must learn by interacting with the environment.

The goal of the agent is to find a policy that maximizes the cumulative reward. Specifically, given a policy $\pi$, its quality is characterized by the Q-function $Q^\pi:\mathcal{S} \times \mathcal{A} \to \mathbb{R}$, defined as $ Q^\pi(s,a)
    = \mathbb{E}_\pi[\sum_{k=0}^\infty \gamma^k \mathcal{R}(S_k, A_k) \mid S_0 = s,\, A_0 = a ]$ for all $(s,a)$,
where $\mathbb{E}_\pi[\cdot]$ denotes the expectation under the policy $\pi$, that is, $A_k \sim \pi(\cdot \mid S_k)$ for all $k \geq 1$. Since we work with a finite MDP, the Q-function can be equivalently viewed as a vector in $\mathbb{R}^{|\mathcal{S}||\mathcal{A}|}$. 
Moreover, for notational simplicity, for any $x \in \mathbb{R}^{|\mathcal{S}||\mathcal{A}|}$ (where $x$ may represent either a Q-function or a policy), we denote by $x(s)$ the $|\mathcal{A}|$-dimensional vector whose $a$-th entry is $x(s,a)$. 
With the Q-function defined, a policy $\pi^*$ is said to be optimal if $Q^*(s,a) := Q^{\pi^*}(s,a) \geq Q^\pi(s,a)$ for all policies $\pi$ and all state-action pairs $(s,a)$. Although this formulation corresponds to a multi-objective optimization problem, it is well known that such an optimal policy always exists \cite{puterman2014markov}. 

The key to finding an optimal policy is the Bellman equation:
\begin{align}\label{eq:Bellman}
    \mathcal{H}(Q)=Q,
\end{align}
where $\mathcal{H}:\mathbb{R}^{|\mathcal{S}||\mathcal{A}|}\to \mathbb{R}^{|\mathcal{S}||\mathcal{A}|}$ is the Bellman operator defined as
\begin{align}\label{def:Bellman_H}
    [\mathcal{H}(Q)](s,a)=\mathcal{R}(s,a)+\gamma\sum_{s'}p(s'|s,a)\max_{a'}Q(s',a'),\quad \forall\,(s,a).
\end{align}
It has been shown in the literature that the Bellman equation \eqref{eq:Bellman} admits a unique solution—the optimal Q-function $Q^*$. Once $Q^*$ is known, an optimal policy $\pi^*$ can be obtained by choosing actions greedily with respect to $Q^*$ \cite{puterman2014markov,bertsekas1996neuro}.

To solve the Bellman equation \eqref{eq:Bellman}, note that $\mathcal{H}(\cdot)$ is a contraction mapping with respect to $\|\cdot\|_\infty$ \cite{puterman2014markov}. A natural approach is therefore to perform the fixed-point iteration $Q_{k+1} = \mathcal{H}(Q_k)$, also known as \emph{Q-value iteration}, which converges geometrically to $Q^*$ by the Banach fixed-point theorem \cite{banach1922operations}. While Q-value iteration is theoretically appealing, it is not implementable in RL because the transition kernel and reward function of the underlying MDP are unknown. This limitation motivates \emph{Q-learning} \cite{watkins1992q}, a data-driven stochastic approximation method for solving the Bellman equation, which we introduce next.

\subsection{Q-Learning with Time-Varying Learning Policies}\label{subsec:sarsa}
The Q-learning algorithm, first introduced in \cite{watkins1992q}, is summarized in Algorithm~\ref{algo:Q-learning}. At iteration $k$, the algorithm computes a learning policy $\pi_k$ from the current estimate $Q_k$ via a mapping $f(\cdot)$, which is discussed in detail below. The agent then samples a transition using $\pi_k$ and updates $Q_k$ via a stochastic approximation algorithm for solving the Bellman equation~\eqref{eq:Bellman}, where $\alpha$ denotes the stepsize (learning rate).

\begin{algorithm}[ht]\caption{Q-Learning with Time-Varying Learning Policies}\label{algo:Q-learning}
	\begin{algorithmic}[1]
		\STATE \textbf{Input:} Integer $K$, initialization $Q_0\in\mathbb{R}^{|\mathcal{S}||\mathcal{A}|}$ satisfying $\|Q_0\|_\infty\leq 1/(1-\gamma)$ and $S_0\in\mathcal{S}$.
		\FOR{$k=0,1,2,\cdots,K-1$}
        \STATE $\pi_k(\cdot\mid S_k)=[f(Q_k)](S_k,\cdot)$
        \STATE Take $A_{k}\sim \pi_k(\cdot\mid S_{k})$, receive $\mathcal{R}(S_k,A_k)$, and observe $S_{k+1}\sim p(\cdot| S_k,A_k)$
		\STATE Update the Q-function according to 
        \begin{align*}
            Q_{k+1}(s,a)=Q_k(s,a)+\alpha\mathds{1}_{\{(S_k,A_k)=(s,a)\}}\left(\mathcal{R}(S_k,A_k)+\gamma \max_{a'} Q_k(S_{k+1},a')-Q_k(S_k,A_k)\right)
        \end{align*}
        for all $(s,a)\in\mathcal{S}\times \mathcal{A}$.
		\ENDFOR
        \STATE \textbf{Output:} $\{Q_k\}_{0\leq k\leq K}$ and $\{\pi_k\}_{0\leq k\leq K}$
	\end{algorithmic}
\end{algorithm} 

As for the function $f(\cdot)$, when it is constant, that is, $f(Q)\equiv \pi_b$ for any $Q\in\mathbb{R}^{|\mathcal{S}||\mathcal{A}|}$, the learning policy is stationary, which corresponds to off-policy Q-learning. Motivated by practical implementations, we instead consider time-varying learning policies. 

To introduce this setting in a general manner, let $\nu:\Delta(\mathcal{A})\to \mathbb{R}$ be any closed, nonnegative, and strongly concave function, and denote $\nu_{\max}=\max_{\mu\in\Delta(\mathcal{A})}\nu(\mu)$. Without loss of generality, we assume that the strong concavity is with respect to the $\ell_1$ norm $\|\cdot\|_1$ (since all norms are equivalent up to multiplicative constants) and that the strong concavity parameter is $1$ (since any other value can be obtained by rescaling $\nu(\cdot)$). Let $\sigma:\mathbb{R}^{|\mathcal{A}|}\to \Delta(\mathcal{A})$ denote the softmax operator induced by $\nu(\cdot)$, defined as
$\sigma(x)=\argmax_{\mu\in\Delta(\mathcal{A})}\{\mu^\top x+\nu(\mu)\}$. For any $Q\in\mathbb{R}^{|\mathcal{S}||\mathcal{A}|}$, we define
\begin{align}\label{def:softmax_general}
    [f(Q)](s)
    = \epsilon\, \frac{\mathbf{1}}{|\mathcal{A}|}
    + (1-\epsilon)\, \sigma\!\left(\frac{Q(s)}{\tau}\right),
    \quad \forall\, s\in\mathcal{S},
\end{align}
where $\mathbf{1}\in\mathbb{R}^{|\mathcal{A}|}$ is the all-ones vector, and $\epsilon\in(0,1]$ and $\tau>0$ are tunable parameters. Also, recall our notation that $[f(Q)](s)$ denotes the $|\mathcal{A}|$-dimensional vector with its $a$-th entry given by $[f(Q)](s,a)$; the same convention applies to $Q(s)$ and $\pi(s)$. A representative example of $\nu(\cdot)$ is the entropy function $\nu(\mu)=-\sum_{a\in\mathcal{A}} \mu(a)\log \mu(a)$, in which case $\sigma(\cdot)$ reduces to the exponential softmax defined as $[\sigma(x)](a)=e^{x(a)}/\sum_{a'\in\mathcal{A}} e^{x(a')}$ for all $a\in\mathcal{A}$.

The rationale for adopting learning policies of the form \eqref{def:softmax_general} is twofold. First, the convex combination with a uniform policy enables explicit control of exploration via $\epsilon$, since $\min_{s,a}\pi_k(a\mid s)\ge \epsilon/|\mathcal{A}|$ for all $k\ge 0$. Second, using a softmax policy rather than the exact greedy policy ensures that the learning policy is Lipschitz continuous with respect to $Q_k$ by the conjugate correspondence theorem \cite[Theorem~5.26]{beck2017first}. This regularity is crucial for our Poisson equation–based analysis of the time-inhomogeneous Markov chain $\{(S_k,A_k)\}$ induced by Algorithm~\ref{algo:Q-learning}. More details are provided in Section~\ref{sec:proof1}. Beyond single-agent Q-learning \cite{tokic2011value,barber2023smoothed}, similar policy structures have been used in Q-learning for normal-form games \cite{leslie2005individual} and in Q-learning with linear function approximation \cite{meyn2024projected,liu2025linear}. More broadly, there is a line of work on RL for entropy-regularized MDPs, where the learning policy is inherently an exponential softmax \cite{haarnoja2017reinforcement}.

Note that while $\tau>0$ can be chosen arbitrarily in \eqref{def:softmax_general}, we do not allow $\tau=0$, which would correspond to the $\epsilon$-greedy policy in~\eqref{def:softmax_general}. Although Q-learning with $\epsilon$-greedy policies is empirically popular \cite{mnih2015human} and has been studied in terms of asymptotic convergence \cite{bertsekas1996neuro}, to the best of our knowledge it has not been analyzed from a finite-time perspective. We leave such an analysis as an interesting direction for future work.

We conclude this section with a brief remark. Although we adopt a constant stepsize $\alpha$ and constant exploration parameters $\epsilon$ and $\tau$ for ease of presentation, most of our analysis extends to time-varying sequences $\{\alpha_k\}$, $\{\epsilon_k\}$, and $\{\tau_k\}$. Further details are provided in Sections~\ref{sec:proof1} and~\ref{sec:proof2}.

\section{Main Results}\label{sec:main results}
This section presents our main theoretical findings. We begin by stating our assumption.

\begin{assumption}\label{as:MC}
    There exists a policy $\pi_b$ such that the Markov chain $\{S_k\}$ induced by $\pi_b$ is irreducible.
\end{assumption}
\begin{remark}\label{remark}
   Note that $\pi_b$ need not be visited along the algorithmic trajectory of Algorithm~\ref{algo:Q-learning}; rather, it should be viewed as an algorithm-independent assumption on the underlying MDP that characterizes its inherent exploration capability. 
   In Section~\ref{subsec:ass_necessity}, we show that Assumption~\ref{as:MC} is necessary even for the asymptotic convergence of Q-learning, which justifies this assumption as minimal.
   Without loss of generality, we assume that $\pi_b(a \mid s) > 0$ for all $(s,a)$, which will serve as the standing assumption throughout the rest of this paper. See Appendix~\ref{ap:behavior_policy} for a proof.
\end{remark}

Assumption~\ref{as:MC} is considerably weaker than those adopted in prior studies of Q-learning. Even in the off-policy setting (where the learning policy $\pi$ is stationary), it is typically assumed that $\pi$ induces an irreducible and aperiodic Markov chain \cite{chen2024lyapunov,qu2020finite,li2020sample}, with only a few recent exceptions \cite{haque2024stochastic,chandak2022concentration}. In the case of on-policy Q-learning \cite{liu2025linear}, and more broadly for RL algorithms with time-varying learning policies—such as SARSA \cite{zou2019finite,chenziyi2022sample} and actor--critic methods \cite{khodadadian2021finite,chenziyi2022sample,chenzy2021sample,xu2021sample,wu2020finite,qiu2019finite}—it is typically assumed that every learning policy along the algorithmic trajectory, or even all policies, induces a uniformly ergodic Markov chain, with stationary distributions uniformly bounded away from zero and mixing parameters uniformly bounded from above. See Appendix \ref{ap:ass_necessity} for a detailed discussion. By adopting a much weaker assumption, our framework not only provides a theoretical contribution but also enables a quantitative characterization of the exploration–exploitation trade-off in Q-learning with on-policy sampling, as demonstrated later in Section~\ref{subsec:main-results}.

The following notation is needed throughout this paper. Let $P_{\pi_b}\in\mathbb{R}^{|\mathcal{S}|\times |\mathcal{S}|}$ denote the transition matrix of the Markov chain $\{S_k\}$ induced by $\pi_b$, and define $\pi_{b,\min} := \min_{s,a} \pi_b(a \mid s)$, which is strictly positive. Since we work with finite MDPs, under Assumption~\ref{as:MC}, the Markov chain $\{S_k\}$ with transition matrix $P_{\pi_b}$ admits a unique stationary distribution \cite{levin2017markov}, denoted by $\mu_{\pi_b}\in\Delta(\mathcal{S})$, satisfying $\mu_{\pi_b,\min} := \min_s \mu_{\pi_b}(s) > 0$. Define $\mathcal{P}_{\pi_b}$ as the transition matrix of the corresponding lazy chain, i.e., $\mathcal{P}_{\pi_b} = (P_{\pi_b} + I)/2$. It is straightforward to verify that the Markov chain under $\mathcal{P}_{\pi_b}$ is \textit{irreducible and aperiodic}, sharing the same stationary distribution $\mu_{\pi_b}$. Moreover, there exist $r_b \in \mathbb{Z}_+$ and $\delta_b > 0$ such that $\min_{s,s'} \mathcal{P}_{\pi_b}^{r_b}(s,s') \geq \delta_b$ \cite[Proposition 1.7]{levin2017markov}. Importantly, the lazy chain is introduced solely for analytical purposes, while the actual sample trajectory in Algorithm~\ref{algo:Q-learning} is generated by the sequence of time-varying learning policies $\{\pi_k\}$. Before proceeding, we emphasize that the constants $\pi_{b,\min}$, $\mu_{\pi_b,\min}$, $r_b$, and $\delta_b$ reflect fundamental exploration properties of the underlying MDP, rather than being algorithm-dependent quantities.

\subsection{Finite-Time Analysis}\label{subsec:main-results}

We now present our main result.

\begin{theorem}\label{thm1}
    Suppose that Assumption~\ref{as:MC} holds and that the stepsize and exploration parameters satisfy $\alpha < 1/c_1$, $\epsilon \in (0,1]$, and $\tau \in (0,1/(1-\gamma)]$, where $c_1 = \frac{1}{2}(\epsilon/|\mathcal{A}|)^{r_b}\mu_{\pi_b,\min}\delta_b(1-\gamma)$. Then, the following inequality holds for all $k \geq 0$:
    \begin{align*}
        \mathbb{E}[\|Q_{k}-Q^*\|_{\infty}^{2}]\leq \underbrace{3 \|Q_0-Q^*\|_{\infty}^{2} \left(1 - \alpha c_{1} \right)^{k}}_{\mathrm{Bias}}  + \underbrace{c_2 \alpha + c_3 \alpha^{2}\log^{4}\left(\frac{c_4}{\alpha}\right) }_{\mathrm{Variance}}, 
    \end{align*}
    where 
    \begin{align*}
        c_{2} = \,&\frac{c_2'(r_b+1)\log(|\mathcal{S}||\mathcal{A}|) }{ \lambda^{3r_{b}+1}\pi_{b,\min}\mu_{\pi_b,\min}^3\delta_b^3(1-\gamma)^4},\\
        c_{3} =\,& \frac{c_3'(r_b+1)^4}{\tau^2\lambda^{6r_b+4}\mu_{\pi_b,\min}^6\pi_{b,\min}^4\delta_b^6(1-\gamma)^6},\quad c_4=\frac{4(r_b+1)}{\delta_{b} \lambda^{r_{b}+1} \mu_{\pi_b,\min} \pi_{b,\min}},
    \end{align*}
    with $\lambda:= \epsilon/|\mathcal{A}|$ and $c_2',c_3'$ being absolute constants.
\end{theorem}
\begin{remark}
Let $\bar{\pi}_k$ be a policy greedily induced by $Q_k$, that is,
$\{a\in\mathcal{A}\mid \bar{\pi}_k(a\mid s)>0\}\subseteq \arg\max_{a\in\mathcal{A}} Q_k(s,a)$
for all $s\in\mathcal{S}$, which is different from the learning policy $\pi_k$. Then, with probability one, we have
$\|Q^{\bar{\pi}_k}-Q^*\|_\infty
\le \frac{2\gamma}{1-\gamma}\,\|Q_k-Q^*\|_\infty$
(cf.~Lemma~\ref{le:Q-function-gap}). Therefore, up to a constant factor, the convergence rate of the iterates $Q_k$ translates directly into the convergence rate of the Q-function corresponding to the greedily induced policy $\bar{\pi}_k$.
\end{remark}

The convergence bound shows that the mean-square error decays geometrically to a neighborhood of radius $\mathcal{O}(\alpha)$. The first term on the right-hand side corresponds to the \emph{bias}, capturing the decay of the error due to initialization, while the second term corresponds to the \emph{variance}. Since a constant stepsize cannot eliminate the variance even asymptotically, the steady-state error scales with the chosen stepsize. This bias–variance trade-off is consistent with existing results for off-policy Q-learning and, more generally, stochastic approximation algorithms with constant stepsizes \cite{srikant2019finite,bhandari2018finite,chen2024lyapunov,zhang2024constant}.

Additionally, we emphasize that the convergence bound is expressed entirely in terms of either primitive algorithm design parameters (e.g., $\alpha$, $\epsilon$, and $\tau$) or algorithm-independent parameters (e.g., $1/(1-\gamma)$, $\mu_{\pi_b,\min}$, $\pi_{b,\min}$, $r_b$, and $\delta_b$), with \emph{no implicit algorithm-dependent constants}. Such explicit quantification is crucial for understanding how exploration limitations affect Q-learning with on-policy sampling. In particular, the exploration behavior depends on both the learning policies $\pi_k$ and the intrinsic properties of the MDP. While the lower bound $\lambda=\epsilon/|\mathcal{A}|$ on the policies captures the degree of exploration induced by $\pi_k$, the parameters $\delta_b$, $r_b$, $\pi_{b,\min}$, and $\mu_{\pi_b,\min}$ characterize the intrinsic exploration capacity of the MDP. Smaller values of $\lambda$, $\delta_b$, $\pi_{b,\min}$, and $\mu_{\pi_b,\min}$, or a larger $r_b$, make it more difficult to explore the state--action space. Quantitatively, this leads to a smaller $c_1$ (slower bias decay) and larger $c_2$, $c_3$, and $c_4$ (higher variance). The effect of these parameters is further reflected in the sample complexity discussed next.

\begin{corollary}\label{coro1}
    For a given $\xi > 0$, the sample complexity to achieve $\mathbb{E}[\|Q_{k} - Q^*\|_{\infty}] \leq \xi$ is 
    \begin{align*}        \mathcal{O}\left(\frac{(r_b+1)\log\left(3\|Q_0 - Q^*\|_\infty/\xi\right)}{\lambda^{4r_b+2}\mu_{\pi_b,\min}^4\pi_{b,\min}\delta_b^4(1-\gamma)^4}\max\left(\frac{\log(|\mathcal{S}||\mathcal{A}|) }{ (1-\gamma)}\frac{1}{\xi^2},\frac{r_b+1}{\tau\lambda\pi_{b,\min}}\frac{1}{\xi}\right)\right).
    \end{align*}
\end{corollary}
The proof of Corollary~\ref{coro1} is provided in Appendix~\ref{pf:coro1}. In terms of the dependence on the accuracy level $\xi$, the leading-order term is $\widetilde{\mathcal{O}}(1/\xi^{2})$, which matches that of off-policy Q-learning \cite{li2020sample,qu2020finite,even2003learning,chen2024lyapunov}. However, the dependence on other problem-specific constants, such as the effective horizon $1/(1-\gamma)$ and the size of the state--action space $|\mathcal{S}||\mathcal{A}|$ (which lower bounds $\mu_{\pi_b,\min}\pi_{b,\min}$), is worse than that of off-policy Q-learning \cite{li2020sample}. This behavior is expected, since Q-learning with on-policy sampling faces greater difficulty in exploring the entire state--action space, whereas off-policy Q-learning typically assumes a stationary learning policy, often uniform. In Section~\ref{sec:numerical}, we present numerical simulations confirming that on-policy Q-learning indeed converges more slowly than off-policy Q-learning.

While on-policy Q-learning exhibits a slower convergence rate (measured in $\mathbb{E}[\|Q_k - Q^*\|_\infty^2]$) compared to off-policy Q-learning, an important advantage is that the learning policies $\pi_k$ also converge to an optimal one, as opposed to remaining stationary in off-policy Q-learning. The explicit convergence rate is characterized in the following theorem.

\begin{theorem}\label{thm2}
    Under the same assumptions as those for Theorem \ref{thm1}, the following inequality holds for all $k \geq 0$. 
    \begin{align*}
        \mathbb{E}[\|Q^{\pi_{k}} - Q^*\|_{\infty}^{2}] \leq\; & \underbrace{ \frac{12\gamma^2}{(1-\gamma)^2} \mathbb{E}[\|Q_{k} - Q^*\|_{\infty}^{2}]}_{T_{1}} +  \underbrace{\frac{12 \epsilon^2}{(1-\gamma)^{4}}  + \frac{3\tau^2\nu_{\max}^2 }{(1-\gamma)^{2}}}_{T_{2}},
    \end{align*}
    where $\nu_{\max}=\max_{\mu\in\Delta(\mathcal{A})}\nu(\mu)$.
\end{theorem}

The proof of Theorem~\ref{thm2} is presented in Section~\ref{sec:proof2}. Note that Theorem~\ref{thm2} quantitatively demonstrates an exploration–exploitation trade-off in on-policy Q-learning. Specifically, consider the following two cases.
\begin{itemize}
    \item \textbf{Small $\epsilon$ and $\tau$: The Exploitation-Dominated Regime.} Suppose we choose $\epsilon$ and $\tau$ close to zero. In this case, the learning policy $\pi_k$ becomes nearly greedy with respect to $Q_k$ and thus lacks sufficient exploration. As a result, the term $T_1$ is large, meaning that the convergence of $Q_k$ to $Q^*$ is slow, as clearly demonstrated by Theorem~\ref{thm1} and Corollary~\ref{coro1}. However, small values of $\epsilon$ and $\tau$ promote exploitation, since $Q_k$ eventually converges to $Q^*$ and $\pi_k$ remains almost greedy with respect to $Q_k$. In this case, the term $T_2$ is small.
    \item \textbf{Large $\epsilon$ and $\tau$: The Exploration-Dominated Regime.}  
    When $\epsilon$ and $\tau$ are large, in particular, $\epsilon\rightarrow  1$ or $\tau\rightarrow\infty$, the learning policy $\pi_k$ is nearly uniform and does not depend on the current estimate $Q_k$. This broad exploration accelerates the convergence of $Q_k$ to $Q^*$, making the term $T_1$ smaller. However, excessive exploration limits exploitation, preventing the policy from fully leveraging the learned $Q_k$ and leading to a persistent gap between $Q^{\pi_k}$ and $Q^*$, as captured by the term $T_2$ in the bound. In the extreme case where $\epsilon = 1$, the algorithm performs pure uniform exploration with no exploitation at all, effectively reducing to off-policy Q-learning.
\end{itemize}

Traditionally, the exploration–exploitation trade-off has been studied primarily in the context of online learning \cite{lattimore2020bandit}, where performance is measured by regret. In recent years, this line of research has been extended to RL, focusing mainly on the episodic setting \cite{jin2018q}—where regret is defined in terms of the averaged value function gap—and the infinite-horizon average-reward setting \cite{zhang2023sharper,wang2022near}, where a natural notion of regret is given by $\sum_{k=0}^{K-1}(r_k - r^*)$. In contrast, our work characterizes an exploration–exploitation trade-off in discounted Q-learning, with the performance metric being the \textit{last-iterate convergence rate}. Importantly, our minimal-assumption framework and explicit characterization of all parameter dependencies (cf. Theorem~\ref{thm1}) are crucial for capturing this trade-off in a precise and interpretable manner. We provide further discussion on the exploration--exploitation trade-off in on-policy Q-learning in Appendix~\ref{ap:Q_to_Policy}.

\subsection{The Necessity of Assumption \ref{as:MC}}\label{subsec:ass_necessity}

We conclude this section by showing that Assumption~\ref{as:MC} is necessary even for the asymptotic convergence of Q-learning. In particular, we construct an MDP instance such that, if Assumption~\ref{as:MC} is violated, there exists an initialization $Q_0$ and a constant $c>0$ for which $\|Q_k - Q^*\|_\infty \ge c$ almost surely for all $k \ge 0$.

Consider a finite Markov reward process (MRP) $\mathcal{M}=(\mathcal{S},P,\mathcal{R},\gamma)$, where $P\in\mathbb{R}^{|\mathcal{S}|\times|\mathcal{S}|}$ is the transition probability matrix, $R(s)\equiv 1$ for all $s\in\mathcal{S}$, and $\gamma=0$. An MRP can be viewed as a special case of an MDP with only one feasible action (and hence a single deterministic policy) in each state. Although there is no policy optimization in an MRP, this setting suffices for our purpose, since our goal is to demonstrate the necessity of Assumption~\ref{as:MC} for the convergence of Q-learning rather than for identifying an optimal policy.

In this simple setup, it is clear that $Q^*(s)=1$ for all $s\in\mathcal{S}$. Moreover, the Q-learning update rule presented in Algorithm~\ref{algo:Q-learning}, Line~5, becomes
\begin{align}\label{eq:Q_example}
    Q_{k+1}(s)=Q_k(s)+\alpha_k\mathds{1}_{\{S_k=s\}}(1-Q_k(s)),\quad \forall\,s\in\mathcal{S}.
\end{align}
The next proposition shows that if $P$ is not irreducible, we cannot hope for the global convergence of $Q_k$ to $Q^*$.

\begin{proposition}\label{prop:ass_necessity}
    Suppose that $P$ is not irreducible. Then, regardless of the choice of $\{\alpha_k\}$, there exists an initialization $Q_0$ and a constant $c>0$ such that $\|Q_k-Q^*\|_\infty\geq c$ almost surely for all $k\geq 0$.
\end{proposition}

The proof of Proposition~\ref{prop:ass_necessity} is given in Appendix~\ref{pf:prop:ass_necessity}. Intuitively, when Assumption~\ref{as:MC} is violated, the Markov chain $\{S_k\}$ contains transient states. However, by the Q-learning update in \eqref{eq:Q_example}, the value $Q_k(s)$ is updated only when state $s$ is visited along the sample trajectory $\{S_k\}$. Therefore, convergence of Q-learning requires that every state be visited infinitely often, which cannot occur in the presence of transient states.

\section{Proof of Theorem \ref{thm1}}\label{sec:proof1}
This section presents the complete proof of Theorem \ref{thm1}. Specifically, we reformulate the main update equation of Q-learning with on-policy sampling as a stochastic approximation with time-inhomogeneous Markovian noise (cf. Section \ref{subsec:SA_reformulation}), set up the Lyapunov drift framework together with the error decomposition for the analysis (cf. Section \ref{subsec:Lyapunov}), and discuss in detail how to handle the rapidly varying time-inhomogeneous Markovian noise using a Poisson equation–based approach (cf. Section \ref{subsec:Poisson}). Finally, we solve the recursive Lyapunov drift inequality to establish the finite-time convergence bound.

To maintain generality in our analysis, we allow the stepsize $\alpha$ and the algorithm design parameters $\epsilon$ and $\tau$ in Algorithm~\ref{algo:Q-learning} to be time-varying sequences $\{\alpha_k\}$, $\{\epsilon_k\}$, and $\{\tau_k\}$. In this case, we denote the policy mapping by $f_k(Q)$, defined as
\begin{align}\label{eq:policy_timevarying}
[f_k(Q)](s)
    = \epsilon_k\, \frac{\mathbf{1}}{|\mathcal{A}|}
    + (1-\epsilon_k)\, \sigma\!\left(\frac{Q(s)}{\tau_k}\right),
    \quad \forall\, s\in\mathcal{S}.
\end{align}

\subsection{Stochastic Approximation with Time-Inhomogeneous Markovian Noise}
\label{subsec:SA_reformulation}
We start by reformulating Algorithm \ref{algo:Q-learning} as a stochastic approximation algorithm for solving the Bellman equation (\ref{eq:Bellman}). Let $\{Y_k\}$ be a stochastic process defined as $Y_k = (S_k, A_k)$ for all $k \geq 0$. Due to the time-varying nature of the learning policies $\{\pi_k\}$, the stochastic process $\{Y_k\}$ forms a time-inhomogeneous Markov chain evolving on the state space $\mathcal{Y} = \mathcal{S} \times \mathcal{A}$. Specifically, at time step $k$, the transition matrix is given by $\bar{P}_k((s,a), (s',a')) := p(s' | s,a) \pi_k(a' | s')$ for any $(s,a), (s',a') \in \mathcal{Y}$. Let $F:\mathbb{R}^{|\mathcal{S}||\mathcal{A}|} \times \mathcal{Y}\to \mathbb{R}^{|\mathcal{S}||\mathcal{A}|}$ be an operator such that given inputs $Q\in\mathbb{R}^{|\mathcal{S}||\mathcal{A}|}$ and $y=(s_0,a_0)\in\mathcal{Y}$, the $(s,a)$-th component of the output is defined as
\begin{align*}
    [F(Q,y)](s,a)  =\,& \mathds{1}_{\{(s_0,a_0)=(s,a)\}}\left(\mathcal{R}(s,a)+\gamma \sum_{s'\in\mathcal{S}}p(s'|s,a)\max_{a'\in\mathcal{A}} Q(s',a')
    -Q(s,a)\right)+ Q(s,a).
\end{align*}
Moreover, for any $k\geq 0$, let 
$M_{k}:\mathbb{R}^{|\mathcal{S}||\mathcal{A}|} \to \mathbb{R}^{|\mathcal{S}||\mathcal{A}|}$ be defined as
\begin{align*}
    [M_{k}(Q)](s,a) = \gamma \mathds{1}_{\{(S_k,A_k)=(s,a)\}}\left(\max_{a'\in\mathcal{A}} Q(S_{k+1},a') - \sum_{s'\in\mathcal{S}}p(s'|s,a)\max_{a'\in\mathcal{A}} Q(s',a')\right)
\end{align*}
for all $Q \in \mathbb{R}^{|\mathcal{S}||\mathcal{A}|}$.
Then, the main update equation presented in Line 5 of Algorithm \ref{algo:Q-learning} can be reformulated as 
\begin{align}\label{algo:sa_reformulation}
    Q_{k+1} = Q_{k} + \alpha_{k}(F(Q_{k},Y_{k}) - Q_{k}+M_{k}(Q_{k})),\quad \forall\,k\geq 0. 
\end{align}
To show Eq.~(\ref{algo:sa_reformulation}) corresponds to a stochastic approximation method for finding $Q^*$, we first establish preliminary results on the Markov chains induced by the learning policies along the algorithm trajectory. Let $\Pi=\{\pi\mid \min_{s,a}\pi(a\mid s)>0\}$.

\begin{lemma}\label{le:policy_exploration}
Under Assumption \ref{as:MC}, for any $\pi \in \Pi$, the induced Markov chain $\{S_n\}_{n\geq 0}$ is irreducible.
\end{lemma}

The proof of Lemma \ref{le:policy_exploration} is given in Appendix~\ref{pf:le:policy_exploration}.
As a result of Lemma~\ref{le:policy_exploration}, for any $\pi \in \Pi$, the Markov chain $\{S_n\}$ induced by $\pi$ admits a unique stationary distribution $\mu_\pi \in \Delta(\mathcal{S})$ \cite{levin2017markov}, which satisfies $\mu_\pi(s)>0$ for all $s\in\mathcal{S}$. Moreover, since $\pi(a|s)>0$ for all $\pi\in \Pi$, the Markov chain $\{Y_n = (S_n, A_n)\}_{n \geq 0}$ induced by $\pi$ is also irreducible and admits a unique stationary distribution $\bar{\mu}_\pi \in \Delta(\mathcal{S} \times \mathcal{A})$, which satisfies $\bar{\mu}_\pi(s,a) = \mu_\pi(s)\pi(a \mid s)$ for all $(s,a)$. Since Algorithm~\ref{algo:Q-learning} employs learning policies of the form $\pi_k = f_k(Q_k)$, all policies encountered along the algorithm trajectory belong to $\Pi$, and hence Lemma~\ref{le:policy_exploration} applies. For each policy $\pi_k$ along the trajectory, we define $\mu_k := \mu_{\pi_k}$ and $\bar{\mu}_k := \bar{\mu}_{\pi_k}$ accordingly.

Let $\Bar{F}:\mathbb{R}^{|\mathcal{S}||\mathcal{A}|} \times \Pi \to \mathbb{R}^{|\mathcal{S}||\mathcal{A}|}$ be defined as
\begin{align*} 
    \Bar{F}(Q,\pi) = \mathbb{E}_{Y \sim \Bar{\mu}_\pi(\cdot)}[F(Q,Y)]
\end{align*}
for any $Q \in \mathbb{R}^{|\mathcal{S}||\mathcal{A}|}$ and $\pi \in \Pi$. 
The following lemma establishes several key properties of the operator $\Bar{F}(\cdot,\cdot)$, which are important for connecting the algorithm presented in Eq. (\ref{algo:sa_reformulation}) with the Bellman equation (\ref{eq:Bellman}). The proof of Lemma \ref{le:f-bar} is presented in Appendix \ref{pf:le:f-bar}.

\begin{lemma}\label{le:f-bar}
The following results hold.
\begin{enumerate}[(1)]
    \item For any  $\pi\in \Pi$, the operator $\Bar{F}(\cdot,\pi)$ is explicitly given by 
    \begin{align*}
        \Bar{F}(Q,\pi) = \big[(I-D_\pi)+D_\pi\mathcal{H}\big](Q), \quad \forall Q\in\mathbb{R}^{|\mathcal{S}||\mathcal{A}|},
    \end{align*}
    where $D_\pi=\text{diag}(\Bar{\mu}_\pi)$ and $\mathcal{H}(\cdot)$ is the Bellman operator defined in \eqref{def:Bellman_H}.
    
    \item For any $Q_1,Q_2\in\mathbb{R}^{|\mathcal{S}||\mathcal{A}|}$ and  $\pi\in \Pi$, we have
    \begin{align*}
        \|\Bar{F}(Q_1,\pi)-\Bar{F}(Q_2,\pi)\|_\infty \leq \gamma_\pi\|Q_1-Q_2\|_\infty, \quad \text{and}\quad
        \|\Bar{F}(Q_1,\pi)\|_\infty \leq \|Q_1\|_\infty+1,    
    \end{align*}
    where $\gamma_{\pi}=1-D_{\pi,\min}(1-\gamma)$ and $D_{\pi,\min}=\min_{s,a}\Bar{\mu}_\pi(s,a)>0$. 
    
    \item For any  $\pi\in \Pi$, the fixed-point equation $\Bar{F}(Q,\pi)=Q$ has a unique solution $Q^\ast$.
    
    \item For any $Q_{1},Q_{2} \in \mathbb{R}^{|\mathcal{S}||\mathcal{A}|}$ 
    satisfying $\|Q_{1}\|_{\infty}, \|Q_{2}\|_{\infty} \leq 1/(1-\gamma)$ and $\pi_{1}, \pi_{2} \in \Pi$, we have
    \begin{align*}
        \|\Bar{F}(Q_{1},\pi_{1}) - \Bar{F}(Q_{2},\pi_{2})\|_{\infty} 
        \leq 3 \|Q_1-Q_2\|_\infty + \frac{2}{1-\gamma}\|\Bar{\mu}_{\pi_1}-\Bar{\mu}_{\pi_2}\|_\infty. 
    \end{align*}
\end{enumerate}
\end{lemma}

Among the properties established in Lemma \ref{le:f-bar}, the most important are Parts (2) and (3), which show that $\Bar{F}(\cdot, \pi)$ is a contraction mapping and that $Q^\ast$ is its unique fixed point, justifying Eq. (\ref{algo:sa_reformulation}) being a stochastic approximation algorithm for finding $Q^*$.

We end this section with the following lemma, which establishes key properties of the operator $F(Q, y)$ that will be used frequently in the remainder of the proof. The proof of Lemma \ref{le:F-lipschitz} is presented in Appendix \ref{pf:le:F-lipschitz}.
\begin{lemma}\label{le:F-lipschitz} 
Let $Q_{1},Q_{2} \in \mathbb{R}^{|\mathcal{S}||\mathcal{A}|}$,  $\pi \in \Pi$, and $y = (s_{0},a_{0}) \in \mathcal{Y}$ be arbitrary. Suppose that $\|Q_{1}\|_{\infty}, \|Q_{2}\|_{\infty} \leq 1/(1-\gamma)$. Then, we have
\begin{align*}
    \|F(Q_{1},y) - F(Q_{2},y)\|_{\infty} \leq \|Q_{1} - Q_{2}\|_{\infty},\quad\text{and}\quad
    \|F(Q_1,y) - \bar{F}(Q_1,\pi)\|_{\infty} \leq \frac{2}{1-\gamma}.
\end{align*}
\end{lemma}

\subsection{A Lyapunov Drift Approach for Error Decomposition}\label{subsec:Lyapunov}
Inspired by \cite{chen2024lyapunov}, we employ a Lyapunov-drift method to analyze the finite-time behavior of the stochastic approximation algorithm presented in Eq. (\ref{algo:sa_reformulation}). The Lyapunov function $M:\mathbb{R}^{|\mathcal{S}||\mathcal{A}|} \to \mathbb{R}^{|\mathcal{S}||\mathcal{A}|}$ is defined as 
\begin{align}\label{def:Moreau}
    M(Q) = \min_{u \in \mathbb{R}^{|\mathcal{S}||\mathcal{A}|}} \left\{\frac{1}{2} \|u\|_\infty^2 + \frac{1}{2\theta} \|Q - u\|_p^2\right\}
\end{align}
for all $Q \in \mathbb{R}^{|\mathcal{S}||\mathcal{A}|}$,
where $\|\cdot\|_p$ denotes the $\ell_p$-norm defined by $\smash{\|Q\|_p = \left(\sum_{s,a} |Q(s,a)|^p\right)^{1/p}}$. The parameters $\theta > 0$ and $p \geq 1$ are tunable and will be chosen in the final step of the proof to optimize the convergence bound.

Since we work in a finite-dimensional Euclidean space, norm equivalence ensures the existence of constants $\ell_p = (|\mathcal{S}||\mathcal{A}|)^{-1/p}$ and $u_p = 1$ such that $\ell_p \|Q\|_p \leq \|Q\|_\infty \leq u_p \|Q\|_p$ for all $Q \in \mathbb{R}^{|\mathcal{S}||\mathcal{A}|}$. Several key properties of the Lyapunov function $M(\cdot)$ were established in \cite{chen2024lyapunov}, and are summarized in the following lemma for completeness.

\begin{lemma}[Proposition 1 from \cite{chen2024lyapunov}]\label{le:Moreau}
The Lyapunov function $M(\cdot)$ satisfies the following properties:
\begin{enumerate}[(1)]
    \item The function $M(\cdot)$ is convex, differentiable, and $L$-smooth with respect to $\|\cdot\|_p$, i.e.,
    \begin{align}
        M(y) \leq M(x) + \langle \nabla M(x), y - x \rangle + \frac{L}{2} \|x - y\|_p^2, \quad \forall\, x, y \in \mathbb{R}^d, \label{eq:lyapunov-expansion}
    \end{align}
    where $L = (p - 1)/\theta$.
    
    \item There exists a norm $\|\cdot\|_m$ such that $M(Q) = \|Q\|_m^2 / 2$.

    \item It holds that $\ell_m \|Q\|_m \leq \|Q\|_\infty \leq u_m \|Q\|_m$ for all $Q \in \mathbb{R}^{|\mathcal{S}||\mathcal{A}|}$, where $\ell_m = (1 + \theta \ell_p^2)^{1/2}$ and $u_m = (1 + \theta u_p^2)^{1/2}$.
\end{enumerate}
\end{lemma}
$M(\cdot)$ serves as a smooth approximation of $\|Q\|_\infty^2/2$. See \cite{chen2024lyapunov} for further discussion of the motivation behind the construction of $M(\cdot)$.

Now, we are ready to use the Lyapunov function $M(\cdot)$ to study the stochastic approximation algorithm (\ref{algo:sa_reformulation}).
For any $k\geq 0$, using Eq. \eqref{algo:sa_reformulation} and Lemma \ref{le:Moreau} (1), we have
\begin{align}
    \mathbb{E}[M(Q_{k+1}-Q^*)]
    \leq \,& \mathbb{E}[M(Q_k-Q^*)]+ \mathbb{E}[\langle \nabla M(Q_k-Q^*),Q_{k+1}-Q_k\rangle] + \frac{L}{2} \mathbb{E}[\|Q_{k+1}-Q_k\|_p^2] \nonumber \\
    =\,& \mathbb{E}[M(Q_k-Q^*)] + \alpha_k \mathbb{E}[\langle \nabla M(Q_k-Q^*),F(Q_{k},Y_{k}) + M_{k}(Q_{k}) - Q_{k}\rangle] \nonumber \\
    &+\frac{L\alpha_k^2}{2}\mathbb{E}[\|F(Q_{k},Y_{k}) + M_{k}(Q_{k}) - Q_{k}\|_p^2] \nonumber \\
    =\,& \mathbb{E}[M(Q_k-Q^*)] + \alpha_k \underbrace{\mathbb{E}[\langle \nabla M(Q_k-Q^*),\bar{F}(Q_k,\pi_k)-Q_k \rangle]}_{:= E_{1}} \nonumber \\
    &+\alpha_k \underbrace{\mathbb{E}[\langle \nabla M(Q_k-Q^*),F(Q_{k},Y_{k}) -\bar{F}(Q_k,\pi_k) \rangle]}_{:= E_{2}} \nonumber \\
    &+\alpha_k \underbrace{\mathbb{E}[\langle \nabla M(Q_k-Q^*),M_{k}(Q_{k}) \rangle]}_{:= E_{3}} \nonumber \\
    &+\frac{L\alpha_k^2}{2} \underbrace{\mathbb{E}[\|F(Q_{k},Y_{k}) + M_{k}(Q_{k}) - Q_{k}\|_p^2]}_{:= E_{4}} \label{eq:Lyapunov-decomposition}.
\end{align} 
Next, we bound each term on the right-hand side of the previous inequality. In particular, we bound the terms $E_1$, $E_3$, and $E_4$ in the following sequence of lemmas, and dedicate the next section to our techniques for bounding the term $E_2$, which arises due to the rapidly varying time-inhomogeneous noise $\{Y_k\}$ and is the most challenging to handle.

\begin{lemma}\label{le:negDrift}
The following inequality holds for all $k \ge 0$:
\begin{align*}
    E_1 \;\le\; -2\left(1 - \frac{u_m}{\ell_m}\gamma_k \right)\, \mathbb{E}\!\left[M(Q_k - Q^*)\right],
\end{align*}
where $\gamma_k := 1 - D_{\pi_k,\min}(1-\gamma)$ and $D_{\pi_k,\min} := \min_{s,a} \bar{\mu}_k(s,a)$. 
\end{lemma}
\begin{lemma}\label{le:E_3}
    It holds for all $k\geq 0$ that $E_3=0$.
\end{lemma}
\begin{lemma}\label{le:E_4}
    It holds for all $k\geq 0$ that $E_4\leq \frac{4 (|\mathcal{S}||\mathcal{A}|)^{2/p}}{(1-\gamma)^2}$.
\end{lemma}
The proofs of Lemmas \ref{le:negDrift}, \ref{le:E_3}, and \ref{le:E_4} are presented in Appendices \ref{pf:le:negDrift}, \ref{pf:le:E_3}, and \ref{pf:le:E_4}, respectively.

\subsection{Handling the Time-Inhomogeneous Markovian noise: A Poisson Equation Approach}\label{subsec:Poisson}
The most challenging term to handle is
\begin{align*}
    E_2 = \mathbb{E}[\langle \nabla M(Q_k - Q^*),\, F(Q_k, Y_k) - \bar{F}(Q_k, \pi_k) \rangle],
\end{align*}
which arises from the time-inhomogeneous nature of the Markov chain $\{Y_k\}$. Specifically, the transition kernel of $\{Y_k\}$ varies over time because the learning policy $\pi_k$ is time-dependent. Moreover, since no lower-bound constraints are imposed on the parameters $\epsilon_k$ and $\tau_k$ that define $\pi_k$ (cf. Eq.~(\ref{eq:policy_timevarying})), the learning policies may vary rapidly over time.

\subsubsection{The Poisson Equation}\label{sec:PE}
To handle rapidly varying time-inhomogeneous Markovian noise under only Assumption \ref{as:MC}, inspired by \cite{chandak2022concentration,haque2024stochastic}, we adopt an approach based on the \emph{Poisson equation} associated with Markov chains, which allows us to decompose the Markovian noise into a martingale-difference sequence and residual terms. It is important to note, however, that \cite{chandak2022concentration,haque2024stochastic} study off-policy Q-learning and TD-learning for policy evaluation---settings that do not involve rapidly varying time-inhomogeneous Markovian noise. 

By Lemma~\ref{le:policy_exploration} and the subsequent discussion, for any $\pi \in \Pi$, the Markov chain $\{Y_n\}$ induced by $\pi$ is irreducible and admits a unique stationary distribution $\bar{\mu}_\pi$. Therefore, for any $Q \in \mathbb{R}^{|\mathcal{S}||\mathcal{A}|}$, $\pi \in \Pi$, and $(s,a)\in\mathcal{S}\times \mathcal{A}$, there exists a mapping $[h(Q,\pi,\cdot)](s,a)$ that solves the Poisson equation \cite{douc2018markov} associated with the mapping $[F(Q,\cdot)](s,a)$:
\begin{align}
    [F(Q,y)](s,a) - [\bar{F}(Q,\pi)](s,a)
    =
    [h(Q,\pi,y)](s,a)
    -
    \sum_{y'\in\mathcal{Y}}\bar{P}_\pi(y,y')\,[h(Q,\pi,y')](s,a).
    \label{eq:PE-SARSA-1}
\end{align}
We study the properties of the solution $h(Q,\pi,y)\in\mathbb{R}^{|\mathcal{S}||\mathcal{A}|}$ in the next subsection. 
With the Poisson equation \eqref{eq:PE-SARSA-1}, we decompose the term $E_2$ from Eq. (\ref{eq:Lyapunov-decomposition}) as follows:
\begin{align}
    E_2 =\,& \mathbb{E}[\langle \nabla M(Q_k - Q^*),\, F(Q_k, Y_k) - \bar{F}(Q_k, \pi_k) \rangle]\nonumber\\
    =\,&\mathbb{E}[\langle \nabla M(Q_k - Q^*), h(Q_k, \pi_k, Y_k) - \textstyle{\sum_{y'\in\mathcal{Y}}}\bar{P}_k(Y_k,y')h(Q_k, \pi_k, y')] \nonumber \\
        =\,& \underbrace{\mathbb{E}[\langle \nabla M(Q_k - Q^*), h(Q_k, \pi_k, Y_{k+1}) - \textstyle{\sum_{y'\in\mathcal{Y}}}\bar{P}_k(Y_k,y')h(Q_k, \pi_k, y')]}_{:=E_{2,1}} \nonumber\\
        &+ \underbrace{\mathbb{E}[\langle \nabla M(Q_k - Q^*), h(Q_k, \pi_k, Y_k) \rangle - \frac{\alpha_{k+1}}{\alpha_k}\mathbb{E}[\langle \nabla M(Q_{k+1} - Q^*), h(Q_{k+1}, \pi_{k+1}, Y_{k+1}) \rangle]}_{:= E_{2,2}} \nonumber\\
        &+ \underbrace{\frac{\alpha_{k+1}}{\alpha_k}\mathbb{E}[\langle \nabla M(Q_{k+1} - Q^*) - \nabla M(Q_k - Q^*), h(Q_{k+1}, \pi_{k+1}, Y_{k+1}) \rangle]}_{:= E_{2,3}} \nonumber\\
        &+ \underbrace{\frac{\alpha_{k+1}}{\alpha_k}\mathbb{E}[\langle \nabla M(Q_k - Q^*), h(Q_{k+1}, \pi_{k+1}, Y_{k+1}) - h(Q_k, \pi_k, Y_{k+1}) \rangle]}_{:= E_{2,4}}\nonumber\\
        &+\underbrace{\left(\frac{\alpha_{k+1}}{\alpha_k}-1\right)\mathbb{E}[\langle \nabla M(Q_k - Q^*), h(Q_k, \pi_k, Y_{k+1}) \rangle]}_{:= E_{2,5}},\label{eq:Poisson_decomposition}
\end{align}
where $\bar{P}_k$ denotes the shorthand notation for $\bar{P}_{\pi_k}$. The logic behind the decomposition of $E_2$ is to construct a martingale-difference term $E_{2,1}$ while the terms $E_{2,2}$ -- $E_{2,5}$ are treated as residuals. The term $E_{2,1}$ vanishes by the tower property of conditional expectations, since the random process
$m_k := h(Q_k,\pi_k,Y_{k+1}) - \textstyle{\sum_{y'\in\mathcal{Y}}}\bar{P}_k(Y_k,y')h(Q_k,\pi_k,y')$
is a martingale difference sequence. Since $E_2$ (and hence $E_{2,1}$--$E_{2,5}$) is multiplied by $\alpha_k$ in the original error decomposition inequality~\eqref{eq:Lyapunov-decomposition}, the ratio $\alpha_{k+1}/\alpha_k$ is introduced solely to create a clean telescoping structure in $E_{2,2}$. This choice will become clear after presenting the overall Lyapunov drift inequality (cf. Proposition \ref{le:final-drift-ineq}).

To bound the terms $E_{2,3}$, $E_{2,4}$, and $E_{2,5}$, we require (i) boundedness properties of the Poisson equation solution $h(Q,\pi,\cdot)$ and (ii) a sensitivity analysis of $h(Q,\pi,y)$ with respect to $(Q,\pi)$. 
Importantly, to fully characterize the convergence rate of Q-learning with on-policy sampling and capture the exploration--exploitation trade-off, all constants in these bounds must be made explicit. In particular, they must depend only on primitive algorithm-design parameters (e.g., $\epsilon_k$, $\tau_k$, and $\alpha_k$) or on algorithm-independent parameters (e.g., $\pi_{b,\min}$, $\mu_{\pi_b,\min}$, $r_b$, and $\delta_b$) that capture fundamental properties of the underlying MDP. 

To this end, we next present a general result regarding the sensitivity analysis of the solution to the Poisson equation.

\subsubsection{Sensitivity Analysis Based on the Lazy Chain}
Consider a Markov chain with transition probability matrix $P$ over a finite state space $\mathcal{X}$, and let $d = |\mathcal{X}|$. Assume that $P$ is irreducible, and let $\mu \in \Delta(\mathcal{X})$ denote its unique stationary distribution \cite{levin2017markov}. The Poisson equation associated with a right-hand-side vector $y \in \mathbb{R}^d$ is given by
\begin{align}
    (I - P)x = y, \label{eq:PE-irred}
\end{align}
where we assume, without loss of generality, that $\mu^\top y = 0$. Let $\mathcal{P} = (P + I)/2$ denote the transition matrix of the corresponding lazy chain, which is irreducible and aperiodic, and therefore satisfies $\max_{i \in \{1, 2, \ldots, d\}} \|P^k(i, \cdot) - \mu(\cdot)\|_{\mathrm{TV}} \leq C\rho^k$ for all $k\geq 0$,
where $(C, \rho)$ are the \emph{mixing parameters} of $\mathcal{P}$. The following proposition presents several key properties of a particular solution to Eq. \eqref{eq:PE-irred}. Specifically, Proposition~\ref{prop:MC} (1) presents the boundedness of $h$, which is a classical result \cite{asmussen1992stationarity,glynn2002hoeffding} and is included here for completeness, while Proposition~\ref{prop:MC} (2) presents the sensitivity analysis of the Poisson equation solution.

\begin{proposition}\label{prop:MC}
Let $P, P_1, P_2 \in \mathbb{R}^{d \times d}$ be three irreducible stochastic matrices, and let $\mu$, $\mu_1$, and $\mu_2$ denote their corresponding stationary distributions. Then, the following results hold:
\begin{enumerate}
    \item For any $y \in \mathbb{R}^d$, the vector $x := \sum_{k=0}^\infty \mathcal{P}^k y/2$ is a solution to the Poisson equation $(I - P)x = y$. Moreover, we have $\|x\|_\infty \leq \frac{C}{1 - \rho} \|y\|_\infty$,
    where $(C,\rho)$ are the mixing parameters associated with $\mathcal{P}$.
    
    \item Let $x_1 = \sum_{k=0}^\infty \mathcal{P}_1^k y_1/2$ and $x_2 = \sum_{k=0}^\infty \mathcal{P}_2^k y_2/2$
    be the solutions to the Poisson equations $(I - P_{1})x = y_1$ and $(I - P_{2})x = y_2$, respectively. Then, we have
    \begin{align*}
    \|x_1 - x_2\|_\infty \leq\,& \frac{1}{4}\left(\frac{\log(\|P_1 - P_2\|_\infty(1-\rho_{\max}))-\log(8C_{\max})}{\log(\rho_{\max})}\right)^2 \|P_1 - P_2\|_\infty (\|y_1\|_\infty+\|y_2\|_\infty)\\
    &+ \frac{1}{2}\left(\frac{\log(\|P_1 - P_2\|_\infty(1-\rho_{\max}))-\log(8C_{\max})}{\log(\rho_{\max})}\right) \|y_1 - y_2\|_\infty.
\end{align*}
    where $C_{\max}=\max(C_1,C_2)$ and $\rho_{\max}=\max(\rho_1,\rho_2)$ with $(C_1,\rho_1)$ and $(C_2,\rho_2)$ being the mixing parameters associated with $\mathcal{P}_1$ and $\mathcal{P}_2$, respectively.
\end{enumerate}
\end{proposition}

The proof of Proposition~\ref{prop:MC} is given in Appendix~\ref{pf:prop:MC} and relies on the geometric mixing of the lazy transition matrix $\mathcal{P}$. Since $P$ and $\mathcal{P}$ share the same stationary distribution and their corresponding Poisson equation solutions differ only by a multiplicative constant, working with the lazy chain allows us to carry out the required sensitivity analysis while leveraging geometric mixing properties that are not available for the original transition matrix $P$.

\subsubsection{Controlling the Rapidly Varying Time-inhomogeneous Markovian noise}

With Proposition \ref{prop:MC} at hand, the next step is to bound the terms $E_{2,3}$--$E_{2,5}$ in Eq.~(\ref{eq:Poisson_decomposition}). Specifically, to bound the term $E_{2,3}$, we identify $P = \bar{P}_k$ and apply Proposition \ref{prop:MC} (1); to bound the term $E_{2,4}$, we identify $P_1 = \bar{P}_{k+1}$ and $P_2 = \bar{P}_k$ and apply Proposition \ref{prop:MC} (2); and to bound the term $E_{2,5}$, we identify $P = \bar{P}_{k+1}$ and apply Proposition \ref{prop:MC} (1). This enables us to bound the terms $E_{2,3}$--$E_{2,5}$ in terms of $Q_k$, $Q_{k+1}$, $\pi_k$, $\pi_{k+1}$, and the mixing parameters associated with the lazy transition matrices $\bar{\mathcal{P}}_{k+1}$ and $\bar{\mathcal{P}}_k$. 
The results are presented in Lemmas~\ref{le:E23}, \ref{le:E24}, and \ref{le:E25}, with proofs given in Appendices~\ref{pf:le:E23}, \ref{pf:le:E24}, and \ref{pf:le:E25}, respectively. For notational simplicity, let $(\bar{C}_k,\bar{\rho}_k)$ denote the mixing parameters associated with the lazy transition matrix $\bar{\mathcal{P}}_k=(\bar{P}_k+I)/2$.

\begin{lemma}\label{le:E23}
The following inequality holds for all $k\geq 0$:
    \begin{align*}
        E_{2,3} \leq \frac{4\bar{C}_{k+1} L (|\mathcal{S}||\mathcal{A}|)^{2/p} \alpha_{k+1}}{(1 - \bar{\rho}_{k+1})(1 - \gamma)^2},
    \end{align*}
    where $L$ is the smoothness parameter of the Lyapunov function $M(\cdot)$ introduced in Lemma \ref{le:Moreau}.
\end{lemma}
\begin{lemma}\label{le:E24}
    The following inequality holds for all $k\geq 0$:
    \begin{align*}
        E_{2,4} \leq \frac{\alpha_{k+1}}{2\alpha_k}\left(1-\frac{u_m}{\ell_m}\gamma_k\right)\mathbb{E}[M(Q_k - Q^*)] + \frac{\alpha_{k+1} N_k^2}{\alpha_k\ell_m^{2}\left(1-\frac{u_m}{\ell_m}\gamma_k\right)},
    \end{align*}
    where
    \begin{align*}
        N_k=\,&\frac{5}{1-\gamma}\left(\frac{\log(g_k(1-\bar{\rho}_{k+1}))-\log(8\bar{C}_{k+1})}{\log(\bar{\rho}_{k+1})}\right)^2 g_k,\\
        g_k=\,&2|\epsilon
    _k-\epsilon_{k+1}|+\frac{2\alpha_k}{\tau_k(1-\gamma)}+\frac{|\tau_k-\tau_{k+1}|}{\tau_k\tau_{k+1}(1-\gamma)}.
    \end{align*}
\end{lemma}
\begin{lemma}\label{le:E25}
    The following inequality holds for all $k\geq 0$:
    \begin{align*}
        E_{2,5} \leq \frac{1}{2}\left(1-\frac{u_m}{\ell_m}\gamma_k\right) \mathbb{E}[M(Q_k - Q^*)] + \frac{4(\alpha_{k+1}-\alpha_{k})^{2}\bar{C}_k^{2}}{\alpha_{k}^{2}\ell_{m}^{2}(1 - \bar{\rho}_k)^{2} (1-\gamma)^{2} \left(1-\frac{u_m}{\ell_m}\gamma_k\right)}.
    \end{align*}
\end{lemma}
Now that we have successfully bounded all the terms on the right-hand side of Eq.~(\ref{eq:Poisson_decomposition}), we arrive at the following result for controlling the error induced by time-inhomogeneous Markovian noise.

\begin{lemma}\label{le:E_2}
    The following inequality holds for all $k\geq 0$:
    \begin{align*}
        E_2\leq\,&\left(1-\frac{u_m}{\ell_m}\gamma_k\right) \mathbb{E}[M(Q_k - Q^*)]+E_{2,2}+\frac{4\bar{C}_{k+1} L (|\mathcal{S}||\mathcal{A}|)^{2/p} \alpha_{k+1}}{(1 - \bar{\rho}_{k+1})(1 - \gamma)^2}\\
        &+ \frac{N_k^2}{\ell_m^{2}\left(1-\frac{u_m}{\ell_m}\gamma_k\right)}+\frac{4(\alpha_{k+1}-\alpha_{k})^{2}\bar{C}_k^{2}}{\alpha_{k}^{2}\ell_{m}^{2}(1 - \bar{\rho}_k)^{2} (1-\gamma)^{2} \left(1-\frac{u_m}{\ell_m}\gamma_k\right)}.
    \end{align*}
\end{lemma}
The proof of Lemma~\ref{le:E_2} directly follows from Lemmas~\ref{le:E23}, \ref{le:E24}, and~\ref{le:E25}, and hence is omitted. 

\subsection{Eliminating Implicit Algorithm-Dependent Parameters}\label{subsec:implicit_to_explicit}
Having established upper bounds for all terms on the right-hand side of~\eqref{eq:Lyapunov-decomposition}, we next derive an overall Lyapunov drift inequality. Before doing so, we note that the parameter $\gamma_k=1-\min_{s,a}\bar{\mu}_k(s,a)(1-\gamma)$ in Lemma~\ref{le:negDrift}, as well as the mixing parameters $(\bar{C}_k,\bar{\rho}_k)$ in Lemma~\ref{le:E_2}, depend implicitly on the learning policy $\pi_k$ generated by Algorithm~\ref{algo:Q-learning}. To ensure that the final results do not involve such implicit, algorithm-dependent quantities, we further upper bound $\gamma_k$, $\bar{C}_k$, and $\bar{\rho}_k$ in terms of explicit algorithm design parameters (e.g., $\epsilon_k$ and $\tau_k$) and primitive problem parameters $(\mu_{\pi_b,\min}, \pi_{b,\min}, \delta_b, r_b)$ that characterize the fundamental exploration properties of the underlying MDP (see Assumption~\ref{as:MC} and the subsequent discussion). This step is essential for quantitatively characterizing the exploration--exploitation trade-off in on-policy Q-learning.

The following lemma provides a lower bound on the minimum component of the stationary distribution $\bar{\mu}_\pi$ and upper bounds on the mixing parameters $(\bar{C}_\pi, \bar{\rho}_\pi)$ for any policy $\pi \in \Pi$.
The proof of Lemma~\ref{le:implicit_to_explicit} is given in Appendix~\ref{pf:le:mixing-parameters}.

\begin{lemma}\label{le:implicit_to_explicit}
Suppose that Assumption~\ref{as:MC} holds. For any policy $\pi \in \Pi$, let $\pi_{\min} := \min_{s,a}\pi(a\mid s)$. Then, the following results hold:
\begin{enumerate}[(1)]
    \item $\min_{s,a}\bar{\mu}_\pi(s,a) \ge \pi_{\min}^{r_b+1}\,\delta_b\,\mu_{\pi_b,\min}$,
    \item $\bar{C}_\pi \le (1 - \frac{1}{2}\delta_b\,\pi_{\min}^{r_b+1}\,\mu_{\pi_b,\min}\,\pi_{b,\min})^{-1}$ and $\bar{\rho}_\pi \le (1 - \frac{1}{2}\delta_b\,\pi_{\min}^{r_b+1}\,\mu_{\pi_b,\min}\,\pi_{b,\min})^{1/(r_b+1)}$.
\end{enumerate}
\end{lemma}
\begin{remark}
    According to Lemma~\ref{le:implicit_to_explicit}, as the policy becomes close to deterministic, i.e., $\pi_{\min}=0$, the lower bound on $\min_{s,a}\bar{\mu}_\pi(s,a)$ can vanish. Appendix~\ref{ap:ass_necessity} provides an explicit example demonstrating this behavior. This observation highlights the necessity of enforcing active exploration in the algorithm, as ensured by our $\epsilon$-softmax policy structure \eqref{def:softmax_general}. This lemma also highlights our technical contribution of working under minimal assumptions. In contrast, existing studies on RL with time-varying policies typically impose strong conditions, such as uniform ergodicity along the entire trajectory, with uniformly bounded mixing rates and stationary distributions bounded away from zero \cite{liu2025linear,zou2019finite}, under which exploration challenges—and the need for Lemma~\ref{le:implicit_to_explicit}—do not arise. Further discussion is provided in Appendix~\ref{ap:ass_necessity}.
\end{remark}

In view of the structure of the learning policies in~\eqref{def:softmax_general}, we have
$\min_{s,a}\pi_k(a\mid s) \ge \lambda_k:=\epsilon_k/|\mathcal{A}|$ for all $k \ge 0$. Combining this observation with Lemma~\ref{le:implicit_to_explicit} yields
\begin{align*}
    \gamma_k &\le 1 - \lambda_k^{r_b+1}\,\delta_b\,\mu_{\pi_b,\min}(1-\gamma),\\
    \bar{C}_k &\le \left(1 - \frac{1}{2}\delta_b\,\lambda_k^{r_b+1}\,\mu_{\pi_b,\min}\,\pi_{b,\min}\right)^{-1},\quad 
    \bar{\rho}_k \le \left(1 - \frac{1}{2}\delta_b\,\lambda_k^{r_b+1}\,\mu_{\pi_b,\min}\,\pi_{b,\min}\right)^{1/(r_b+1)}.
\end{align*}

\subsection{Establishing the Lyapunov Drift Inequality}\label{subsec:final-drift-ineq}
Having obtained the bounds on the terms $E_{1},\ldots,E_{4}$ in Eq. (\ref{eq:Lyapunov-decomposition}), we are now ready to put them together to get the one-step drift inequality. 
\begin{proposition}\label{le:final-drift-ineq}
    The following inequality holds for all $k\geq 0$
    \begin{align*}
        \mathbb{E}[M(Q_{k+1}-Q^*)]
        \leq\, & \left[ 1 - \alpha_{k}\left(1-\frac{u_m}{\ell_m}\gamma_k\right) \right] \mathbb{E}[M(Q_k-Q^*)]+\alpha_kE_{2,2}+ \frac{\alpha_{k} N_k^2}{\ell_{m}^{2}\left(1-\frac{u_m}{\ell_m}\gamma_k\right)}\\
        &+ \frac{6\bar{C}_{k+1} L (|\mathcal{S}||\mathcal{A}|)^{2/p} \alpha_{k}^{2}}{(1 - \bar{\rho}_{k+1})(1 - \gamma)^2} 
        + \frac{4(\alpha_{k+1}-\alpha_{k})^{2}\bar{C}_k^{2}}{\alpha_{k}(1 - \bar{\rho}_k)^{2} (1-\gamma)^{2} \left(1-\frac{u_m}{\ell_m}\gamma_k\right)}.
    \end{align*}
\end{proposition}
The proof of Proposition~\ref{le:final-drift-ineq} follows directly by combining Eq.~\eqref{eq:Lyapunov-decomposition} with Lemmas~\ref{le:negDrift}, \ref{le:E_2}, \ref{le:E_3}, and~\ref{le:E_4}, and is therefore omitted. From the right-hand side of the bound in Proposition~\ref{le:final-drift-ineq}, the first term is contractive, the second term $\alpha_k E_{2,2}$ admits a telescoping structure (see \eqref{eq:Poisson_decomposition} for the expression of $E_{2,2}$), and the remaining terms are dominated, in order, by the negative drift.

Proposition~\ref{le:final-drift-ineq} establishes the foundation for deriving the convergence rate of Algorithm~\ref{algo:Q-learning}. In particular, for \emph{arbitrary} choices of stepsizes $\{\alpha_k\}$ and learning policy parameters $\{\epsilon_k\}$ and $\{\tau_k\}$—including both constant and diminishing sequences—the convergence rates of Algorithm~\ref{algo:Q-learning} can be obtained by repeatedly invoking Proposition~\ref{le:final-drift-ineq}. For the purpose of proving Theorem~\ref{thm1}, we henceforth focus on the constant-parameter case by setting $\alpha_k \equiv \alpha$, $\epsilon_k \equiv \epsilon$, and $\tau_k \equiv \tau$. The final steps in proving Theorem~\ref{thm1} are as follows:

\begin{itemize}
    \item Repeatedly applying the one-step drift inequality in Proposition~\ref{le:final-drift-ineq} to obtain an overall bound on $\mathbb{E}[M(Q_k - Q^*)]$, and using Lemma~\ref{le:Moreau} to translate this bound into one on $\mathbb{E}[\|Q_k - Q^*\|_\infty^2]$.
    \item Using Lemma \ref{le:implicit_to_explicit} to make all parameters explicit in terms of either the primitive algorithm design parameters (e.g., $\epsilon$ and $\tau$) or the algorithm-independent parameters $(\mu_{\pi_b,\min}, \pi_{b,\min}, \delta_b, r_b)$ that capture the fundamental properties of the underlying MDP.
    \item Fixing the tunable parameters $p$ and $\theta$ used in defining the Lyapunov function (cf. Eq.~(\ref{def:Moreau})).
\end{itemize}
The details are presented in Appendix~\ref{pf:le:Qk-final-bound-constant}. The proof of Theorem~\ref{thm1} is thus completed after these final steps.

\section{Proof of Theorem \ref{thm2}}\label{sec:proof2}
To prove Theorem~\ref{thm2}, we need to translate the $Q$-function error $\|Q_k - Q^*\|_\infty$ into the policy error $\|Q^{\pi_k} - Q^*\|_\infty$. As in the proof of Theorem~\ref{thm1}, we work in a general setting that allows the algorithmic parameters $\alpha_k$, $\epsilon_k$, and $\tau_k$ to vary with $k$.

Recall that $\mathcal{H}(\cdot)$ denotes the Bellman operator (see Eq.~\eqref{def:Bellman_H}).  
Given a policy $\pi$, let $\mathcal{H}^\pi:\mathbb{R}^{|\mathcal{S}||\mathcal{A}|}\to\mathbb{R}^{|\mathcal{S}||\mathcal{A}|}$ denote the Bellman operator associated with $\pi$, defined as
\begin{align*}
    [\mathcal{H}^\pi(Q)](s,a)
    = \mathcal{R}(s,a)
    + \gamma \sum_{s',a'} p(s' \mid s,a)\pi(a' \mid s')Q(s',a'),
    \quad \forall\,(s,a).
\end{align*}
Similar to $\mathcal{H}(\cdot)$, the operator $\mathcal{H}^\pi(\cdot)$ is also a $\gamma$-contraction with respect to $\|\cdot\|_\infty$, and $Q^\pi$ is its unique fixed point \cite{bertsekas1996neuro}.

For any $k\ge 0$, using the Bellman equations $Q^*=\mathcal{H}(Q^*)$ and $Q^{\pi_k}=\mathcal{H}^{\pi_k}(Q^{\pi_k})$, we have
\begin{align*}
    \|Q^{\pi_{k}} - Q^*\|_{\infty}
    =\;& \|\mathcal{H}^{\pi_{k}}(Q^{\pi_{k}}) - \mathcal{H}(Q^*)\|_{\infty} \\
    =\;& \|\mathcal{H}^{\pi_{k}}(Q^{\pi_{k}}) - \mathcal{H}^{\pi_{k}}(Q_{k})
    + \mathcal{H}^{\pi_{k}}(Q_{k}) - \mathcal{H}(Q_{k})
    + \mathcal{H}(Q_{k}) - \mathcal{H}(Q^*)\|_{\infty} \\
    \le\;& \|\mathcal{H}^{\pi_{k}}(Q^{\pi_{k}}) - \mathcal{H}^{\pi_{k}}(Q_{k})\|_{\infty}
    + \|\mathcal{H}^{\pi_{k}}(Q_{k}) - \mathcal{H}(Q_{k})\|_{\infty}
    + \|\mathcal{H}(Q_{k}) - \mathcal{H}(Q^*)\|_{\infty} \\
    \le\;& \gamma \|Q^{\pi_{k}} - Q_{k}\|_{\infty}
    + \|\mathcal{H}^{\pi_{k}}(Q_{k}) - \mathcal{H}(Q_{k})\|_{\infty}
    + \gamma\|Q_{k} - Q^*\|_{\infty} \\
    =\;& \gamma \|Q^{\pi_{k}} - Q^* + Q^* - Q_{k}\|_{\infty}
    + \|\mathcal{H}^{\pi_{k}}(Q_{k}) - \mathcal{H}(Q_{k})\|_{\infty}
    + \gamma\|Q_{k} - Q^*\|_{\infty} \\
    \le\;& \gamma \|Q^{\pi_{k}} - Q^*\|_{\infty}
    + 2\gamma \|Q_{k} - Q^*\|_{\infty}
    + \|\mathcal{H}^{\pi_{k}}(Q_{k}) - \mathcal{H}(Q_{k})\|_{\infty},
\end{align*}
which implies
\begin{align}\label{eq1:pf:thm2}
    \|Q^{\pi_{k}} - Q^*\|_{\infty}
    \le\;
    \frac{2 \gamma}{1-\gamma} \|Q_{k} - Q^*\|_{\infty}
    + \frac{1}{1-\gamma}
    \|\mathcal{H}^{\pi_{k}}(Q_{k}) - \mathcal{H}(Q_{k})\|_{\infty}.
\end{align}

It remains to bound $\|\mathcal{H}^{\pi_{k}}(Q_{k}) - \mathcal{H}(Q_{k})\|_{\infty}$.  
For any $k\ge 0$ and $(s,a)$, using the definition of $\pi_k$ (cf.\ Eq.~\eqref{eq:policy_timevarying}), we have 
\begin{align}
    & \left| [\mathcal{H}(Q_{k})](s,a) - [\mathcal{H}^{\pi_{k}}(Q_{k})](s,a) \right| \nonumber\\
    =\,&  \gamma \sum_{s'\in\mathcal{S}} p(s'\rvert s, a)
    \left\{
    \max_{\mu\in\Delta(\mathcal{A})} \mu^\top Q_{k}(s')
    - [f_k(Q_k)](s')^\top Q_{k}(s')
    \right\} \nonumber\\
    \le\,&  \gamma \max_{s'\in\mathcal{S}}
    \left\{
    \max_{\mu\in\Delta(\mathcal{A})} \mu^\top Q_{k}(s')
    - [f_k(Q_k)](s')^\top Q_{k}(s')
    \right\} \nonumber\\
    =\,& \gamma \max_{s'\in\mathcal{S}}
    \left\{
    \max_{\mu\in\Delta(\mathcal{A})} \mu^\top Q_{k}(s')
    - \frac{\epsilon_k}{|\mathcal{A}|}\mathbf{1}^\top Q_k(s')
    - (1-\epsilon_k)\sigma\!\left(\frac{Q_k(s')}{\tau_k}\right)^\top Q_k(s')
    \right\} \nonumber\\
    \le\,& 2 \epsilon_{k} \gamma\|Q_{k}\|_\infty
    + \gamma (1-\epsilon_{k}) \max_{s'\in\mathcal{S}}
    \left\{
    \max_{\mu\in\Delta(\mathcal{A})} \mu^\top Q_{k}(s')
    - \sigma\!\left(\frac{Q_k(s')}{\tau_k}\right)^\top Q_k(s')
    \right\}.
    \label{eq:H_gap}
\end{align}

To proceed, since $\sigma(x)=\argmax_{\mu\in\Delta(\mathcal{A})}\{\mu^\top x+\nu(\mu)\}$, for any $x\in\mathbb{R}^{|\mathcal{A}|}$ and $\tau>0$,
\begin{align*}
    \sigma\!\left(\frac{x}{\tau}\right)^\top x
    = \max_{\mu\in\Delta(\mathcal{A})}\{\mu^\top x+\tau\nu(\mu)\}
    - \tau\nu\!\left(\sigma\!\left(\tfrac{x}{\tau}\right)\right)
    \ge \max_{\mu\in\Delta(\mathcal{A})}\mu^\top x - \tau\nu_{\max},
\end{align*}
where the inequality follows from the nonnegativity of $\nu(\cdot)$. Combining this bound with \eqref{eq:H_gap} yields
\begin{align*}
    \left| [\mathcal{H}(Q_{k})](s,a) - [\mathcal{H}^{\pi_{k}}(Q_{k})](s,a) \right|
    \le 2 \epsilon_{k} \gamma\|Q_{k}\|_\infty
    + \gamma (1-\epsilon_{k}) \tau_k \nu_{\max} 
    \le \frac{2\epsilon_k}{1-\gamma} + \tau_k\nu_{\max},
\end{align*}
where the last inequality uses $\|Q_k\|_\infty\le 1/(1-\gamma)$ \cite{gosavi2006boundedness}, $\gamma<1$, and $\epsilon_k>0$.
Since this bound holds for all $(s,a)$, we obtain
\begin{align*}
    \|\mathcal{H}^{\pi_{k}}(Q_{k}) - \mathcal{H}(Q_{k})\|_{\infty}
    \le \frac{2\epsilon_{k}}{1-\gamma} + \tau_{k} \nu_{\max}.
\end{align*}
Substituting the bound for $\|\mathcal{H}^{\pi_{k}}(Q_{k}) - \mathcal{H}(Q_{k})\|_{\infty}$ into \eqref{eq1:pf:thm2} yields
\begin{align*}
    \|Q^{\pi_{k}} - Q^*\|_{\infty}
    \le\;
    \frac{2 \gamma}{1-\gamma} \|Q_{k} - Q^*\|_{\infty}
    + \frac{2 \epsilon_{k}}{(1-\gamma)^{2}}
    + \frac{\tau_{k} \nu_{\max}}{1-\gamma}.
\end{align*}
Since $(a+b+c)^2\le 3(a^2+b^2+c^2)$ for any $a,b,c\in\mathbb{R}$, the above inequality implies
\begin{align*}
    \|Q^{\pi_{k}} - Q^*\|_{\infty}^{2}
    \le \frac{12 \gamma^2}{(1-\gamma)^2} \|Q_{k} - Q^*\|_{\infty}^{2}
    + \frac{12 \epsilon_k^{2}}{(1-\gamma)^{4}}
    + \frac{3 \tau_k^{2} \nu_{\max}^2}{(1-\gamma)^{2}}.
\end{align*}
Theorem~\ref{thm2} then follows by (i) taking expectations on both sides and (ii) setting $\epsilon_k\equiv\epsilon$ and $\tau_k\equiv\tau$.

\section{Numerical Simulations}\label{sec:numerical}
In this section, we present numerical simulations. Importantly, the goal is not to demonstrate the empirical success of Q-learning, which has already been extensively validated in the literature, but rather to verify Theorems~\ref{thm1} and~\ref{thm2}. Specifically, we demonstrate that Q-learning with on-policy sampling converges more slowly compared to off-policy sampling. On the other hand, the learning policies in Q-learning with on-policy sampling also converge to an optimal one, which serves as an advantage compared to off-policy Q-learning.

\subsection{MDP Setup}\label{subsec:simulations_setup}
We begin by describing the construction of the MDP. The goal is to design an instance in which exploration requires effort rather than being free. Consider an infinite-horizon discounted MDP with $\mathcal{S} = \{s_1, s_2, \ldots, s_n\}$ and $\mathcal{A} = \{a_1, a_2, \ldots, a_m\}$, where we set $n = 20$ and $m = 10$. 
The transition probabilities are defined as follows: for all $s \in \mathcal{S}$ and $a \neq a_m$, we have $p(s \mid s, a) = 1$, and for $a = a_m$, we have $p(s_{(i+1)\bmod n} \mid s_i, a_m) = 1$. 
In other words, taking any action other than $a_m$ keeps the system in the same state, whereas taking action $a_m$ moves the system deterministically to the next state in a cyclic manner (i.e., from $s_i$ to $s_{(i+1)\bmod n}$). 
We refer to the actions $a_1, \ldots, a_{m-1}$ collectively as \emph{stay} and to $a_m$ as \emph{move}. 
The reward function $R$ is defined by $R(s, \mathrm{\emph{stay}}) = 0$ and $R(s, \mathrm{\emph{move}}) = 1$ for every $s \in \mathcal{S}$, and the discount factor is set to $\gamma = 0.99$. 
This construction is illustrated in Figure~\ref{fig:cyclic-mdp}.

\begin{figure}[h]
    \centering
    \begin{minipage}{0.35\textwidth}
        \centering
        \includegraphics[width=\linewidth]{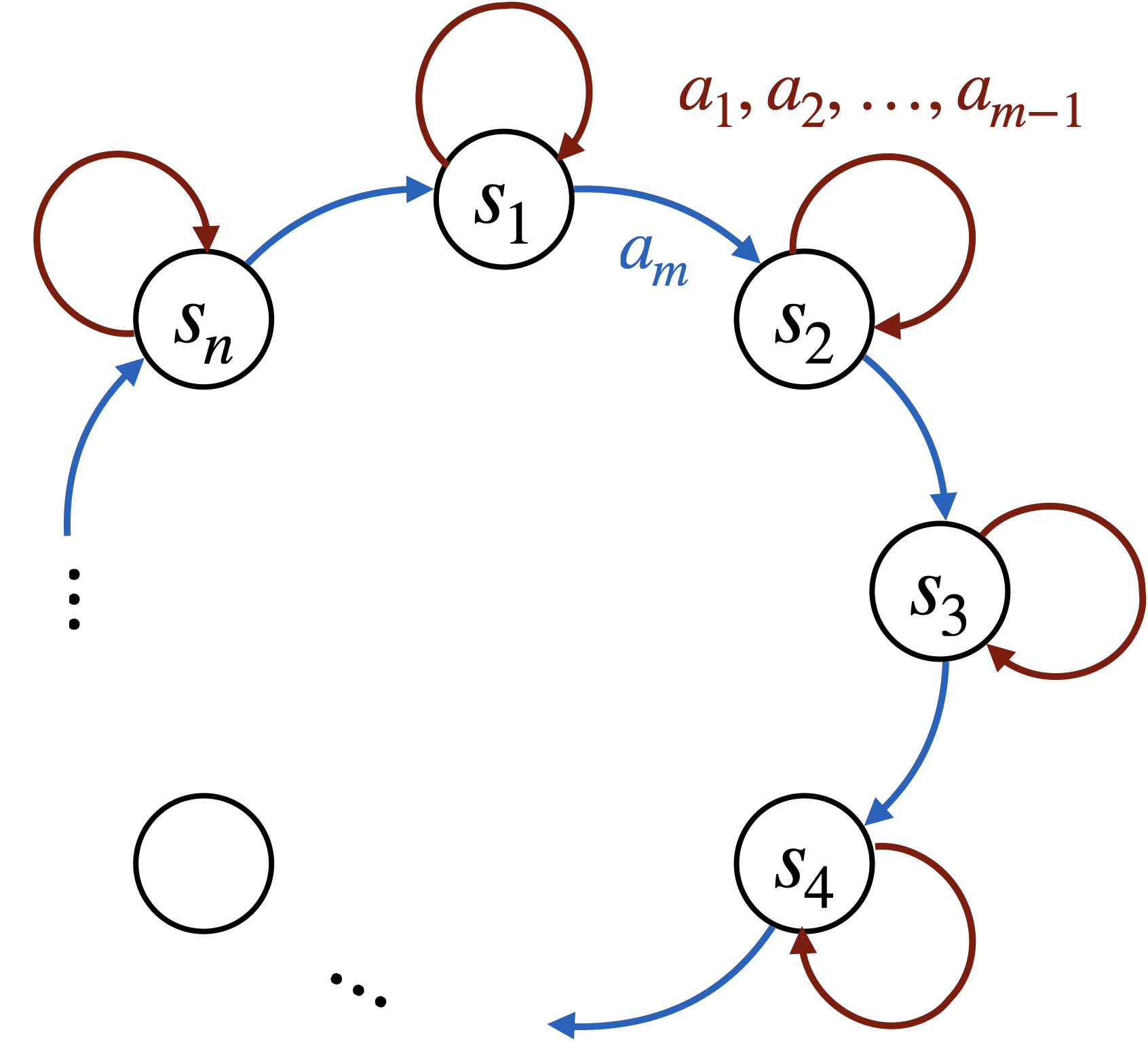}
        \caption{The MDP structure.}
        \label{fig:cyclic-mdp}
    \end{minipage}\hspace{2 mm}
    \begin{minipage}{0.42\textwidth}
        \centering
        \includegraphics[width=\linewidth]{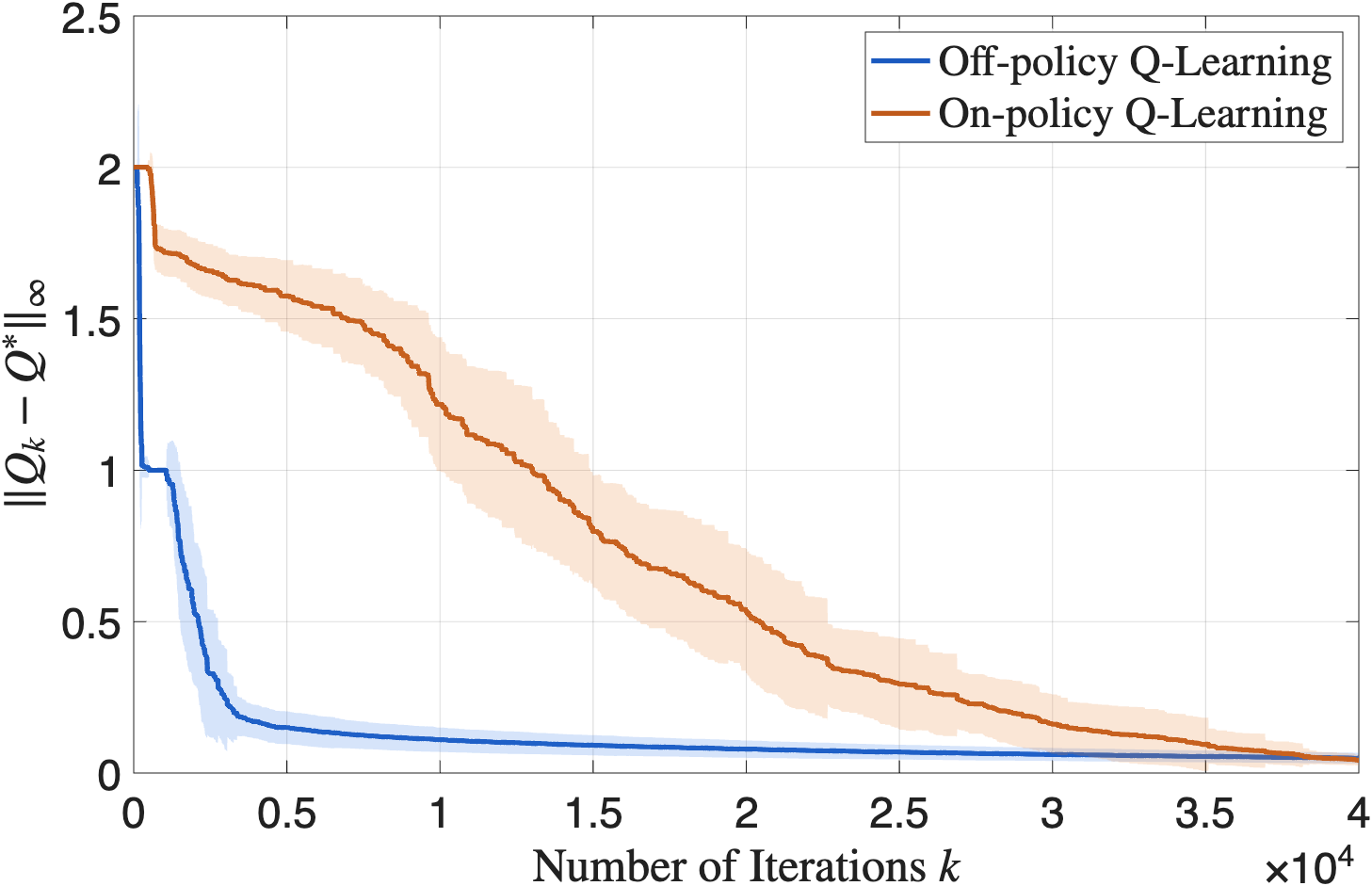}
        \caption{Convergence rates of $Q_k$.}
        \label{fig:QkQstar}
    \end{minipage}
    \vspace{-3 mm}
\end{figure}

This design yields a simple yet structured environment in which only the transition matrix corresponding to $a_m$ enables the agent to explore the entire state space. 
Note that the policy $\pi_b$ that deterministically selects $a_m$ for all states induces an irreducible Markov chain $\{S_k\}$ over $\mathcal{S}$, thereby satisfying Assumption~\ref{as:MC}. In this example, it can be easily verified that the optimal Q-function $Q^*$ satisfies $Q^*(s,\mathrm{\emph{stay}}) = 99$ and $Q^*(s,\mathrm{\emph{move}}) = 100$ for all $s \in \mathcal{S}$. 

\subsection{Convergence Rates: On-Policy Q-Learning vs. Off-Policy Q-Learning}
As indicated by the sample complexity result in Corollary~\ref{coro1}, due to exploration limitations, Q-learning with on-policy sampling is expected to converge more slowly than its off-policy counterpart. We next verify this observation numerically. In Algorithm~\ref{algo:Q-learning}, we choose the learning policy to be a convex combination of the uniform policy and the exponential softmax policy induced by $Q_k$, which corresponds to $\nu(\cdot)$ being the entropy function.

By running on-policy Q-learning (cf. Algorithm~\ref{algo:Q-learning}) with $\epsilon =0.1$ and $\tau = 0.1$ and initialization $Q_{0}(s, \mathrm{\emph{stay}}) = 100$ and $Q_{0}(s, \mathrm{\emph{move}}) = 90$, along with off-policy Q-learning using the same initialization and a uniform learning policy, we plot $\|Q_k - Q^*\|_\infty$ as a function of $k$ in Figure~\ref{fig:QkQstar}. 
It is evident that although both algorithms converge, on-policy Q-learning converges more slowly due to its inherent exploration challenges, whereas off-policy Q-learning does not suffer from such limitations. 
Moreover, because on-policy Q-learning employs rapidly time-varying stochastic policies, it exhibits a larger standard deviation. 
This phenomenon is consistent with and corroborates our theoretical results.

\subsection{Convergence Rates of the Learning Policies}

While Q-learning with on-policy sampling converges more slowly in terms of $\|Q_k - Q^*\|_\infty$, its advantage is that the learning policies gradually converge to an optimal one. Using the same MDP setup and algorithmic parameters, we plot $\|Q^{\pi_k} - Q^*\|_\infty$ as a function of $k$ for four choices of $(\epsilon,\tau)$ in Figure~\ref{fig:EETF}: (i) $\epsilon=0.01$, $\tau=0.1$, (ii) $\epsilon=0.05$, $\tau=0.1$, (iii) $\epsilon=0.08$, $\tau=0.1$, and (iv) $\epsilon=0.1$, $\tau=0.1$. We fix $\tau=0.1$ and vary $\epsilon$, since the softmax component is introduced to ensure Lipschitz continuity of the policy, whereas $\epsilon$ controls the level of exploration. As $\epsilon$ decreases, the convergence rate slows down while the asymptotic accuracy improves. This behavior is consistent with Theorem~\ref{thm2} and clearly illustrates the exploration--exploitation trade-off.

\begin{figure}[h]
    \centering
    \begin{minipage}{0.42\textwidth}
        \centering
        \includegraphics[width=\linewidth]{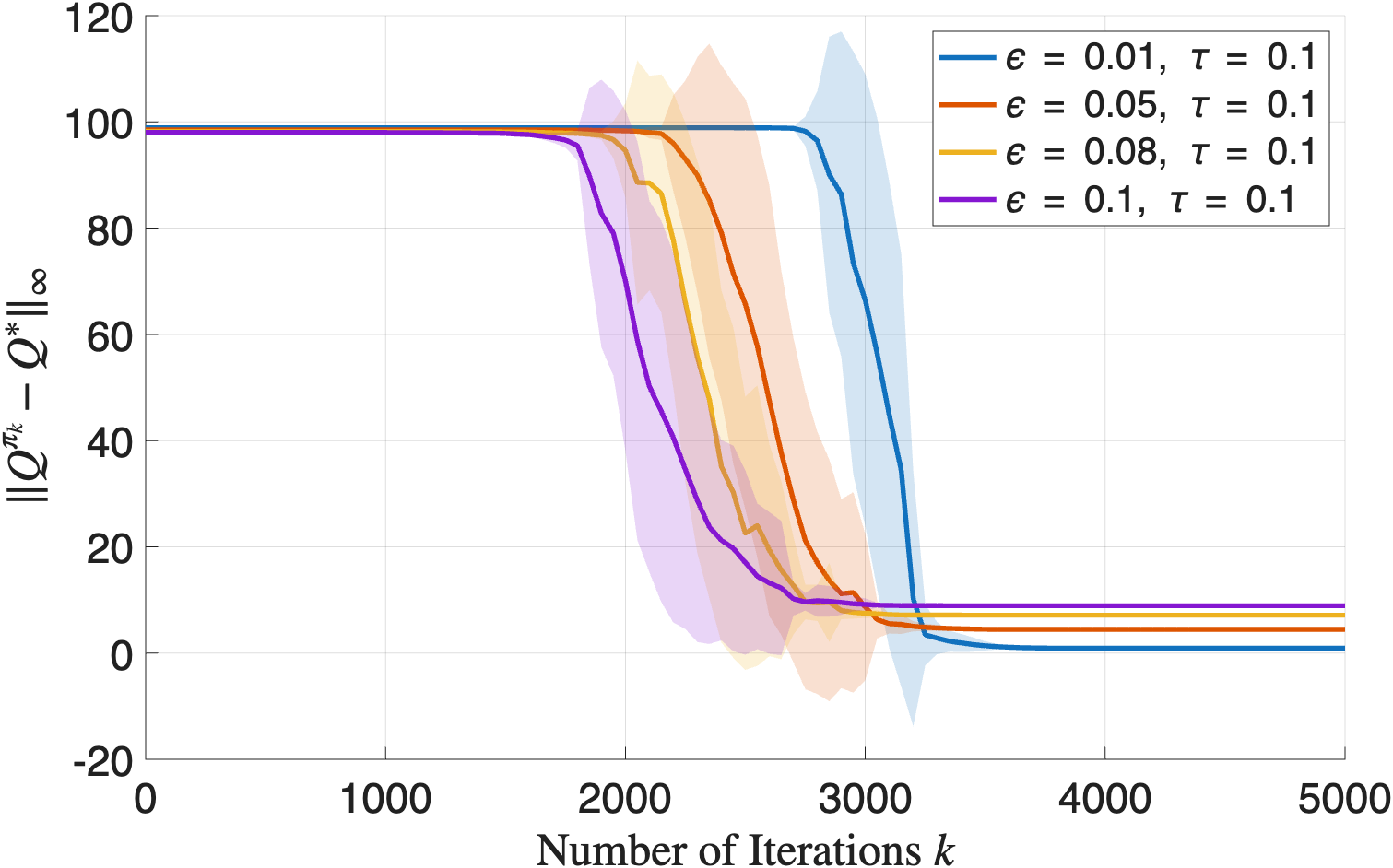}
        \caption{Convergence rates of $Q^{\pi_k}$.}
        \label{fig:Qpik}
    \end{minipage}\hspace{2 mm}
    \begin{minipage}{0.42\textwidth}
        \centering
        \includegraphics[width=\linewidth]{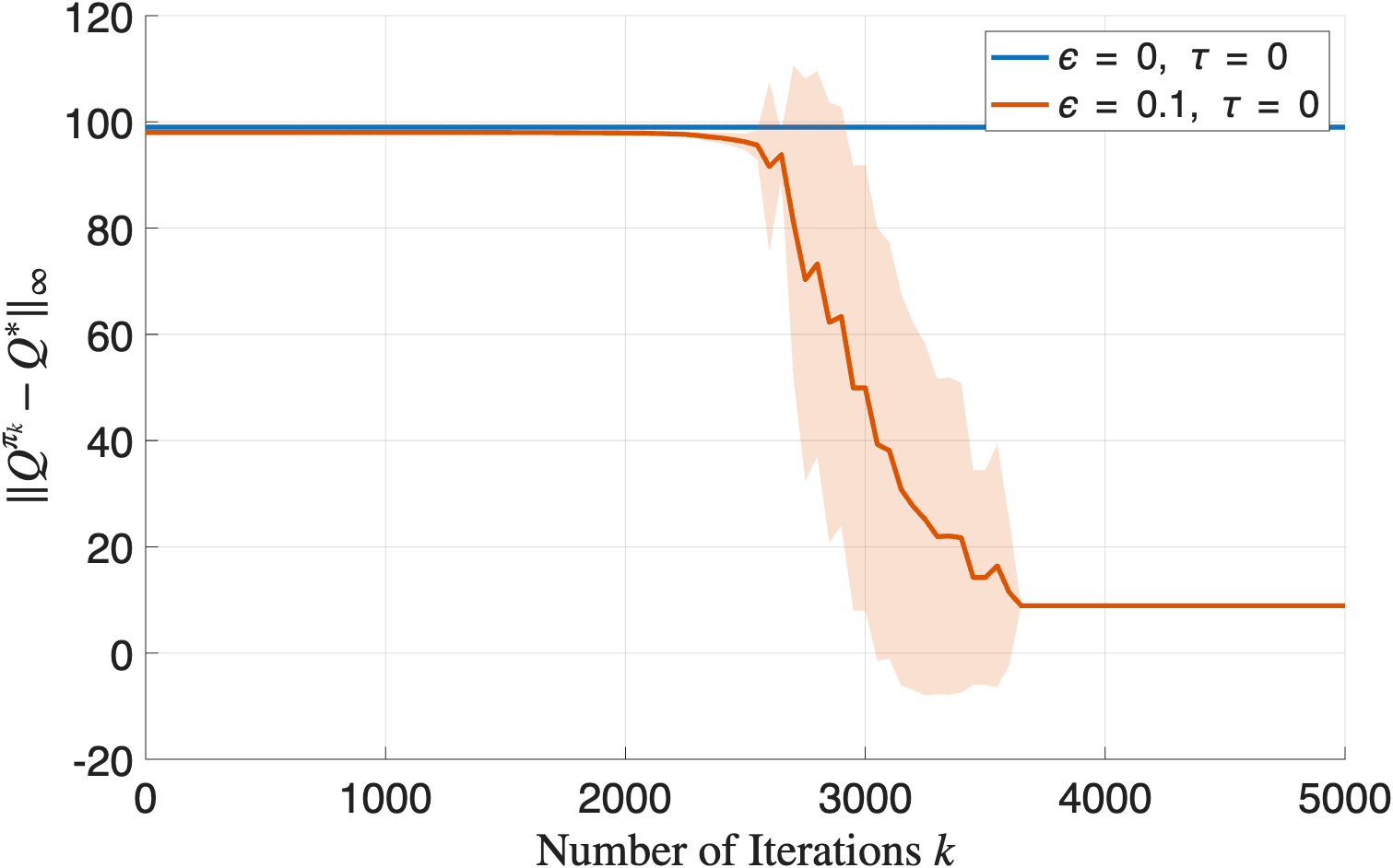}
        \caption{The exploration–exploitation trade-off.}
        \label{fig:EETF}
    \end{minipage}
    \vspace{-3 mm}
\end{figure}

Finally, we test the extreme case by setting $\tau=0$ in Figure~\ref{fig:EETF}. This case is not covered by our theory, since the resulting policy is no longer Lipschitz with respect to the Q-function. We observe that the algorithm still converges when $\epsilon=0.1$, which corresponds to the $\epsilon$-greedy policy. As discussed in Section~\ref{subsec:sarsa}, a theoretical analysis of Q-learning with an $\epsilon$-greedy learning policy is left as future work. In contrast, when $\tau=\epsilon=0$, the algorithm lacks any exploration component and does not converge.

\section{Conclusion}\label{sec:conclusion}
We present a finite-time analysis of Q-learning with rapidly time-varying learning policies under minimal assumptions. 
Our results show that although the algorithm achieves an $\mathcal{O}(1/\xi^2)$ sample complexity, its dependence on problem-specific constants is worse than that of off-policy Q-learning due to limited exploration. 
In contrast, Q-learning with on-policy sampling guarantees the convergence of the learning policy. 
From a technical standpoint, to address the challenge of time-inhomogeneous Markovian noise induced by time-varying learning policies and minimal structural assumptions, we develop an analytical framework based on the Poisson equation for Markov chain decomposition and characterize the properties of Poisson equation solutions through the analysis of the lazy chain. 
This framework for analyzing on-policy Q-learning can potentially be extended to a broader class of RL algorithms with time-varying learning policies.

As for future directions, an immediate one is to extend our analysis to $\epsilon$-greedy learning policies. More broadly, existing statistical lower bounds \cite{gheshlaghi2013minimax} are derived under the generative model, where i.i.d.\ samples from any state–action pair are available, and matching upper bounds for Q-learning are known in the off-policy setting under stationary, uniformly ergodic learning policies \cite{li2020sample}. However, a gap remains for practical RL algorithms with rapidly time-varying learning policies. Although this paper provides the first principled characterization in such a setting, it remains unclear what the corresponding lower bounds are, and whether both $\|Q_k - Q^*\|_\infty$ (favoring exploration) and $\|Q^{\pi_k} - Q^*\|_\infty$ (favoring exploitation) can achieve rates matching the statistical lower bound. Addressing this question constitutes a main direction for future work.

\bibliographystyle{apalike}
\bibliography{references}

\newpage
\begin{center}
    {\LARGE\bfseries Appendices}
\end{center}

\appendix

\section{Discussion of Assumption \ref{as:MC}}\label{ap:ass_necessity}

This section provides further discussion of Assumption~\ref{as:MC}. In particular, we compare it with the typical assumptions imposed in the literature on Q-learning.

There are generally two types of assumptions. The first type concerns the MDP structure, rather than any particular policy. Such assumptions (e.g., multichain, weakly communicating, communicating, unichain, or recurrent) are commonly imposed for \textit{average-reward} MDPs, where the structure of the solution to the Bellman equation critically depends on these conditions \cite[Figure~8.3.1]{puterman2014markov}. In contrast, for discounted MDPs with finite state–action spaces (and hence bounded rewards), as considered in this work, the Bellman operator is always a contraction mapping, and no additional structural assumptions are required for the Bellman equation to admit a unique solution. Therefore, our assumption does not fall into this category.

The second type of assumptions concerns the exploration capabilities of the learning policy (whether stationary or time-varying), to which our assumption belongs. Next, we discuss this type of assumption in the context of non-asymptoitic analysis of Q-learning.

\subsection{Off-Policy Q-Learning}

In the existing literature on off-policy Q-learning, where the learning policy $\pi$ is stationary, the following assumption is typically imposed.

\begin{assumption}\label{as:MC_off_Q}
The policy $\pi$ satisfies $\pi(a \mid s) > 0$ for all $(s,a)$, and the Markov chain $\{S_k\}$ induced by $\pi$ is irreducible and aperiodic.
\end{assumption}

\begin{remark}
Very recently, the aperiodicity requirement on $\{S_k\}$ was relaxed in \cite{haque2024stochastic,chandak2022concentration}.
\end{remark}

Since the learning policy is stationary, the sample trajectory $\{(S_k,A_k)\}$ generated by off-policy Q-learning forms a time-homogeneous Markov chain. To handle the resulting stochasticity, existing analyses typically rely on conditioning arguments that exploit the geometric mixing properties of $\{(S_k,A_k)\}$, which hold under Assumption~\ref{as:MC_off_Q} \cite{levin2017markov}; see, for example, \cite{srikant2019finite,chen2024lyapunov}.

While Assumption~\ref{as:MC} is comparable to Assumption~\ref{as:MC_off_Q}, the underlying sampling mechanisms are fundamentally different (on-policy versus off-policy). As a result, the technical challenges and analysis required in our setting are substantially different.

\subsection{On-Policy Q-Learning}
In the existing literature studying on-policy Q-learning, or more generally RL algorithms with \textit{time-varying} learning policies, the following assumption is commonly imposed \cite{khodadadian2021finite,chenziyi2022sample,chenzy2021sample,xu2021sample,wu2020finite,qiu2019finite,zhang2023convergence,liu2025linear,zou2019finite}. 

\begin{assumption}\label{as:existing}
For any policy $\pi$, the Markov chain $\{S_k\}$ induced by $\pi$ is irreducible and aperiodic, and hence admits a unique stationary distribution $\mu_\pi \in \Delta(\mathcal{S})$. Moreover, $\mu_{\min}:=\inf_\pi \min_{s,a} \mu_\pi(s) > 0$, and there exist constants $C>0$ and $\rho \in (0,1)$ such that
\begin{align*}
    \sup_{\pi}\max_{s \in \mathcal{S}}
    \left\| P_\pi^k(s,\cdot) - \mu_\pi(\cdot) \right\|_{\mathrm{TV}}
    \le C \rho^k .
\end{align*}
\end{assumption}
\begin{remark}
Sometimes Assumption \ref{as:existing} is relaxed by requiring the stated conditions to hold only for the sequence of policies generated along the algorithm trajectory, rather than for all policies. While this relaxation is often adopted in the literature, it introduces a form of circularity: verifying the required conditions typically depends on properties of the policy sequence itself, which in turn is generated by the algorithm whose convergence analysis relies on these conditions. As a result, such assumptions are often difficult to verify independently of the algorithm’s behavior.
\end{remark}

Under Assumption \ref{as:existing}, the challenge of exploration is effectively assumed away, since all states are visited sufficiently often regardless of the learning policies generated by on-policy Q-learning.

Compared with Assumption~\ref{as:existing}, our Assumption~\ref{as:MC} is significantly weaker. First, we assume only irreducibility and do not require aperiodicity; consequently, the Markov chain $\{S_k\}$ need not be mixing. Second, Assumption~\ref{as:MC} implies that any policy $\pi \in \Pi := \{\pi \mid \min_{s,a}\pi(a \mid s) > 0\}$ induces an irreducible Markov chain $\{S_k\}$ (cf.\ Lemma~\ref{le:policy_exploration}). As a result, the stationary distribution $\mu_\pi$ exists, is unique, and satisfies $\mu_\pi(s) > 0$ for all $s$ and all $k$. However, the quantity $\inf_{\pi \in \Pi} \min_{s \in \mathcal{S}} \mu_\pi(s)$ is, in general, not necessarily positive. A concrete example illustrating this phenomenon is presented below.

\begin{example}
\textit{Consider an MDP with state space $\mathcal{S}=\{s_0,s_1\}$ and action space $\mathcal{A}=\{a_0,a_1\}$. The transition dynamics are defined as follows: from state $s_0$, action $a_0$ keeps the chain at $s_0$ with probability $1$, while action $a_1$ moves the chain to $s_1$ with probability $1$; from state $s_1$, both actions $a_0$ and $a_1$ move the chain to $s_0$ with probability $1$. The reward function is defined as $\mathcal{R}(s_0,a_0)=1$ and $\mathcal{R}(s_0,a_1)=\mathcal{R}(s_1,a_0)=\mathcal{R}(s_1,a_1)=0$. Therefore, the set of optimal policies is $\Pi^*=\{\pi^*\mid \pi^*(a_0|s_0)=1\}$.}

\begin{center}
\begin{tikzpicture}[>=Stealth, thick]
\node[circle, draw, minimum size=0.8 cm, line width=1pt] (s0) at (0,0) {\Large $s_0$};
\node[circle, draw, minimum size=0.8 cm, line width=1pt] (s1) at (4,0) {\Large $s_1$};

\draw[->, line width=1pt] (s0) edge[loop left, looseness=12] node[left, xshift=-2pt] {$a_0,\ R=1$} (s0);

\draw[->, line width=1pt] (s0) to[bend left=12] node[above] {$a_1,\ R=0$} (s1);

\draw[->, line width=1pt] (s1) to[bend left=12] node[below] {$a_0,a_1,\ R=0$} (s0);

\end{tikzpicture}

\end{center}
\textit{ Let $\pi_b$ be defined as $\pi_b(a_1\mid s_0)=1/2$ and $\pi_b(a_0\mid s_0)=1/2$ (with arbitrary action selection at $s_1$). Then the induced Markov chain on states has transition matrix 
 \begin{align*}
     P_{\pi_b}=\begin{pmatrix}1/2 & 1/2\\ 1 & 0\end{pmatrix},
 \end{align*} 
 which is irreducible and hence admits a unique stationary distribution $\mu_{\pi_b}$. Therefore, Assumption \ref{as:MC} is satisfied.}

\textit{Now consider the family of stationary randomized policies $\{\pi^{(p)}\}_{p\in(0,1]}$ defined by $\pi^{(p)}(a_1\mid s_0)=p$ and $\pi^{(p)}(a_0\mid s_0)=1-p$ (again with arbitrary action selection at $s_1$). For any $p>0$, the induced chain is irreducible, since the transition $s_0\to s_1$ occurs with probability $p$ and the transition $s_1\to s_0$ occurs with probability $1$. Solving $\mu_{\pi^{(p)}}^\top =\mu_{\pi^{(p)}} P_{\pi^{(p)}}$ yields $\mu_{\pi^{(p)}}(s_0)=1/(1+p)$ and $\mu_{\pi^{(p)}}(s_1)=p/(1+p)$. Therefore, $\min_{s\in\mathcal{S}}\mu_{\pi^{(p)}}(s)=p/(1+p)\to 0$ as $p\to 0$, which implies that $\inf_{\pi\in \Pi}\min_{s\in\mathcal{S}}\mu_\pi(s)=0$.}
\end{example}

In general, under only our relaxed Assumption \ref{as:MC}, we can have
$\inf_{\pi\in \Pi}\min_{s\in\mathcal{S}}\mu_\pi(s)=0$, as illustrated in the previous example. Therefore, on-policy Q-learning must incorporate active exploration to ensure that all states are visited sufficiently often. In particular, we provide an explicit lower bound on $\min_{s\in\mathcal{S}} \mu_k(s)$ in terms of both algorithm-independent constants (e.g., $\mu_{b,\min}$, $\pi_{b,\min}$, $r_b$, $\delta_b$) and algorithm-dependent parameters (e.g., $\epsilon_k$ and $\tau_k$) in Lemma~\ref{le:implicit_to_explicit}. The dependence of $\min_{s\in\mathcal{S}} \mu_k(s)$ on $\epsilon_k$ and $\tau_k$ is the main source of the exploration--exploitation trade-off in on-policy Q-learning. Capturing this dependence within a minimal-assumption framework, therefore, constitutes a notable technical contribution of this work.

\section{Further Discussion on the Exploration-Exploitation Trade-Off}\label{ap:Q_to_Policy}

In Q-learning, there are typically three quantities of interest: (1) the iterate $Q_k$; (2) the learning policy $\pi_k$ and its associated Q-function $Q^{\pi_k}$; and (3) the policy $\bar{\pi}_k$ greedily induced by $Q_k$, that is,
$\{a\in\mathcal{A}\mid \bar{\pi}_k(a\mid s)>0\}\subseteq \argmax_{a\in\mathcal{A}} Q_k(s,a)$ for all $s\in\mathcal{S}$,
along with its associated Q-function $Q^{\bar{\pi}_k}$. We next discuss these quantities in the context of the exploration–exploitation trade-off in on-policy and off-policy Q-learning. The following lemma is needed for the discussion.

\begin{lemma}\label{le:Q-function-gap}
    For any $Q\in\mathbb{R}^{|\mathcal{S}||\mathcal{A}|}$, let $\bar{\pi}_Q$ be the policy greedily induced by $Q$, that is, $\{a\in\mathcal{A}\mid \bar{\pi}_Q(a|s)>0\}\subseteq \argmax_{a\in\mathcal{A}}Q(s,a)$ for all $s\in\mathcal{S}$. Then, we have
    \begin{align*}
        \|Q^{\bar{\pi}_Q}-Q^*\|_\infty\leq \frac{2\gamma}{1-\gamma}\|Q-Q^*\|_\infty.
    \end{align*}
\end{lemma}
\begin{proof}[Proof of Lemma \ref{le:Q-function-gap}]
Using the monotonicity, translation invariance, and contraction properties of the Bellman operators, we have
\begin{align*}
    Q^* - Q^{\bar{\pi}_Q}
    =\,& \mathcal{H}(Q^*) - \mathcal{H}^{\bar{\pi}_Q}(Q^{\bar{\pi}_Q}) \\
    =\,& \mathcal{H}(Q^*) - \mathcal{H}(Q)
        + \mathcal{H}(Q) - \mathcal{H}^{\bar{\pi}_Q}(Q^{\bar{\pi}_Q}) \\
    =\,& \mathcal{H}(Q^*) - \mathcal{H}(Q)
        + \mathcal{H}^{\bar{\pi}_Q}(Q) - \mathcal{H}^{\bar{\pi}_Q}(Q^{\bar{\pi}_Q}) \\
    \le\,& \mathcal{H}(Q^*) - \mathcal{H}(Q)
        + \mathcal{H}^{\bar{\pi}_Q}\!\left(Q^* + \|Q - Q^*\|_\infty \mathbf{1}\right)
        - \mathcal{H}^{\bar{\pi}_Q}(Q^{\bar{\pi}_Q}) \\
    =\,& \mathcal{H}(Q^*) - \mathcal{H}(Q)
        + \mathcal{H}^{\bar{\pi}_Q}(Q^*)
        - \mathcal{H}^{\bar{\pi}_Q}(Q^{\bar{\pi}_Q})
        + \gamma \|Q - Q^*\|_\infty \mathbf{1} \\
    \le\,& 2\gamma \|Q - Q^*\|_\infty \mathbf{1}
        + \gamma \|Q^{\bar{\pi}_Q} - Q^*\|_\infty \mathbf{1},
\end{align*}
which implies
\begin{align*}
    \|Q^* - Q^{\bar{\pi}_Q}\|_\infty
    \le \frac{2\gamma}{1-\gamma}\,\|Q - Q^*\|_\infty.
\end{align*}

\end{proof}

\subsection{Off-Policy Q-Learning}
In off-policy Q-learning, the learning policy is stationary, that is, $\pi_k \equiv \pi$ for some fixed policy $\pi$ and all $k \ge 0$. Consequently, $\|Q^{\pi_k}-Q^*\|_\infty$ remains constant over time. This setting has been studied extensively in the literature. In particular, the sample complexity to achieve $\|Q_k-Q^*\|_\infty \le \epsilon$ (either in expectation or with high probability) is $\tilde{\mathcal{O}}(\epsilon^{-2}\bar{\mu}_{\pi,\min}^{-1})$ \cite{li2020sample,chen2024lyapunov}, where $\bar{\mu}_{\pi,\min}$ denotes the minimum entry of the stationary distribution of the Markov chain $\{(S_k,A_k)\}$ induced by $\pi$. We ignore the dependence on the effective horizon $1/(1-\gamma)$ here, as it is not central to the exploration--exploitation discussion.

In light of Lemma~\ref{le:Q-function-gap}, the sample complexity to achieve $\|Q^{\bar{\pi}_k}-Q^*\|_\infty \le \epsilon$ is also $\tilde{\mathcal{O}}(\epsilon^{-2}\bar{\mu}_{\pi,\min}^{-1})$. Therefore, to ensure that $\|Q_k-Q^*\|_\infty$ (or equivalently, $\|Q^{\bar{\pi}_k}-Q^*\|_\infty$) is small, the learning policy $\pi$ should be chosen to maximize $\bar{\mu}_{\pi,\min}$, ideally achieving a uniform distribution when possible, in which case $\bar{\mu}_{\pi,\min} = (|\mathcal{S}||\mathcal{A}|)^{-1}$. This corresponds to a pure exploration regime, as off-policy Q-learning does not account for the quality of the learning policy (relative to the optimal one) during the learning process.

\subsection{On-Policy Q-Learning}
In on-policy Q-learning, the learning policy $\pi_k$ is time-varying. To ensure that both $Q_k \to Q^*$ and $Q^{\pi_k} \to Q^*$, there is a clear exploration--exploitation trade-off. In particular, the convergence rate of $\|Q_k - Q^*\|_\infty$ (and also $\|Q^{\bar{\pi}_k} - Q^*\|_\infty$ by Lemma~\ref{le:Q-function-gap}) requires all state--action pairs to be visited sufficiently often, and therefore favors exploration, similar to off-policy Q-learning. In contrast, to ensure $Q^{\pi_k} \to Q^*$, the learning policy $\pi_k$ must be close to the greedy policy with respect to $Q_k$, which favors exploitation. Together, these requirements characterize the exploration--exploitation trade-off in on-policy Q-learning, as quantified in Theorem~\ref{thm2}.

Intuitively, drawing an analogy to the bandit setting, off-policy Q-learning resembles a best-arm identification procedure: each arm is sampled uniformly to estimate empirical rewards (i.e., $Q_k$), and in the end, a policy is obtained by selecting the action with the largest empirical reward (i.e., $\bar{\pi}_k$). In contrast, on-policy Q-learning is closer in spirit to online learning algorithms such as UCB and Thompson sampling, where a careful balance between exploration and exploitation is required, although the performance metric here is last-iterate convergence rather than regret.

Based on the above discussion, in the simulations comparing Q-learning with on-policy and off-policy sampling in Section~\ref{sec:numerical}, we ensure fairness by comparing $\|Q_k - Q^*\|_\infty$ for both algorithms (cf.~Figure~\ref{fig:QkQstar}), which is equivalent to comparing $\|Q^{\bar{\pi}_k} - Q^*\|_\infty$ by Lemma~\ref{le:Q-function-gap}, and by comparing $\|Q^{\pi_k} - Q^*\|_\infty$ for both algorithms (cf.~Figure~\ref{fig:Qpik}), where we omit the constant curve for off-policy Q-learning. In contrast, directly comparing $Q^{\bar{\pi}_k}$ from off-policy Q-learning with $Q^{\pi_k}$ from on-policy Q-learning would not be fair, since the former requires only exploration, whereas the latter requires balancing the exploration--exploitation trade-off.

\section{Proofs of All Technical Results in Section \ref{sec:main results}}
\subsection{Assuming $\pi_b(a\mid s)>0$ for all $(s,a)$ is without loss of generality}\label{ap:behavior_policy}

\begin{lemma}\label{le:pib_pib'}
    The following two statements are equivalent:
\begin{enumerate}[(1)]
    \item There exists a policy $\pi_b$ such that the Markov chain $\{S_k\}$ induced by $\pi_b$ is irreducible.
    \item There exists a policy $\pi_b'$ satisfying $\pi_b'(a \mid s) > 0$ for all $(s,a)$ such that the Markov chain $\{S_k\}$ induced by $\pi_b'$ is irreducible.
\end{enumerate}
\end{lemma}
\begin{proof}
    The implication $(2) \Rightarrow (1)$ is immediate. We now prove $(1) \Rightarrow (2)$. For any $(s,s')$, we have
\begin{align*}
    P_{\pi_b'}(s,s')
    =\,& \sum_{a\in\mathcal{A}} p(s' \mid s,a)\,\pi_b'(a \mid s) \\
    \ge\,& \sum_{a\in\mathcal{A}:\,\pi_b(a \mid s)>0} p(s' \mid s,a)\,\pi_b'(a \mid s) \\
    =\,& \sum_{a\in\mathcal{A}:\,\pi_b(a \mid s)>0} p(s' \mid s,a)\,\pi_b(a \mid s)\,
    \frac{\pi_b'(a \mid s)}{\pi_b(a \mid s)} \\
    \ge\,& \sum_{a\in\mathcal{A}:\,\pi_b(a \mid s)>0} p(s' \mid s,a)\,\pi_b(a \mid s)
    \cdot \left( \min_{s,a:\,\pi_b(a \mid s)>0} \frac{\pi_b'(a \mid s)}{\pi_b(a \mid s)} \right) \\
    =\,& \sum_{a\in\mathcal{A}} p(s' \mid s,a)\,\pi_b(a \mid s)
    \cdot \left( \min_{s,a:\,\pi_b(a \mid s)>0} \frac{\pi_b'(a \mid s)}{\pi_b(a \mid s)} \right) \\
    =\,& \delta' \, P_{\pi_b}(s,s'),
\end{align*}
where $\delta' := \min_{s,a:\,\pi_b(a \mid s)>0} \pi_b'(a \mid s)/\pi_b(a \mid s)$.
The inequality above implies that $P_{\pi_b'} \ge \delta' P_{\pi_b}$. Since $P_{\pi_b}$ is irreducible, for any $(s,s')$ there exists $k>0$ such that $P_{\pi_b}^k(s,s')>0$. For the same $k$, we have
\begin{align*}
    P_{\pi_b'}^k(s,s') \ge {\delta'}^{k} P_{\pi_b}^k(s,s') > 0,
\end{align*}
which implies that the Markov chain $\{S_k\}$ induced by $\pi_b'$ is also irreducible.
\end{proof}

\subsection{Proof of Corollary \ref{coro1}}\label{pf:coro1}
For a given $\xi > 0$, to ensure $\mathbb{E}[\|Q_k - Q^*\|_\infty] \leq \xi$, by Jensen's inequality, it suffices to guarantee that $\mathbb{E}[\|Q_k - Q^*\|_\infty^2] \leq \xi^2$.
Using Theorem \ref{thm1}, it is enough to have
\begin{align*}
    3\|Q_0 - Q^*\|_\infty^2 \left(1 - \alpha c_{1} \right)^{k}  + c_2 \alpha + c_3\alpha^{2} \log^{4}(c_4/\alpha)  \leq \xi^{2}.
\end{align*}
Ignoring the logarithmic factor and using the numerical inequality $1 + x \leq e^{x}$ for all $x \in \mathbb{R}$, it is then sufficient to have 
\begin{align*}
   3\|Q_0 - Q^*\|_\infty^2 e^{-\alpha c_1 k}  + c_2 \alpha + c_3\alpha^{2}  \leq \xi^{2}.
\end{align*}
To achieve the above, we make each term on the left-hand side less than $\xi^{2}/3$. 
Since the second and third terms are independent of $k$, we first control those. 
Precisely, we choose $\alpha$ such that 
\begin{align*}
     c_{2} \alpha \leq \frac{\xi^{2}}{3} \quad & \text{and} \quad c_{3} \alpha^{2} \leq \frac{\xi^{2}}{3}\;\Rightarrow\; \alpha\leq \min\left(\frac{\xi^{2}}{3 c_{2}},\frac{\xi}{\sqrt{3 c_{3}}}\right)\;\Rightarrow\;\frac{1}{\alpha}\geq \max\left(\frac{3c_2}{\xi^2},\frac{\sqrt{3 c_{3}}}{\xi}\right).
\end{align*}
With this choice of $\alpha$, we need to choose $k$ such that $3\|Q_0 - Q^*\|_\infty^2 e^{-k c_{1} \alpha} \leq \xi^2/3$:
\begin{align*}
    k \geq \frac{2\log\left(3\|Q_0 - Q^*\|_\infty/\xi\right)}{c_1\alpha}\geq \frac{2\log\left(3\|Q_0 - Q^*\|_\infty/\xi\right)}{c_1}\max\left(\frac{3c_2}{\xi^2},\frac{\sqrt{3 c_{3}}}{\xi}\right).
\end{align*}
Finally, recall that 
\begin{align*}
        c_{1} =\,& \frac{1}{2}\lambda^{r_b}\mu_{\pi_b,\min}\delta_b(1-\gamma),\quad c_{2} = \frac{c_2'(r_b+1)\log(|\mathcal{S}||\mathcal{A}|) }{ \lambda^{3r_{b}+1}\pi_{b,\min}\mu_{\pi_b,\min}^3\delta_b^3(1-\gamma)^4},\\
        c_{3} =\,& \frac{c_3'(r_b+1)^4}{\tau^2\lambda^{6r_b+4}\mu_{\pi_b,\min}^6\pi_{b,\min}^4\delta_b^6(1-\gamma)^6}.
    \end{align*}
    Altogether, the sample complexity to achieve $\mathbb{E}[\|Q_k - Q^*\|_\infty] \leq \xi$ is
    \begin{align*}        \mathcal{O}\left(\frac{(r_b+1)\log\left(3\|Q_0 - Q^*\|_\infty/\xi\right)}{\lambda^{4r_b+2}\mu_{\pi_b,\min}^4\pi_{b,\min}\delta_b^4(1-\gamma)^4}\max\left(\frac{\log(|\mathcal{S}||\mathcal{A}|) }{ (1-\gamma)\xi^2},\frac{r_b+1}{\tau\lambda\pi_{b,\min}\xi}\right)\right).
    \end{align*}

\subsection{Proof of Proposition \ref{prop:ass_necessity}}\label{pf:prop:ass_necessity}
    Since $\{S_k\}$ is a finite Markov chain and is not irreducible, there exists a proper subset $\mathcal{C}$ of $\mathcal{S}$ such that $P(s,s')=0$ for any $s\in\mathcal{C}$ and $s'\in \mathcal{S}\setminus \mathcal{C}$ \cite[Section 1.2]{norris1998markov}. Therefore, when initializing $S_0\in\mathcal{C}$, we have $\mathds{1}_{\{S_k=s\}}=0$ for any $s\in\mathcal{S}\setminus \mathcal{C}$. In view of the update rule (\ref{eq:Q_example}), we have, with probability one,
\begin{align*}
    \|Q_k-Q^*\|_\infty
    \geq
    \max_{s\in\mathcal{S}\setminus \mathcal{C}}|Q_k(s)-Q^*(s)|
    =
    \max_{s\in\mathcal{S}\setminus \mathcal{C}}|Q_0(s)-Q^*(s)|.
\end{align*}
Therefore, as long as there exists $s' \in \mathcal{S}\setminus \mathcal{C}$ such that $Q_0(s') \neq Q^*(s')$, we have $\|Q_k - Q^*\|_\infty \ge c := |Q_0(s') - Q^*(s')|$ with probability one for all $k \ge 0$.

\section{Proofs of All Technical Results in Section \ref{sec:proof1}}
\subsection{Proof of Lemma \ref{le:policy_exploration}}\label{pf:le:policy_exploration}
Lemma~\ref{pf:le:policy_exploration} follows directly as a corollary of Lemma~\ref{le:pib_pib'}.

\subsection{Proof of Lemma \ref{le:f-bar}}\label{pf:le:f-bar}
\begin{enumerate}[(1)]
    \item By definition of $\Bar{F}(\cdot)$, for any $(s,a)$, we have 
    \begin{align*}
        [\Bar{F}(Q,\pi)](s,a) & =
        \mathbb{E}_{Y \sim \Bar{\mu}_\pi}[F(Q,Y)(s,a)] \\
        & = \mu_{\pi}(s)\pi(a\rvert s) \left(\mathcal{R}(s,a) 
        + \gamma \sum_{s'\in\mathcal{S}} p(s'\rvert s, a) \max_{a'\in\mathcal{A}} Q(s',a') 
        -Q(s,a)\right) + Q(s,a) \\
        & = \mu_{\pi}(s)\pi(a\rvert s) \left([\mathcal{H}(Q)](s,a)
        -Q(s,a)\right) + Q(s,a) \\
        & = (1-D_\pi(s,a))Q(s,a)+D_\pi(s,a)[\mathcal{H}(Q)](s,a).
    \end{align*}
    It follows that
    \begin{align*}
    \Bar{F}(Q,\pi)=[(I-D_\pi)+D_\pi\mathcal{H}](Q),\quad \forall\,Q\in\mathbb{R}^{|\mathcal{S}||\mathcal{A}|}.
\end{align*}
\item Since the Bellman operator $\mathcal{H}(\cdot)$ is a $\gamma$-contraction with respect to $\|\cdot\|_\infty$, it follows—by the same reasoning as in the proof of \cite[Proposition 5 (3)(b)]{chen2024lyapunov}—that the operator $\Bar{F}(\cdot, \pi)$ is a $\gamma_\pi$-contraction with respect to $\|\cdot\|_\infty$. 
As a result, we have
\begin{align*}
    \|\bar{F}(Q_1,\pi)\|_\infty=\|\bar{F}(Q_1,\pi)-\bar{F}(0,\pi)\|_\infty+\|\bar{F}(0,\pi)\|_\infty
    \leq \|Q_1\|_\infty+1,
\end{align*}
where the last inequality follows from $\|\bar{F}(0,\pi)\|_\infty\leq  \max_{s,a}|\mathcal{R}(s,a)|\leq 1$.
\item Since $\mathcal{H}(Q^*) = Q^*$, we have
\begin{align*}
    \Bar{F}(Q^*,\pi) = \left[(I - D_\pi) + D_\pi \mathcal{H}\right](Q^*) = (I - D_\pi) Q^* + D_\pi Q^* = Q^*.
\end{align*}
The uniqueness follows from $\Bar{F}(\cdot,\pi)$ being a contraction mapping \cite{banach1922operations}.

\item Using the definition of $\Bar{F}(\cdot)$, we have
\begin{align*}
    &\|\Bar{F}(Q_{1},\pi_{1}) - \Bar{F}(Q_{2},\pi_{2})\|_{\infty} \\
    =\,&\left\|Q_1 + D_{\pi_1}(\mathcal{H}(Q_1) - Q_1) - Q_2 - D_{\pi_2}(\mathcal{H}(Q_2) - Q_2)\right\|_\infty \\
    \leq\,&\|Q_1 - Q_2\|_\infty + \left\|D_{\pi_1}(\mathcal{H}(Q_1) - Q_1) - D_{\pi_2}(\mathcal{H}(Q_2) - Q_2)\right\|_\infty \\
    \leq\,&\|Q_1 - Q_2\|_\infty + \left\|(D_{\pi_1} - D_{\pi_2})(\mathcal{H}(Q_1) - Q_1)\right\|_\infty \\
    &+ \left\|D_{\pi_2} \left(\mathcal{H}(Q_1) - \mathcal{H}(Q_2) - Q_1 + Q_2\right)\right\|_\infty \\
    \leq\,&\|Q_1 - Q_2\|_\infty + \|D_{\pi_1} - D_{\pi_2}\|_\infty \|\mathcal{H}(Q_1) - Q_1\|_\infty \\
    &+ \|D_{\pi_2}\|_\infty \|\mathcal{H}(Q_1) - \mathcal{H}(Q_2)\|_\infty + \|D_{\pi_2}\|_\infty \|Q_1 - Q_2\|_\infty,
\end{align*}
where the last inequality follows from the definition of induced matrix norms and the triangle inequality. To proceed, we have the following observations:
    \begin{align*}        \|D_{\pi_2}\|_\infty=\,&\max_{s,a}\mu_{\pi_2}(s)\pi_2(a\mid s)\leq 1,\\
        \|D_{\pi_1}-D_{\pi_2}\|_\infty
        =\,&\|\bar{\mu}_{\pi_1}-\bar{\mu}_{\pi_2}\|_\infty,\\
        \|\mathcal{H}(Q_1)-Q_1\|_\infty\leq\,&\|\mathcal{H}(Q_1)\|_\infty+\|Q_1\|_\infty
        \leq \frac{2}{1-\gamma},\\
        \|\mathcal{H}(Q_1)-\mathcal{H}(Q_2)\|_\infty\leq \,&\gamma\|Q_1-Q_2\|_\infty\leq \|Q_1-Q_2\|_\infty.
    \end{align*}
    It follows that
    \begin{align*}
        \|\Bar{F}(Q_{1},\pi_{1}) - \Bar{F}(Q_{2},\pi_{2})\|_{\infty}
        \leq \,&(1+\|D_{\pi_2}\|_\infty)\|Q_1-Q_2\|_\infty+\|D_{\pi_1}-D_{\pi_2}\|_\infty\|\mathcal{H}(Q_1)-Q_1\|_\infty\\
        &+\|D_{\pi_2}\|_\infty\|\mathcal{H}(Q_1)-\mathcal{H}(Q_2)\|_\infty \\
        \leq \,& 3 \|Q_1-Q_2\|_\infty + \frac{2}{1-\gamma}\|\bar{\mu}_{\pi_1}-\bar{\mu}_{\pi_2}\|_\infty.
    \end{align*}
\end{enumerate}
\subsection{Proof of Lemma \ref{le:F-lipschitz}}\label{pf:le:F-lipschitz}
\begin{enumerate}[(1)]
    \item For any $(s,a)$, by the definition of $F(\cdot)$, we have
    \begin{align*}
        & \lvert [F(Q_{1},y)](s,a) - [F(Q_{2},y)](s,a) \rvert \\
         \leq \,& \gamma \mathds{1}_{\{(s_0,a_0)=(s,a)\}}    \left|\sum_{s'\in\mathcal{S}} p(s'\rvert s, a) \max_{a'\in\mathcal{A}} Q_{1}(s',a')  
         - \sum_{s'\in\mathcal{S}} p(s'\rvert s, a) \max_{a'\in\mathcal{A}} Q_{2}(s',a') \right|\\
         &+(1-\mathds{1}_{\{(s_0,a_0)=(s,a)\}})|Q_1(s,a)-Q_2(s,a)|\\
         \leq \,& \gamma \mathds{1}_{\{(s_0,a_0)=(s,a)\}} \sum_{s'\in\mathcal{S}} p(s'|s,a) \left| \max_{a'\in\mathcal{A}} Q_{1}(s',a') - \max_{a'\in\mathcal{A}} Q_{2}(s',a') \right| \\
         &+(1-\mathds{1}_{\{(s_0,a_0)=(s,a)\}})\|Q_1-Q_2\|_\infty\\
         \leq \; & \gamma \mathds{1}_{\{(s_0,a_0)=(s,a)\}} \|Q_{1} - Q_{2}\|_{\infty} + (1-\mathds{1}_{\{(s_0,a_0)=(s,a)\}})\|Q_1-Q_2\|_\infty \\
        \leq \,& \|Q_{1} - Q_{2}\|_{\infty}.
    \end{align*}
Since the right-hand side of the previous inequality does not depend on $(s,a)$, we have
\begin{align*}
\|F(Q_{1},y) - F(Q_{2},y)\|_{\infty} \leq \|Q_{1} - Q_{2}\|_{\infty}.    
\end{align*}
\item For any  $(s,a)$, we have
\begin{align*}
    & \left| [F(Q_1,y)](s,a) - [\bar{F}(Q_1,\pi)](s,a) \right| \\
    & \; = \left| \mathds{1}_{\{(s,a)=(s_0,a_0)\}}-D_\pi(s,a) \right| \left| \mathcal{R}(s,a) 
        + \gamma \sum_{s'\in\mathcal{S}} p(s'\rvert s, a) \max_{a'\in\mathcal{A}} Q_{1}(s',a') 
        -Q_{1}(s,a) \right| \\
    & \; \leq \left| \mathcal{R}(s,a) 
        + \gamma \sum_{s'\in\mathcal{S}} p(s'\rvert s, a) \max_{a'\in\mathcal{A}} Q_{1}(s',a') 
        -Q_{1}(s,a) \right| \\
    & \; \leq 1 + \gamma \|Q_{1}\|_{\infty} + \|Q_{1}\|_{\infty} \\
    & \; \leq 1 + \frac{\gamma}{1-\gamma} + \frac{1}{1-\gamma} \\
    & \; = \frac{2}{1-\gamma}.
\end{align*}
Since the above inequality holds for any $(s,a)$, we have
\begin{align*}
    \|F(Q_1,y) - \bar{F}(Q_1,\pi)\|_{\infty} \leq \frac{2}{1-\gamma}.
\end{align*}
\end{enumerate}

\subsection{Proof of Lemma \ref{le:negDrift}}\label{pf:le:negDrift}
    Since $Q^\ast$ is the unique fixed point of $\Bar{F}(\cdot,\pi_k)$ for any $k$ (cf. Lemma \ref{le:f-bar} (3)), we have 
    \begin{align}
        &\langle \nabla M(Q_k-Q^*),\Bar{F}(Q_k,\pi_k)-Q_k\rangle \nonumber\\
         = \,&\langle \nabla M(Q_k-Q^*),\Bar{F}(Q_k,\pi_k) - \Bar{F}(Q^\ast,\pi_k) + Q^\ast - Q_k \rangle \nonumber\\
         =\,& \langle \nabla M(Q_k-Q^*),\Bar{F}(Q_k,\pi_k) - \Bar{F}(Q^\ast,\pi_k) \rangle 
        - \langle \nabla M(Q_k-Q^*), Q_{k} - Q^\ast \rangle.\label{eq1:pf:le:negDrift}
    \end{align}
    By Lemma \ref{le:Moreau}, we have 
    \begin{align*}
        &\langle \nabla M(Q_k-Q^*),\Bar{F}(Q_k,\pi_k) - \Bar{F}(Q^\ast,\pi_k) \rangle \\
         =\,&\|Q_k-Q^\ast\|_{m}\langle \nabla \|Q_k-Q^\ast\|_{m},\Bar{F}(Q_k,\pi_k) - \Bar{F}(Q^\ast,\pi_k) \rangle \\
         \leq\,& \|Q_k-Q^\ast\|_{m} \left\|\nabla \|Q_k-Q^\ast\|_{m}\right\|_{m}^* \|\Bar{F}(Q_k,\pi_k) - \Bar{F}(Q^\ast,\pi_k)\|_{m}\tag{$\|\cdot\|_{m}^*$ is the dual norm of $\|\cdot\|_{m}$}\\
          \leq\,& \frac{1}{\ell_m} \|Q_k-Q^\ast\|_{m} \left\|\nabla \|Q_k-Q^\ast\|_{m}\right\|_{m}^*\|\Bar{F}(Q_k,\pi_k) - \Bar{F}(Q^\ast,\pi_k)\|_{\infty}\\
         \leq\,& \frac{\gamma_k}{\ell_m} \|Q_k-Q^\ast\|_{m} \left\|\nabla \|Q_k-Q^\ast\|_{m}\right\|_{m}^*\|Q_k - Q^\ast\|_{\infty} 
        \tag{Lemma \ref{le:f-bar} (2)}\\
        \leq\,&\gamma_{k}\frac{u_{m}}{\ell_m}\|Q_k-Q^\ast\|_{m}^2 \left\|\nabla \|Q_k-Q^\ast\|_{m}\right\|_{m}^*\\
        =\,& 2 \gamma_k\frac{u_{m}}{\ell_m}M(Q_k-Q^\ast) \left\|\nabla \|Q_k-Q^\ast\|_{m}\right\|_{m}^*.
    \end{align*}
    To bound $\left\|\nabla \|Q_k-Q^*\|_m\right\|_{m}^*$, we use the following result from \cite{shalev2012online}. 
    \begin{lemma}\label{lem:shalev}
        Let $f:\mathcal{X}\to\mathbb{R}$ be a convex differentiable function. Then, $f$ is $L$-Lipschitz over $\mathcal{X}$ with respect to some norm 
        $\|\cdot\|$, if and only if  
        $\sup_{x\in\mathcal{X}}\|\nabla f(x)\|_{\ast} \leq L$, where $\|\cdot\|_{\ast}$ is the dual norm of $\|\cdot\|$. 
    \end{lemma}
    Since for any $Q_1,Q_2$, we have by the triangle inequality that
    \begin{align*}
        |\|Q_1\|_{m} - \|Q_2\|_{m}| \leq \|Q_1 - Q_1\|_{m}, 
    \end{align*}
    the function $\|Q\|_{m}$ is $1$-Lipschitz with respect to $\|\cdot\|_{m}$. 
    Therefore, by Lemma \ref{lem:shalev}, we have $\|\nabla \|Q_k-Q^*\|_m\|_m^* \leq 1$, 
    and consequently, 
    \begin{align}\label{eq2:pf:le:negDrift}
        \langle \nabla M(Q_k-Q^*),\Bar{F}(Q_k,\pi_k) - \Bar{F}(Q^\ast,\pi_k) \rangle  
        \leq  2 \gamma_k\frac{u_{m}}{\ell_m} M(Q_k - Q^\ast). 
    \end{align}
    
    Next, we bound the term $\langle \nabla M(Q_k-Q^*), Q_{k} - Q^\ast \rangle$ (on the right-hand side of Eq. \eqref{eq1:pf:le:negDrift}) from below. Using Lemma \ref{le:Moreau}, we have
    \begin{align*}
        \langle \nabla M(Q_k-Q^*), Q_{k} - Q^\ast \rangle 
        = \|Q_k-Q^\ast\|_{m} \left\langle \nabla \|Q_k-Q^\ast\|_{m}, Q_{k} - Q^\ast \right\rangle.
    \end{align*}
    Since $\|Q\|_{m}$ is a convex function, we have
    \begin{align*}
         \|0\|_{m} \geq \,&\|Q_{k} - Q^{\ast}\|_{m} + \left\langle \nabla \|Q_{k} - Q^{\ast}\|_{m}, Q^*-Q_{k} \right\rangle \\
        \Longrightarrow \quad \|Q_{k} - Q^{\ast}\|_{m}  \leq\,& \left\langle \nabla \|Q_{k} - Q^{\ast}\|_{m}, Q_{k}-Q^* \right\rangle.
    \end{align*}
    As a result, we have 
    \begin{align*}
        \langle \nabla M(Q_k-Q^*), Q_{k} - Q^\ast \rangle 
        \geq \|Q_{k} - Q^{\ast}\|_{m}^2=2M(Q_k-Q^*).
    \end{align*}
    Using the previous inequality and Eq. (\ref{eq2:pf:le:negDrift}) in Eq. (\ref{eq1:pf:le:negDrift}), we have
    \begin{align*}
        \langle \nabla M(Q_k-Q^*),\Bar{F}(Q_k,\pi_k)-Q_k\rangle \leq  -2\left(1 - \gamma_k\frac{u_{m}}{\ell_m}\right)M(Q_k - Q^\ast). 
    \end{align*}
    Taking expectations on both sides of the previous inequality gives
\begin{align*}
    E_1 \;\leq\; -2\left(1 - \frac{u_m}{\ell_m} \gamma_k \right)\, \mathbb{E}\!\left[M(Q_k - Q^*)\right].
\end{align*}

\subsection{Proof of Lemma \ref{le:E_3}}\label{pf:le:E_3}
Recall that $\mathcal{F}_k$ is the $\sigma$-algebra generated by $\{Y_0,Y_1,\cdots,Y_k\}$. Since both  $Q_k$ and $\pi_k$ are measurable with respect to $\mathcal{F}_k$, we have
by the tower property of conditional expectations that
\begin{align*}
    E_3=\mathbb{E}[\langle \nabla M(Q_k-Q^*), \mathbb{E}[M_{k}(Q_{k},\pi_{k})\mid \mathcal{F}_k]\rangle].
\end{align*}
It remains to show that $\mathbb{E}[M_{k}(Q_{k},\pi_{k})\mid \mathcal{F}_k]=0$, i.e., $M_{k}(Q_{k},\pi_{k})$ is a martingale difference sequence with respect to $\mathcal{F}_k$.
For any $(s,a)$, we have
    \begin{align*}
       &\mathbb{E}\left[M_{k}(Q_{k},\pi_{k})(s,a)\mid  \mathcal{F}_{k}\right] \\
       =\;& \mathbb{E} \left[ \gamma \mathds{1}_{\{(S_k,A_k)=(s,a)\}}\left(\max_{a'\in\mathcal{A}} Q_{k}(S_{k+1},a') - \sum_{s'\in\mathcal{S}}p(s'|s,a)\max_{a'\in\mathcal{A}} Q_{k}(s',a')\right) \,\middle|\, \mathcal{F}_{k} \right] \\
       =\;& \gamma \mathds{1}_{\{(S_k,A_k)=(s,a)\}}\left(\mathbb{E} \left[\max_{a'\in\mathcal{A}} Q_{k}(S_{k+1},a')\,\middle|\, \mathcal{F}_{k}\right] - \sum_{s'\in\mathcal{S}}p(s'|s,a)\max_{a'\in\mathcal{A}} Q_{k}(s',a')\right).
    \end{align*}
    Since
    \begin{align*}
        \mathbb{E} \left[\max_{a'\in\mathcal{A}} Q_{k}(S_{k+1},a')\,\middle|\, \mathcal{F}_{k}\right] 
        & = \sum_{s'\in\mathcal{S}} \mathbb{E}\left[\mathds{1}_{\{s' = S_{k+1}\}} \max_{a'\in\mathcal{A}} Q_{k}(s',a') \middle| \mathcal{F}_{k}\right] \\
        & = \sum_{s' \in \mathcal{S}} \max_{a'\in\mathcal{A}} Q_{k}(s',a')\; \mathbb{E}\left[ \mathds{1}_{\{s' = S_{k+1}\}} \mid \mathcal{F}_{k} \right] \tag{$Q_{k} \in \mathcal{F}_{k}$} \\
        & = \sum_{s' \in \mathcal{S}} \max_{a'\in\mathcal{A}} Q_{k}(s',a')\; \mathbb{E}\left[ \mathds{1}_{\{s' = S_{k+1}\}} \mid S_{k}, A_{k} \right] \tag{The Markov property} \\
        & = \sum_{s' \in \mathcal{S}} \max_{a'\in\mathcal{A}} Q_{k}(s',a') p(s'|s,a), 
    \end{align*}
    we have $\mathbb{E}\left[M_{k}(Q_{k},\pi_{k})(s,a)\rvert \mathcal{F}_{k}\right] = 0$.
\subsection{Proof of Lemma \ref{le:E_4}}\label{pf:le:E_4}
Using the definitions of $F(Q_k,\pi_k,Y_k)$ and $M_k(Q_k,\pi_k)$, we have for any $(s,a)$ that
\begin{align*}
    & \left\lvert [F(Q_{k},\pi_{k},Y_{k})](s,a) + [M_{k}(Q_{k},\pi_{k})](s,a) - Q_{k}(s,a) \right\rvert \\
    =\,& \Big\lvert \mathds{1}_{\{(S_{k},A_{k})=(s,a)\}}\Big[\mathcal{R}(s,a) + \gamma \max_{a'\in\mathcal{A}} Q_{k}(S_{k+1},a') - Q_{k}(s,a)\Big] \Big\rvert \\
    \leq\,& \|\mathcal{R}\|_{\infty} + \gamma \|Q_{k}\|_{\infty} + \|Q_{k}\|_{\infty} \\
    \leq\,&1+\frac{\gamma}{1-\gamma}+\frac{1}{1-\gamma}\tag{$\max_{s,a}|\mathcal{R}(s,a)|\leq 1$ and $\|Q_k\|_\infty\leq 1/(1-\gamma)$ \cite{gosavi2006boundedness}}\\
    =\,&\frac{2}{1 -\gamma}.
\end{align*}
Since the previous inequality holds for all $(s,a)$, we have
\begin{align}\label{eq:need}
    \|F(Q_{k},\pi_{k},Y_{k}) + M_{k}(Q_{k},\pi_{k}) - Q_{k}\|_\infty^2\leq\frac{4}{(1-\gamma)^2}, 
\end{align}
which further implies 
\begin{align*}
    E_{4} = \; & \mathbb{E}[\|F(Q_{k},Y_{k}) + M_{k}(Q_{k}) - Q_{k}\|_p^2] \\
    \leq\;& \frac{1}{\ell_{p}^{2}}\mathbb{E}[\|F(Q_{k},Y_{k}) + M_{k}(Q_{k}) - Q_{k}\|_{\infty}^2] \\
    \leq\;& \frac{4}{\ell_{p}^{2}(1-\gamma)^2} \\
    =\; & \frac{4 (|\mathcal{S}||\mathcal{A}|)^{2/p}}{(1-\gamma)^2}. \tag{$\ell_p = (|\mathcal{S}||\mathcal{A}|)^{-1/p}$}
\end{align*}

\subsection{Proof of Proposition \ref{prop:MC}}\label{pf:prop:MC}
Throughout the proof, we assume without loss of generality that $\mu^\top y = \mu_1^\top y_1 = \mu_2^\top y_2 = 0$.
\begin{enumerate}
    \item The fact that $x = \sum_{k=0}^\infty \mathcal{P}^k y / 2$ is a solution to the Poisson equation is a classical result \cite{asmussen1992stationarity,glynn2002hoeffding}. To bound $\|x\|_\infty$, note that we have
    \begin{align*}
        \|x\|_\infty 
        &\leq \frac{1}{2}\sum_{k=0}^\infty \|\mathcal{P}^k y\|_\infty \\
        &= \frac{1}{2}\sum_{k=0}^\infty \max_i \left| \sum_j (\mathcal{P}^k(i,j) - \mu(j)) y_j \right| \tag{$\mu^\top y = 0$} \\
        &\leq \frac{1}{2}\|y\|_\infty \sum_{k=0}^\infty \max_i \sum_j |\mathcal{P}^k(i,j) - \mu(j)| \\
        &\leq \frac{1}{2}\|y\|_\infty \sum_{k=0}^\infty 2 C \rho^k \\
        &= \frac{C\|y\|_\infty}{1 - \rho}.
    \end{align*}

    \item For any $n \geq 0$, we have
\begin{align*}
	\|x_1 - x_2\|_\infty 
	&= \frac{1}{2}\left\| \sum_{k=0}^\infty \mathcal{P}_1^k y_1 - \sum_{k=0}^\infty \mathcal{P}_2^k y_2 \right\|_\infty \\
	&\leq \frac{1}{2}\left\| \sum_{k=0}^{n-1} \mathcal{P}_1^k y_1 - \sum_{k=0}^{n-1} \mathcal{P}_2^k y_2 \right\|_\infty + \frac{1}{2}\left\| \sum_{k=n}^\infty \mathcal{P}_1^k y_1 - \sum_{k=n}^\infty \mathcal{P}_2^k y_2 \right\|_\infty \\
	&\leq \frac{1}{2}\sum_{k=0}^{n-1} \|\mathcal{P}_1^k\|_\infty \|y_1 - y_2\|_\infty + \frac{1}{2}\sum_{k=0}^{n-1} \|\mathcal{P}_1^k - \mathcal{P}_2^k\|_\infty \|y_2\|_\infty \\
	&\quad + \frac{1}{2}\left\| \sum_{k=n}^\infty \mathcal{P}_1^k y_1 \right\|_\infty + \frac{1}{2}\left\| \sum_{k=n}^\infty \mathcal{P}_2^k y_2 \right\|_\infty.
\end{align*}
We now bound each term on the right-hand side. Since each $\mathcal{P}_1^k$ is a stochastic matrix,
\begin{align*}
	\sum_{k=0}^{n-1} \|\mathcal{P}_1^k\|_\infty \|y_1 - y_2\|_\infty 
	= \sum_{k=0}^{n-1} \|y_1 - y_2\|_\infty 
	= n \|y_1 - y_2\|_\infty.
\end{align*}

Next, we bound the difference $\|P_1^k - P_2^k\|_\infty$ recursively:
\begin{align*}
	\|\mathcal{P}_1^k - \mathcal{P}_2^k\|_\infty 
	&\leq \|\mathcal{P}_1(\mathcal{P}_1^{k-1} - \mathcal{P}_2^{k-1})\|_\infty + \|(\mathcal{P}_1 - \mathcal{P}_2)\mathcal{P}_2^{k-1}\|_\infty \\
	&\leq \|\mathcal{P}_1\| \cdot \|\mathcal{P}_1^{k-1} - \mathcal{P}_2^{k-1}\|_\infty + \|\mathcal{P}_1 - \mathcal{P}_2\|_\infty \cdot \|\mathcal{P}_2^{k-1}\|_\infty \\
	&\leq \|\mathcal{P}_1^{k-1} - \mathcal{P}_2^{k-1}\|_\infty + \|\mathcal{P}_1 - \mathcal{P}_2\|_\infty \\
	&\leq \cdots \\
	&\leq k \|\mathcal{P}_1 - \mathcal{P}_2\|_\infty.
\end{align*}
Therefore,
\begin{align*}
	\sum_{k=0}^{n-1} \|\mathcal{P}_1^k - \mathcal{P}_2^k\|_\infty \|y_2\|_\infty 
	\leq \|\mathcal{P}_1 - \mathcal{P}_2\|_\infty \|y_2\|_\infty \sum_{k=0}^{n-1} k 
	= \frac{n(n - 1)}{2} \|\mathcal{P}_1 - \mathcal{P}_2\|_\infty \|y_2\|_\infty.
\end{align*}

Using the same technique as in Part (1), we obtain the following tail bounds:
\begin{align*}
	\left\| \sum_{k=n}^\infty \mathcal{P}_1^k y_1 \right\|_\infty 
	\leq \frac{2C_1 \rho_1^n}{1 - \rho_1} \|y_1\|_\infty, \quad
	\left\| \sum_{k=n}^\infty \mathcal{P}_2^k y_2 \right\|_\infty 
	\leq \frac{2C_2 \rho_2^n}{1 - \rho_2} \|y_2\|_\infty,
\end{align*}
where $(C_1,\rho_1)$ and $(C_2,\rho_2)$ are mixing parameters associated with $\mathcal{P}_1$ and $\mathcal{P}_2$, respectively.

Putting everything together, we have
\begin{align*}
	\|x_1 - x_2\|_\infty 
	\leq \frac{C_1 \rho_1^n \|y_1\|_\infty}{1 - \rho_1}
	+ \frac{C_2 \rho_2^n \|y_2\|_\infty}{1 - \rho_2}
	+ \frac{n}{2} \|y_1 - y_2\|_\infty
	+ \frac{n(n - 1)}{4} \|P_1 - P_2\|_\infty \|y_2\|_\infty.
\end{align*}
Using an entirely similar argument, we also have
\begin{align*}
    \|x_1 - x_2\|_\infty 
	\leq \frac{C_1 \rho_1^n \|y_1\|_\infty}{1 - \rho_1}
	+ \frac{C_2 \rho_2^n \|y_2\|_\infty}{1 - \rho_2}
	+ \frac{n}{2} \|y_1 - y_2\|_\infty
	+ \frac{n(n - 1)}{4} \|P_1 - P_2\|_\infty \|y_1\|_\infty.
\end{align*}
Adding up the previous two inequalities, we obtain
\begin{align*}
    \|x_1 - x_2\|_\infty 
	\leq\,& \frac{C_1 \rho_1^n \|y_1\|_\infty}{1 - \rho_1}+\frac{n(n - 1)}{8} \|P_1 - P_2\|_\infty \|y_1\|_\infty.\\
    &+ \frac{C_2 \rho_2^n \|y_2\|_\infty}{1 - \rho_2}+\frac{n(n - 1)}{8} \|P_1 - P_2\|_\infty \|y_2\|_\infty+ \frac{n}{2} \|y_1 - y_2\|_\infty\\
    \leq \,&\frac{C_{\max} n^2\rho_{\max}^n (\|y_1\|_\infty+\|y_2\|_\infty)}{1 - \rho_{\max}}+\frac{n^2}{8} \|P_1 - P_2\|_\infty (\|y_1\|_\infty+\|y_2\|_\infty)+ \frac{n}{2} \|y_1 - y_2\|_\infty,
\end{align*}
where $C_{\max}=\max(C_1,C_2)$ and $\rho_{\max}=\max(\rho_1,\rho_2)$. 

Finally, since the previous inequality holds for any $n$, by choosing 
\begin{align*}
    n=\frac{\log(\frac{\|P_1 - P_2\|_\infty(1-\rho_{\max})}{8C_{\max}})}{\log(\rho_{\max})},
\end{align*}
we obtain
\begin{align*}
    \|x_1 - x_2\|_\infty \leq\,& \frac{1}{4}\left(\frac{\log(\|P_1 - P_2\|_\infty(1-\rho_{\max}))-\log(8C_{\max})}{\log(\rho_{\max})}\right)^2 \|P_1 - P_2\|_\infty (\|y_1\|_\infty+\|y_2\|_\infty)\\
    &+ \frac{1}{2}\left(\frac{\log(\|P_1 - P_2\|_\infty(1-\rho_{\max}))-\log(8C_{\max})}{\log(\rho_{\max})}\right) \|y_1 - y_2\|_\infty.
\end{align*}
\end{enumerate}

\subsection{Proof of Lemma \ref{le:implicit_to_explicit}}\label{pf:le:mixing-parameters}
\begin{enumerate}[(1)] 
\item For any $s, s' \in \mathcal{S}$, we have
    \begin{align}
        P_{\pi}(s,s') & = \sum_{a\in\mathcal{A}} p(s'\rvert s,a) \pi(a \rvert s) \nonumber \\
        & = \sum_{a\in\mathcal{A}} p(s'\rvert s, a) \pi_{b}(a\rvert s) \frac{\pi(a|s)}{\pi_{b}(a\rvert s)}\nonumber \tag{$\pi_{b}(a\rvert s)\in (0,1)$}\\
        & \geq \min_{s,a}\pi(a|s) \sum_{a\in\mathcal{A}} p(s'\rvert s, a) \pi_{b}(a\rvert s) \nonumber \\
        & = \pi_{\min} P_{\pi_{b}}(s,s'). \label{eq2:pf:le:mixing-parameters}
    \end{align}
    Now, considering the corresponding lazy chain $\mathcal{P}_\pi=(I+P_\pi)/2$, for any $s, s' \in \mathcal{S}$ 
    \begin{align*}
        \mathcal{P}_{\pi}(s,s') & = \frac{1}{2}\left[ \mathds{1}_{\{s = s'\}} + P_{\pi}(s,s') \right] \\
        & \geq \frac{\pi_{\min}}{2}\left[  \mathds{1}_{\{s = s'\}} +  P_{\pi_{b}}(s,s') \right] \tag{Eq. \eqref{eq2:pf:le:mixing-parameters}} \\
        & = \pi_{\min} \mathcal{P}_{\pi_{b}}(s,s')
    \end{align*}
    Thus, we have the entry-wise inequality $\mathcal{P}_{\pi} \geq \pi_{\min} \mathcal{P}_{\pi_{b}}$, a repeated application of which gives $\mathcal{P}_{\pi}^{k} \geq \pi_{\min}^{k} \mathcal{P}_{\pi_{b}}^{k}$ for all $k \geq 0$. 
    Since $\mu_{\pi}$ is the stationary distribution of both $P_{\pi}$ and $\mathcal{P}_{\pi}$, we have for any $s\in\mathcal{S}$ that
    \begin{align*}
        \mu_{\pi}(s) & = \sum_{s'\in\mathcal{S}} \mu_{\pi}(s') \mathcal{P}_{\pi}^{r_{b}}(s',s) \tag{$\mu_\pi^\top =\mu_\pi^\top P_\pi^k$ for any $k\geq 0$}\\
        & \geq \pi_{\min}^{r_{b}} \sum_{s'\in\mathcal{S}} \mu_{\pi}(s') \mathcal{P}_{\pi_{b}}^{r_{b}}(s',s) \\
        & \geq \pi_{\min}^{r_{b}} \sum_{s'\in\mathcal{S}} \mu_{\pi}(s') \delta_{b} \mu_{\pi_b}(s)\tag{Definition of $\delta_b$}\\
        & \geq \pi_{\min}^{r_{b}} \delta_{b} \mu_{\pi_b,\min} \sum_{s'\in\mathcal{S}} \mu_{\pi}(s') \\
        & = \pi_{\min}^{r_{b}} \delta_{b} \mu_{\pi_b,\min}. 
    \end{align*}
    It follows that
    \begin{align*}
        \min_{s,a}\bar{\mu}_\pi(s,a)=\min_{s,a}\mu_\pi(s)\pi(a|s)\geq \pi_{\min}^{r_{b}+1} \delta_{b} \mu_{\pi_b,\min}.
    \end{align*}
	\item We first show that $\Bar{\mathcal{P}}_{\pi}^{k} \geq \pi_{\min}^{k} \Bar{\mathcal{P}}_{\pi_{b}}^{k}$ for all $k\geq 0$. For any $(s,a), (s',a') \in \mathcal{S}\times\mathcal{A}$, we have    
	\begin{align*}
		\bar{\mathcal{P}}_{\pi}((s,a),(s',a')) 
		& = \frac{1}{2}\left[\mathds{1}_{\{(s,a) = (s',a')\}} + p(s'\rvert s,a) \pi(a'\rvert s')\right] \\
		& = \frac{1}{2}\left[\mathds{1}_{\{(s,a) = (s',a')\}} + p(s'\rvert s,a) \pi_{b}(a'\rvert s') \frac{\pi(a'\rvert s')}{\pi_{b}(a'\rvert s')}\right] \\
		& \geq \frac{\pi_{\min}}{2}\left[\mathds{1}_{\{(s,a) = (s',a')\}} +  p(s'\rvert s,a) \pi_{b}(a'\rvert s') \right] \tag{$\pi_{b}(a'\rvert s')\in (0,1) ,\pi_{\min} \in (0,1)$} \\
		& = \frac{\pi_{\min}}{2} \left[ \mathds{1}_{\{(s,a) = (s',a')\}} + \bar{P}_{\pi_{b}}((s,a),(s',a')) \right] \\
		& = \pi_{\min}\bar{\mathcal{P}}_{\pi_{b}}((s,a),(s',a')).
	\end{align*}
	Therefore, we have the entry-wise inequality $\Bar{\mathcal{P}}_{\pi} \geq \pi_{\min} \Bar{\mathcal{P}}_{\pi_{b}}$, and hence, $\Bar{\mathcal{P}}_{\pi}^{k} \geq \pi_{\min}^{k} \Bar{\mathcal{P}}_{\pi_{b}}^{k}$ for all $k\geq 0$.
	By the definition of $\Bar{\mathcal{P}}_{\pi_{b}}$, for any $k \geq 0$, we have
	\begin{align*}
		\Bar{\mathcal{P}}_{\pi_{b}}^{k} = \frac{1}{2^{k}}\left[ I + \bar{P}_{\pi_{b}}\right]^{k} = \frac{1}{2^{k}}\sum_{j = 0}^{k} \binom{k}{j}\bar{P}_{\pi_{b}}^{j}. 
	\end{align*}
	Therefore, for any $(s,a),(s',a') \in \mathcal{Y}$, we have
	\begin{align*}
		\Bar{\mathcal{P}}_{\pi_{b}}^{r_{b}+1}((s,a),(s',a')) & = \frac{1}{2^{r_{b}+1}}\sum_{j = 0}^{r_{b}+1} \binom{r_{b}+1}{j}\bar{P}_{\pi_{b}}^{j}((s,a),(s',a')) \\
		& \geq \frac{1}{2^{r_{b}+1}}\sum_{j = 1}^{r_{b}+1} \binom{r_{b}+1}{j}\bar{P}_{\pi_{b}}^{j}((s,a),(s',a')) \\
		& = \frac{1}{2^{r_{b}+1}}\sum_{j = 1}^{r_{b}+1} \binom{r_{b}+1}{j} \sum_{s''\in\mathcal{S}} p(s''\rvert s, a) P_{\pi_{b}}^{j-1}(s'',s') \pi_{b}(a'\rvert s') \\
		& = \frac{1}{2^{r_{b}+1}} \sum_{s''\in\mathcal{S}} p(s''\rvert s, a) \left[ \sum_{j = 1}^{r_{b}+1} \binom{r_{b}+1}{j} P_{\pi_{b}}^{j-1}(s'',s') \right] \pi_{b}(a'\rvert s')\\
		& = \frac{1}{2^{r_{b}+1}} \sum_{s''\in\mathcal{S}} p(s''\rvert s, a) \left[ \sum_{i = 0}^{r_{b}} \binom{r_{b}+1}{i+1} P_{\pi_{b}}^i(s'',s') \right] \pi_{b}(a'\rvert s')\tag{Change of variable: $i=j-1$}\\
		& = \frac{1}{2^{r_{b}+1}} \sum_{s''\in\mathcal{S}} p(s''\rvert s, a) \left[ \sum_{i = 0}^{r_{b}} \binom{r_{b}}{i}\frac{r_b+1}{i+1} P_{\pi_{b}}^i(s'',s') \right] \pi_{b}(a'\rvert s')\\
		&\geq \frac{1}{2^{r_{b}+1}} \sum_{s''\in\mathcal{S}} p(s''\rvert s, a) \left[ \sum_{i = 0}^{r_{b}} \binom{r_{b}}{i} P_{\pi_{b}}^i(s'',s') \right] \pi_{b}(a'\rvert s') \tag{$r_{b} \geq i$} \\
		&= \frac{1}{2} \sum_{s''\in\mathcal{S}} p(s''\rvert s, a) \mathcal{P}_{\pi_b}^{r_b}(s'',s') \pi_{b}(a'\rvert s')\\
		&\geq \frac{\delta_b}{2} \sum_{s''\in\mathcal{S}} p(s''\rvert s, a) \mu_{\pi_b}(s') \pi_{b}(a'\rvert s')\\
		&= \frac{\delta_b}{2} \bar{\mu}_{\pi_b}(s',a').
	\end{align*}
	Since $\Bar{\mathcal{P}}_{\pi}^{k} \geq \pi_{\min}^{k} \Bar{\mathcal{P}}_{\pi_{b}}^{k}$ for all $k\geq 0$, we have 
	\begin{align*}
		\Bar{\mathcal{P}}_{\pi}^{r_{b}+1}((s,a),(s',a')) & \geq \pi_{\min}^{r_b+1}\Bar{\mathcal{P}}_{\pi_b}^{r_{b}+1}((s,a),(s',a'))\\
		&\geq \frac{1}{2}\delta_{b}\pi_{\min}^{r_{b}+1} \Bar{\mu}_{\pi_{b}}(s',a') \\
		& = \frac{1}{2}\delta_{b}\pi_{\min}^{r_{b}+1} \frac{\bar{\mu}_{\pi_{b}}(s',a')}{\bar{\mu}_{\pi}(s',a')}\bar{\mu}_{\pi}(s',a')\tag{$\bar{\mu}_\pi(s,a)>0$ for all $(s,a)$} \\
		& \geq \frac{1}{2}\delta_{b}\pi_{\min}^{r_{b}+1} \mu_{\pi_b}(s') \pi_{b}(a'\rvert s') \bar{\mu}_{\pi}(s',a') \tag{$\bar{\mu}_{\pi}(s',a') <1$} \\
		& \geq \frac{1}{2}\delta_{b} \pi_{\min}^{r_{b}+1} \mu_{\pi_b,\min}\pi_{b,\min} \bar{\mu}_{\pi}(s',a').
	\end{align*}
	With the previous inequality at hand, we follow the proof of \cite[Theorem 4.9 from Eq. (4.15) to Eq. (4.21)]{levin2017markov} to conclude that
	\begin{align*}
		\max_{(s,a)} \|\Bar{\mathcal{P}}_{\pi}^{k}((s,a),(\cdot,\cdot))) - \bar{\mu}_{\pi}(\cdot,\cdot)\|_{\text{TV}} \leq \bar{C}_\pi \bar{\rho}_\pi^{k},\quad \forall\,k\geq 0,
	\end{align*}
	where
	\begin{align*}
		\bar{C}_\pi = \left(1 - \frac{1}{2}\delta_{b} \pi_{\min}^{r_{b}+1} \mu_{\pi_b,\min} \pi_{b,\min}\right)^{-1}, \quad \text{and}\quad 
		\bar{\rho}_\pi = \left(1 - \frac{1}{2}\delta_{b} \pi_{\min}^{r_{b}+1} \mu_{\pi_b,\min} \pi_{b,\min}\right)^{1/(r_{b}+1)}.
	\end{align*}
\end{enumerate}

\subsection{Proof of Lemma \ref{le:E23}}\label{pf:le:E23}
By H\"{o}lder's inequality, we have
\begin{align}
    & \mathbb{E}[\langle \nabla M(Q_{k+1} - Q^*) - \nabla M(Q_k - Q^*), h(Q_{k+1}, \pi_{k+1}, Y_{k+1}) \rangle] \nonumber \\
    \leq\,& \mathbb{E}[\|\nabla M(Q_{k+1} - Q^*) - \nabla M(Q_k - Q^*)\|_q \cdot \|h(Q_{k+1}, \pi_{k+1}, Y_{k+1})\|_p] \nonumber \\
    \leq\,& (|\mathcal{S}||\mathcal{A}|)^{1/p} \mathbb{E}[\|\nabla M(Q_{k+1} - Q^*) - \nabla M(Q_k - Q^*)\|_q \cdot \|h(Q_{k+1}, \pi_{k+1}, Y_{k+1})\|_\infty], \label{eq1:pf:le:PE_first_term}
\end{align}
where $1/p + 1/q = 1$.

Since the Lyapunov function $M(\cdot)$ is $L$-smooth with respect to $\|\cdot\|_p$, we have
\begin{align}
    \|\nabla M(Q_{k+1} - Q^*) - \nabla M(Q_k - Q^*)\|_q 
    \leq\,& L \|Q_{k+1} - Q_k\|_p \nonumber \\
    \leq\,& L(|\mathcal{S}||\mathcal{A}|)^{1/p} \|Q_{k+1} - Q_k\|_\infty \nonumber \\
    =\,& \alpha_k L(|\mathcal{S}||\mathcal{A}|)^{1/p} \|F(Q_k, Y_k) + M_k(Q_k) - Q_k\|_\infty \nonumber \\
    \leq\,& \frac{2L(|\mathcal{S}||\mathcal{A}|)^{1/p} \alpha_k}{1 - \gamma}, \label{eq2:pf:le:PE_first_term}
\end{align}
where the last inequality follows from Eq. (\ref{eq:need}).
It remains to bound $\|h(Q_{k+1}, \pi_{k+1}, Y_{k+1})\|_\infty$. Note that, fixing $(s,a)$, $[h(Q_{k+1}, \pi_{k+1}, Y_{k+1})](s,a)$ solves the Poisson equation
\begin{align*}
     &[h(Q_{k+1},\pi_{k+1},Y_{k+1})](s,a) - \sum_{y'\in\mathcal{Y}}\bar{P}_{k+1}(Y_{k+1},y')[h(Q_{k+1},\pi_{k+1},y')](s,a)\\
    =\,& [F(Q_{k+1}, \pi_{k+1}, Y_{k+1})](s,a) - [\bar{F}(Q_{k+1}, \pi_{k+1})](s,a).
\end{align*}
Therefore, denoting $(\bar{C}_{k+1},\bar{\rho}_{k+1})$ as the mixing parameters associated with the lazy transition matrix $\bar{\mathcal{P}}_{k+1}$, we have by Proposition \ref{prop:MC} (1) that
\begin{align*}
    |[h(Q_{k+1}, \pi_{k+1}, Y_{k+1})](s,a)|
    \leq\,& \frac{\bar{C}_{k+1}}{1 - \bar{\rho}_{k+1}} \max_{y\in\mathcal{Y}}|[F(Q_{k+1}, y)](s,a) - [\bar{F}(Q_{k+1}, \pi_{k+1})](s,a)| \\
    \leq \,& \frac{\bar{C}_{k+1}}{1 - \bar{\rho}_{k+1}} \max_{y\in\mathcal{Y}}\|F(Q_{k+1}, y) - \bar{F}(Q_{k+1}, \pi_{k+1})\|_\infty \\
    \leq \,&\frac{2\bar{C}_{k+1}}{(1 - \bar{\rho}_{k+1})(1-\gamma)}, 
\end{align*}
where the last inequality follows from $\|Q_k\|_\infty\leq 1/(1-\gamma)$ \cite{gosavi2006boundedness} and Lemma \ref{le:F-lipschitz}. 
The previous inequality implies
\begin{align}\label{eq:PES_boundedness}
    \|h(Q_{k+1}, \pi_{k+1}, Y_{k+1})\|_\infty\leq \frac{2\bar{C}_{k+1}}{(1 - \bar{\rho}_{k+1})(1-\gamma)}.
\end{align}
Using the previous inequality and Eq. \eqref{eq2:pf:le:PE_first_term} in Eq. \eqref{eq1:pf:le:PE_first_term}, we obtain
\begin{align*}
    \mathbb{E}[\langle \nabla M(Q_{k+1} - Q^*) - \nabla M(Q_k - Q^*), h(Q_{k+1}, \pi_{k+1}, Y_{k+1}) \rangle] 
    \leq \frac{4\bar{C}_{k+1} L (|\mathcal{S}||\mathcal{A}|)^{2/p} \alpha_k}{(1 - \bar{\rho}_{k+1})(1 - \gamma)^2}, 
\end{align*}
which, upon multiplying both sides by $\alpha_{k+1}/\alpha_{k}$, yields the desired inequality. 

\subsection{Proof of Lemma \ref{le:E24}}\label{pf:le:E24}

For any $k\geq 0$, using Lemma \ref{le:Moreau}, we have
\begin{align}
    & \langle \nabla M(Q_k - Q^*), h(Q_{k+1}, \pi_{k+1}, Y_{k+1}) - h(Q_k, \pi_k, Y_{k+1}) \rangle \nonumber\\
    =\,& \|Q_k - Q^*\|_m \langle \nabla \|Q_k - Q^*\|_m, h(Q_{k+1}, \pi_{k+1}, Y_{k+1}) - h(Q_k, \pi_k, Y_{k+1}) \rangle \nonumber\\
    \leq\,& \|Q_k - Q^*\|_m \left\| \nabla \|Q_k - Q^*\|_m \right\|_m^* \cdot \|h(Q_{k+1}, \pi_{k+1}, Y_{k+1}) - h(Q_k, \pi_k, Y_{k+1})\|_m \nonumber\\
    \leq\,& \|Q_k - Q^*\|_m \cdot \|h(Q_{k+1}, \pi_{k+1}, Y_{k+1}) - h(Q_k, \pi_k, Y_{k+1})\|_m\nonumber\\
    \leq \,&\frac{1}{\ell_m}\sqrt{2M(Q_k - Q^*)} \cdot \|h(Q_{k+1}, \pi_{k+1}, Y_{k+1}) - h(Q_k, \pi_k, Y_{k+1})\|_\infty\nonumber\\
    \leq\,& \frac{1}{2}\left(1-\frac{u_m}{\ell_m}\gamma_k\right)M(Q_k - Q^*) +\frac{1}{\ell_m^{2}\left(1-\frac{u_m}{\ell_m}\gamma_k\right)} \|h(Q_{k+1}, \pi_{k+1}, Y_{k+1}) - h(Q_k, \pi_k, Y_{k+1})\|_\infty^2,\label{eq1:pf:le:PE-second-term}
\end{align}
where the last line follows from $a^2+b^2\geq 2ab$ for any $a,b\in\mathbb{R}$.
To proceed, applying Proposition \ref{prop:MC} (2), we have 
\begin{align*}
    &\|h(Q_{k+1}, \pi_{k+1}, Y_{k+1}) - h(Q_k, \pi_k, Y_{k+1})\|_\infty\\ 
    \leq\,& \frac{1}{4}\left(\frac{\log(\|\bar{P}_{k+1} - \bar{P}_k\|_\infty(1-\rho_{\max}))-\log(8C_{\max})}{\log(\rho_{\max})}\right)^2 \|\bar{P}_{k+1} - \bar{P}_k\|_\infty\\
    &\times (\|F(Q_{k+1},Y_{k+1})-\bar{F}(Q_{k+1},\pi_{k+1})\|_\infty+\|F(Q_k,Y_k)-\bar{F}(Q_k,\pi_k)\|_\infty)\\
    &+ \frac{1}{2}\left(\frac{\log(\|\bar{P}_{k+1} - \bar{P}_k\|_\infty(1-\rho_{\max}))-\log(8C_{\max})}{\log(\rho_{\max})}\right)\\
    &\times \|F(Q_{k+1},Y_{k+1})-\bar{F}(Q_{k+1},\pi_{k+1})-F(Q_{k},Y_{k+1})+\bar{F}(Q_{k},\pi_{k})\|_\infty\\
    \leq \,&\frac{1}{1-\gamma}\left(\frac{\log(\|\bar{P}_{k+1} - \bar{P}_k\|_\infty(1-\rho_{\max}))-\log(8C_{\max})}{\log(\rho_{\max})}\right)^2 \|\bar{P}_{k+1} - \bar{P}_k\|_\infty\\
    &+ \frac{1}{2}\left(\frac{\log(\|\bar{P}_{k+1} - \bar{P}_k\|_\infty(1-\rho_{\max}))-\log(8C_{\max})}{\log(\rho_{\max})}\right)\\
    &\times \left(4\|Q_{k+1}-Q_k\|_\infty+\frac{2}{1-\gamma}\|\bar{\mu}_{k+1}-\bar{\mu}_k\|_\infty\right)
\end{align*}
where $C_{\max}=\max(\bar{C}_k,\bar{C}_{k+1})$, $\rho_{\max}=\max(\bar{\rho}_k,\bar{\rho}_{k+1})$, and the last inequality follows from Lemmas \ref{le:f-bar} and \ref{le:F-lipschitz}. 

To further bound the right-hand side of the previous inequality, observe that
\begin{align*}
    \|Q_{k+1}-Q_k\|_\infty=\,&\alpha_k\|F(Q_k,Y_k)+M_k(Q_k,\pi_k)-Q_k\|_\infty\leq \frac{2\alpha_k}{1-\gamma},\tag{Eq. (\ref{eq:need})}\\
    \|\bar{\mu}_{\pi_k}-\bar{\mu}_{\pi_{k+1}}\|_\infty\leq\,& 2\frac{\log(\|\pi_{k+1}-\pi_k\|_\infty)-\log(4\bar{C}_k)}{\log(\bar{\rho}_k)}\cdot\|\pi_k-\pi_{k+1}\|_\infty\tag{Lemma \ref{le:mu-pi-lipschitz}}\\
\text{and}\quad 
    \|\bar{P}_{\pi_k}-\bar{P}_{\pi_{k+1}}\|_\infty=\,&\max_{s,a}\sum_{s',a'}\left|\bar{P}_{\pi_k}((s,a),(s',a'))-\bar{P}_{\pi_{k+1}}((s,a),(s',a'))\right|\\
    =\,&\max_{s,a}\sum_{s',a'}p(s'|s,a)\left|\pi_k(a'|s')-\pi_{k+1}(a'|s')\right|\\
    =\,&\max_{s'}\sum_{a'}\left|\pi_k(a'|s')-\pi_{k+1}(a'|s')\right|\\
    =\,&\|\pi_k-\pi_{k+1}\|_\infty.
\end{align*}
Therefore, we have
\begin{align}
    &\|h(Q_{k+1}, \pi_{k+1}, Y_{k+1}) - h(Q_k, \pi_k, Y_{k+1})\|_\infty\nonumber\\ 
    \leq \,&\frac{1}{1-\gamma}\left(\frac{\log(\|\pi_k-\pi_{k+1}\|_\infty(1-\rho_{\max}))-\log(8C_{\max})}{\log(\rho_{\max})}\right)^2 \|\pi_k-\pi_{k+1}\|_\infty\nonumber\\
    &+ \frac{2}{1-\gamma}\left(\frac{\log(\|\pi_k-\pi_{k+1}\|_\infty(1-\rho_{\max}))-\log(8C_{\max})}{\log(\rho_{\max})}\right)\nonumber\\
    &\times \left(2\alpha_k+\frac{\log(\|\pi_{k+1}-\pi_k\|_\infty)-\log(4\bar{C}_k)}{\log(\bar{\rho}_k)}\cdot\|\pi_k-\pi_{k+1}\|_\infty\right).\label{eq2:pf:le:PE-second-term}
\end{align}
It remains to bound $\|\pi_k-\pi_{k+1}\|_\infty$. Since $\nu(\cdot)$ is $1$-strongly concave with respect to $\|\cdot\|_1$, by the conjugate correspondence theorem \cite[Theorem~5.26]{beck2017first}, $\sigma(\cdot)$ satisfies $\|\sigma(x_1)-\sigma(x_2)\|_1\leq \|x_1-x_2\|_\infty$. Therefore, for any $s\in\mathcal{S}$, we have
\begin{align*}
    &\|\pi_{k+1}(s)-\pi_k(s)\|_1\\
    =\,&\left\|\frac{\epsilon_k \mathbf{1}}{|\mathcal{A}|}+(1-\epsilon_k)\sigma\left(\frac{Q_k(s)}{\tau_k}\right)-\frac{\epsilon_{k+1}\mathbf{1}}{|\mathcal{A}|}-(1-\epsilon_{k+1})\sigma\left(\frac{Q_{k+1}(s)}{\tau_{k+1}}\right)\right\|_1\\
    \leq \,&\frac{|\epsilon_k-\epsilon_{k+1}|}{|\mathcal{A}|}\|\mathbf{1}\|_1+\left\|\sigma\left(\frac{Q_k(s)}{\tau_k}\right)-\sigma\left(\frac{Q_{k+1}(s)}{\tau_{k+1}}\right)\right\|_1+|\epsilon_k-\epsilon_{k+1}|\left\|\sigma\left(\frac{Q_k(s)}{\tau_k}\right)\right\|_1\\
    = \,&2|\epsilon
    _k-\epsilon_{k+1}|+\left\|\sigma\left(\frac{Q_k(s)}{\tau_k}\right)-\sigma\left(\frac{Q_{k+1}(s)}{\tau_{k+1}}\right)\right\|_1\\
    \leq \,&2|\epsilon
    _k-\epsilon_{k+1}|+\left\|\frac{Q_k(s)}{\tau_k}-\frac{Q_{k+1}(s)}{\tau_{k+1}}\right\|_\infty\tag*{\cite[Theorem~5.26]{beck2017first}}\\
    \leq \,&2|\epsilon
    _k-\epsilon_{k+1}|+\frac{1}{\tau_k}\left\|Q_k-Q_{k+1}\right\|_\infty+\frac{|\tau_k-\tau_{k+1}|}{\tau_k\tau_{k+1}}\left\|Q_{k+1}\right\|_\infty\nonumber\\
    \leq \,&2|\epsilon
    _k-\epsilon_{k+1}|+\frac{2\alpha_k}{\tau_k(1-\gamma)}+\frac{|\tau_k-\tau_{k+1}|}{\tau_k\tau_{k+1}(1-\gamma)}\\
    :=\,&g_k.
\end{align*}
As a result, by the definition of matrix-induced norms, we have
\begin{align*}
    \|\pi_{k+1}-\pi_k\|_\infty
    =\max_{s\in\mathcal{S}}\|\pi_{k+1}(s)-\pi_k(s)\|_1\leq g_k.
\end{align*}

Using the previous inequality in Eq. (\ref{eq2:pf:le:PE-second-term}), we have
\begin{align*}
    \|h(Q_{k+1}, \pi_{k+1}, Y_{k+1}) - h(Q_k, \pi_k, Y_{k+1})\|_\infty
    \leq \,&\frac{1}{1-\gamma}\left(\frac{\log(g_k(1-\rho_{\max}))-\log(8C_{\max})}{\log(\rho_{\max})}\right)^2 g_k\\
    &+ \frac{2}{1-\gamma}\left(\frac{\log(g_k(1-\rho_{\max}))-\log(8C_{\max})}{\log(\rho_{\max})}\right)\\
    &\times \left(2\alpha_k+\frac{\log(g_k)-\log(4\bar{C}_k)}{\log(\bar{\rho}_k)}\cdot g_k\right)\\
    \leq \,&\frac{5}{1-\gamma}\left(\frac{\log(g_k(1-\rho_{\max}))-\log(8C_{\max})}{\log(\rho_{\max})}\right)^2 g_k\\
    :=\,&N_k.
\end{align*}

Finally, using the previous inequality in Eq. (\ref{eq1:pf:le:PE-second-term}), we obtain
\begin{align*}
    &\langle \nabla M(Q_k - Q^*), h(Q_{k+1}, \pi_{k+1}, Y_{k+1}) - h(Q_k, \pi_k, Y_{k+1}) \rangle \\
    \leq\,&  \frac{1}{2}\left(1-\frac{u_m}{\ell_m}\gamma_k\right)M(Q_k - Q^*) +\frac{N_k^2}{\ell_m^{2}\left(1-\frac{u_m}{\ell_m}\gamma_k\right)},  
\end{align*}
and thus
\begin{align*}
    E_{3,4} & = \frac{\alpha_{k+1}}{\alpha_k}\mathbb{E}[\langle \nabla M(Q_k - Q^*), h(Q_{k+1}, \pi_{k+1}, Y_{k+1}) - h(Q_k, \pi_k, Y_{k+1}) \rangle] \\
    & \leq \frac{\alpha_{k+1}}{2\alpha_k}\left(1-\frac{u_m}{\ell_m}\gamma_k\right)\mathbb{E}[M(Q_k - Q^*)] + \frac{\alpha_{k+1} N_k^2}{\alpha_k\ell_m^{2}\left(1-\frac{u_m}{\ell_m}\gamma_k\right)}.
\end{align*}

\subsection{Proof of Lemma \ref{le:E25}}\label{pf:le:E25}

For any $k\geq 0$, using Lemma \ref{le:Moreau} (2) and H\"{o}lder's inequality, we have
\begin{align*}
    \langle \nabla M(Q_k - Q^*), h(Q_k, \pi_k, Y_{k+1}) \rangle 
     \leq \,&\|Q_{k} - Q^*\|_m \left\| \nabla \|Q_k - Q^*\|_m \right\|_m^* \cdot \|h(Q_k, \pi_k, Y_{k+1})\|_m\\
     \leq \,&\|Q_k - Q^*\|_m \|h(Q_k, \pi_k, Y_{k+1})\|_m\tag{Lemma \ref{lem:shalev}}\\
     \leq \,&\frac{1}{\ell_m}\sqrt{2M(Q_k-Q^*)} \|h(Q_k, \pi_k, Y_{k+1})\|_\infty\tag{Lemma \ref{le:Moreau} (2) and (3)}\\
     \leq \,&\frac{2\bar{C}_k}{\ell_m(1 - \bar{\rho}_k)(1-\gamma)}\sqrt{2M(Q_k-Q^*)},
\end{align*}
where the last inequality follows from Eq. (\ref{eq:PES_boundedness}). It follows that
\begin{align*}
    &\frac{\alpha_{k+1}-\alpha_{k}}{\alpha_k} \langle \nabla M(Q_k - Q^*), h(Q_k, \pi_k, Y_{k+1})\rangle   \\
     \leq \,& \frac{2|\alpha_{k+1}-\alpha_{k}|\bar{C}_k}{\alpha_k\ell_m(1 - \bar{\rho}_k)(1-\gamma)}\sqrt{2M(Q_k-Q^*)}\\
     \leq \,&\frac{1}{2}\left(1-\frac{u_m}{\ell_m}\gamma_k\right)M(Q_k - Q^*) + \frac{4(\alpha_{k+1}-\alpha_{k})^{2}\bar{C}_k^{2}}{\alpha_{k}^{2}\ell_{m}^{2}(1 - \bar{\rho}_k)^{2} (1-\gamma)^{2} \left(1-\frac{u_m}{\ell_m}\gamma_k\right)}.
\end{align*}
where the last inequality follows from $(a^2+b^2\geq 2ab)$ for any $a,b\in\mathbb{R}$. Taking expectations on both sides of the previous inequality yields
\begin{align*}
    E_{3,5}=\,&\frac{\alpha_{k+1}-\alpha_{k}}{\alpha_k} \mathbb{E}[\langle \nabla M(Q_k - Q^*), h(Q_k, \pi_k, Y_{k+1})\rangle]\\
    \leq \,&\frac{1}{2}\left(1-\frac{u_m}{\ell_m}\gamma_k\right)\mathbb{E}[M(Q_k - Q^*)] + \frac{4(\alpha_{k+1}-\alpha_{k})^{2}\bar{C}_k^{2}}{\alpha_{k}^{2}\ell_{m}^{2}(1 - \bar{\rho}_k)^{2} (1-\gamma)^{2} \left(1-\frac{u_m}{\ell_m}\gamma_k\right)}.
\end{align*}

\subsection{Solving the Recursion}\label{pf:le:Qk-final-bound-constant}
We begin by simplifying the bound in Proposition~\ref{le:final-drift-ineq} under constant parameters $\alpha_k \equiv \alpha$, $\epsilon_k \equiv \epsilon$, and $\tau_k \equiv \tau$. For clarity, we write $E_{2,2}$ as $E_{2,2}(k)$ to emphasize its dependence on $k$.
 Then, we have
\begin{align*}
        \mathbb{E}[M(Q_{k+1}-Q^*)]
        \leq\, & \left[ 1 - \alpha_{k}\left(1-\frac{u_m}{\ell_m}\gamma_k\right) \right] \mathbb{E}[M(Q_k-Q^*)]+\alpha_kE_{2,2}(k)+ \frac{\alpha_{k} N_k^2}{\ell_{m}^{2}\left(1-\frac{u_m}{\ell_m}\gamma_k\right)}\\
        &+ \frac{6\bar{C}_{k+1} L (|\mathcal{S}||\mathcal{A}|)^{2/p} \alpha_{k}^{2}}{(1 - \bar{\rho}_{k+1})(1 - \gamma)^2} 
        + \frac{4(\alpha_{k+1}-\alpha_{k})^{2}\bar{C}_k^{2}}{\alpha_{k}(1 - \bar{\rho}_k)^{2} (1-\gamma)^{2} \left(1-\frac{u_m}{\ell_m}\gamma_k\right)}\\
        =\,&\left[ 1 - \alpha\left(1-\frac{u_m}{\ell_m}\bar{\gamma}\right) \right] \mathbb{E}[M(Q_k-Q^*)]+\alpha E_{2,2}(k)\\
        &+ \frac{100\alpha^3}{\tau^2\ell_{m}^{2}\left(1-\frac{u_m}{\ell_m}\bar{\gamma}\right)(1-\gamma)^4}\left(\frac{\log(2\alpha(1-\bar{\rho})/[8\bar{C}\tau(1-\gamma)])}{\log(\bar{\rho})}\right)^4\\
        &+ \frac{6\bar{C} L (|\mathcal{S}||\mathcal{A}|)^{2/p} \alpha^{2}}{(1 - \bar{\rho})(1 - \gamma)^2},
    \end{align*}
    where we recall that $\lambda:=\min_{1\leq k\leq K}\min_{s,a}\pi_k(a|s)\geq \epsilon/|\mathcal{A}|$, and 
    \begin{align*}
        \bar{\gamma}=\,&1-\lambda^{r_b}\mu_{\pi_b,\min}\delta_b(1-\gamma),\quad \bar{C}=  \left(1 - \frac{1}{2}\delta_{b} \lambda^{r_{b}+1} \mu_{\pi_b,\min} \pi_{b,\min}\right)^{-1},\\       
        \bar{\rho} =\,&  \left(1 - \frac{1}{2}\delta_{b} \lambda^{r_{b}+1} \mu_{\pi_b,\min} \pi_{b,\min}\right)^{1/(r_{b}+1)}.
    \end{align*}
    Repeatedly using the previous inequality, we obtain
    \begin{align}
        \mathbb{E}[M(Q_k-Q^*)]\leq \,&\left[ 1 - \alpha\left(1-\frac{u_m}{\ell_m}\bar{\gamma}\right) \right]^k \mathbb{E}[M(Q_0-Q^*)]+\underbrace{\sum_{i=0}^{k-1}\alpha E_{2,2}(i)\left[ 1 - \alpha\left(1-\frac{u_m}{\ell_m}\bar{\gamma}\right) \right]^{k-i-1}}_{\text{The telescoping term}}\nonumber\\
        &+ \frac{100\alpha^2}{\tau^2\ell_{m}^{2}\left(1-\frac{u_m}{\ell_m}\bar{\gamma}\right)^2(1-\gamma)^4}\left(\frac{\log(2\alpha(1-\bar{\rho})/[8\bar{C}\tau(1-\gamma)])}{\log(\bar{\rho})}\right)^4\nonumber\\
        &+ \frac{6\bar{C} L (|\mathcal{S}||\mathcal{A}|)^{2/p} \alpha}{\left(1-\frac{u_m}{\ell_m}\bar{\gamma}\right)(1 - \bar{\rho})(1 - \gamma)^2}.\label{eq:before_last}
    \end{align}
    We next simplify the telescoping term. For simplicity of notation, denote 
    \begin{align*}
        v_k=\mathbb{E}[\langle \nabla M(Q_k - Q^*), h(Q_k, \pi_k, Y_k) \rangle]\quad \text{and}\quad \phi=1 - \alpha\left(1-\frac{u_m}{\ell_m}\bar{\gamma}\right).
    \end{align*}
    Then, we have
    \begin{align*}
    \sum_{i=0}^{k-1}\alpha E_{2,2}(i)\phi^{k-i-1}
    =\,&\alpha\phi^{k}\sum_{i=0}^{k-1} \frac{v_i-v_{i+1}}{\phi^{i+1}}\\
    =\,&\alpha\phi^{k}\left(\sum_{i=0}^{k-1} \frac{v_i}{\phi^{i+1}}-\sum_{i=0}^{k-1} \frac{v_{i+1}}{\phi^{i+1}}\right)\\
    =\,&\alpha\phi^{k}\left(\frac{1}{\phi}\sum_{i=0}^{k-1} \frac{v_i}{\phi^{i}}-\sum_{i=1}^{k} \frac{v_{i}}{\phi^{i}}\right)\\
    =\,&\alpha \phi^{k-1}v_0-\alpha v_k+\alpha\phi^{k-1}\left(1-\phi\right)\sum_{i=1}^{k-1} \frac{v_i}{\phi^{i}}.
\end{align*}
To proceed, we next bound $|v_k|$. Note that for any $k\geq 0$, we have
\begin{align*}
    |v_{k}| & = \left|\mathbb{E}[\langle \nabla M(Q_k - Q^*), h(Q_k, \pi_k, Y_k) \rangle] \right| \\
    & \leq  \mathbb{E}\left[ \left| \langle \nabla M(Q_k - Q^*), h(Q_k, \pi_k, Y_k) \rangle \right| \right] \tag{Jensen's inequality}\\
    & \leq  \mathbb{E}\left[ \|Q_{k} - Q^*\|_m \left\| \nabla \|Q_k - Q^*\|_m \right\|_m^* \cdot \|h(Q_k, \pi_k, Y_{k+1})\|_m \right] \nonumber \tag{Lemma \ref{le:Moreau} and H\"{o}lder's inequality}\\ 
    & \leq   \mathbb{E}\left[ \|Q_k - Q^*\|_m \|h(Q_k, \pi_k, Y_{k+1})\|_m \right] \nonumber \tag{Lemma \ref{lem:shalev}} \\
    & \leq \frac{1}{\ell_{m}^{2}} \mathbb{E}\left[ \|Q_k - Q^*\|_\infty \|h(Q_k, \pi_k, Y_{k+1})\|_\infty \right]\\
   & \leq \frac{4\bar{C}}{\ell_{m}^{2}(1 - \bar{\rho})(1-\gamma)^2}\tag{Eq. (\ref{eq:PES_boundedness}) and $\|Q_k-Q^*\|_\infty\leq 2/(1-\gamma)$}
\end{align*}
It follows that
\begin{align*}
    &\sum_{i=0}^{k-1}\alpha E_{2,2}(i)\phi^{k-i-1}\\
    =\,&\alpha \phi^{k-1}v_0-\alpha v_k+\alpha\phi^{k-1}\left(1-\phi\right)\sum_{i=1}^{k-1} \frac{v_i}{\phi^{i}}\\
    \leq \,&\alpha \phi^{k-1}\frac{4 \bar{C} }{\ell_{m}^{2}(1 - \bar{\rho})(1-\gamma)^{2}}+\alpha \frac{4 \bar{C} }{\ell_{m}^{2}(1 - \bar{\rho})(1-\gamma)^{2}}+\frac{4 \bar{C} \alpha}{\ell_{m}^{2}(1 - \bar{\rho})(1-\gamma)^{2}}\phi^{k-1}\left(1-\phi\right)\sum_{i=1}^{k-1} \frac{1}{\phi^{i}}\\
    \leq \,&\frac{4 \bar{C} \alpha \phi^{k-1}}{\ell_{m}^{2}(1 - \bar{\rho})(1-\gamma)^{2}}+\frac{4 \bar{C} \alpha}{\ell_{m}^{2}(1 - \bar{\rho})(1-\gamma)^{2}}+\frac{4 \bar{C} \alpha}{\ell_{m}^{2}(1 - \bar{\rho})(1-\gamma)^{2}}\phi^{k-1}\\
    \leq \,&\frac{12 \bar{C} \alpha}{\ell_{m}^{2}(1 - \bar{\rho})(1-\gamma)^{2}}.
\end{align*}
Using the previous inequality in Eq. (\ref{eq:before_last}), we have
\begin{align*}
    \mathbb{E}[M(Q_k-Q^*)]\leq \,&\left[ 1 - \alpha\left(1-\frac{u_m}{\ell_m}\bar{\gamma}\right) \right]^k \mathbb{E}[M(Q_0-Q^*)]+\frac{12 \bar{C} \alpha}{\ell_{m}^{2}(1 - \bar{\rho})(1-\gamma)^{2}}\\
        &+ \frac{100\alpha^2}{\tau^2\ell_{m}^{2}\left(1-\frac{u_m}{\ell_m}\bar{\gamma}\right)^2(1-\gamma)^4}\left(\frac{\log(2\alpha(1-\bar{\rho})/[8\bar{C}\tau(1-\gamma)])}{\log(\bar{\rho})}\right)^4\\
        &+ \frac{6\bar{C} L (|\mathcal{S}||\mathcal{A}|)^{2/p} \alpha}{\left(1-\frac{u_m}{\ell_m}\bar{\gamma}\right)(1 - \bar{\rho})(1 - \gamma)^2}\\
        \leq \,&\left[ 1 - \alpha\left(1-\frac{u_m}{\ell_m}\bar{\gamma}\right) \right]^k \mathbb{E}[M(Q_0-Q^*)]\\
        &+ \frac{100\alpha^2}{\tau^2\ell_{m}^{2}\left(1-\frac{u_m}{\ell_m}\bar{\gamma}\right)^2(1-\gamma)^4}\left(\frac{\log(2\alpha(1-\bar{\rho})/[8\bar{C}\tau(1-\gamma)])}{\log(\bar{\rho})}\right)^4\\
        &+ \frac{6 \bar{C}(|\mathcal{S}||\mathcal{A}|)^{2/p} \alpha}{(1 - \bar{\rho})(1-\gamma)^{2}}\left(\frac{2}{\ell_{m}^{2}}+\frac{L}{\left(1-\frac{u_m}{\ell_m}\bar{\gamma}\right)}\right)
\end{align*}
To translate the above into a bound on $\mathbb{E}[\|Q_k - Q^*\|_\infty]$, using Lemma \ref{le:Moreau} (3), we have
\begin{align*}
    \mathbb{E}[\|Q_k-Q^*\|_\infty^2]\leq \,&\frac{u_m^2}{\ell_m^2}\left[ 1 - \alpha\left(1-\frac{u_m}{\ell_m}\bar{\gamma}\right) \right]^k \mathbb{E}[\|Q_0-Q^*\|_\infty^2]\\
        &+ \frac{200u_m^2\alpha^2}{\tau^2\ell_{m}^{2}\left(1-\frac{u_m}{\ell_m}\bar{\gamma}\right)^2(1-\gamma)^4}\left(\frac{\log(2\alpha(1-\bar{\rho})/[8\bar{C}\tau(1-\gamma)])}{\log(\bar{\rho})}\right)^4\\
        &+ \frac{12 \bar{C}(|\mathcal{S}||\mathcal{A}|)^{2/p} \alpha}{(1 - \bar{\rho})(1-\gamma)^{2}}\left(\frac{2u_m^2}{\ell_{m}^{2}}+\frac{Lu_m^2}{\left(1-\frac{u_m}{\ell_m}\bar{\gamma}\right)}\right).
\end{align*}
The final step of the proof is to make all constants in the convergence bound explicit. We begin by specifying the tunable parameters $\theta$ and $p$ used in defining the Lyapunov function $M(\cdot)$. By choosing $p = 2 \log(|\mathcal{S}||\mathcal{A}|)$ and $\theta = \left((1 + \bar{\gamma}) / 2\bar{\gamma}\right)^{2} - 1$, we have
\begin{align*}
    &(|\mathcal{S}||\mathcal{A}|)^{2/p}=e\leq 3,\;
    u_p=1,\;\ell_p=(|\mathcal{S}||\mathcal{A}|)^{-1/p}=\frac{1}{\sqrt{e}},\\
    &\frac{u_{m}^{2}}{\ell_{m}^{2}} = \frac{1 + \theta u_p^2}{1 + \theta \ell_p^2} = \frac{1 + \theta}{1 + \frac{\theta}{e}} 
    = \frac{e(1 + \theta)}{e + \theta} < e < 3,\\
    &u_{m}^{2} = (1 + \theta) = \left(\frac{1 + \bar{\gamma}}{2 \bar{\gamma}}\right)^{2} < \frac{1}{\bar{\gamma}^{2}} = \frac{1}{(1 - \lambda^{r_b}\delta_{b} \mu_{\pi_b,\min}(1-\gamma))^{2}}\leq 4,\\
    &\frac{u_m}{\ell_m}=\sqrt{\frac{e(1+\theta)}{e+\theta}}\leq \sqrt{1+\theta}=\frac{1 + \bar{\gamma}}{2\bar{\gamma}}\;\Rightarrow\;1 - \frac{u_{m}}{\ell_{m}}\hat{\gamma} \geq  \frac{1 - \bar{\gamma}}{2},\\
    &L=\frac{p-1}{\theta}\leq \frac{8\log(|\mathcal{S}||\mathcal{A}|)}{1-\bar{\gamma}}.
\end{align*}
Therefore, we have
\begin{align*}
    \mathbb{E}[\|Q_k-Q^*\|_\infty^2]\leq \,&3\left[ 1 - \alpha\left(\frac{1-\bar{\gamma}}{2}\right) \right]^k \mathbb{E}[\|Q_0-Q^*\|_\infty^2]+\frac{2520\bar{C}\log(|\mathcal{S}||\mathcal{A}|) \alpha}{(1 - \bar{\rho})(1-\gamma)^{2}(1-\bar{\gamma})^2}\\
        &+ \frac{2400\alpha^2}{\tau^2\left(1-\bar{\gamma}\right)^2(1-\gamma)^4}\left(\frac{\log(2\alpha(1-\bar{\rho})/[8\bar{C}\tau(1-\gamma)])}{\log(\bar{\rho})}\right)^4.
\end{align*}
Finally, since
\begin{align*}
        \bar{\gamma}=\,&1-\lambda^{r_b}\mu_{\pi_b,\min}\delta_b(1-\gamma),\quad \bar{C}=  \left(1 - \frac{1}{2}\delta_{b} \lambda^{r_{b}+1} \mu_{\pi_b,\min} \pi_{b,\min}\right)^{-1},\\       
        \bar{\rho} =\,&  \left(1 - \frac{1}{2}\delta_{b} \lambda^{r_{b}+1} \mu_{\pi_b,\min} \pi_{b,\min}\right)^{1/(r_{b}+1)}\Rightarrow 1-\bar{\rho}\geq \frac{\delta_{b} \lambda^{r_{b}+1} \mu_{\pi_b,\min} \pi_{b,\min}}{2(r_b+1)},
    \end{align*}
    where the last inequality follows from Bernoulli's inequality, we have
    \begin{align*}
    \mathbb{E}[\|Q_k-Q^*\|_\infty^2]\leq \,&3\left[ 1 - \alpha\left(\frac{\lambda^{r_b}\mu_{\pi_b,\min}\delta_b(1-\gamma)}{2}\right) \right]^k \mathbb{E}[\|Q_0-Q^*\|_\infty^2]\\
    &+\frac{10080(r_b+1)\log(|\mathcal{S}||\mathcal{A}|) \alpha}{ \lambda^{3r_{b}+1}\pi_{b,\min}\mu_{\pi_b,\min}^3\delta_b^3(1-\gamma)^4}\\
        &+ \frac{2400\alpha^2}{\tau^2\lambda^{2r_b}\mu_{\pi_b,\min}^2\delta_b^2(1-\gamma)^6}\left(\frac{(r_b+1)\log(8\bar{C}\tau(1-\gamma))/[4\alpha(1-\bar{\rho}])}{\delta_{b} \lambda^{r_{b}+1} \mu_{\pi_b,\min} \pi_{b,\min}}\right)^4\\
        \leq \,&3\left[ 1 - \alpha\left(\frac{\lambda^{r_b}\mu_{\pi_b,\min}\delta_b(1-\gamma)}{2}\right) \right]^k \mathbb{E}[\|Q_0-Q^*\|_\infty^2]\\
    &+\frac{10080(r_b+1)\log(|\mathcal{S}||\mathcal{A}|) \alpha}{ \lambda^{3r_{b}+1}\pi_{b,\min}\mu_{\pi_b,\min}^3\delta_b^3(1-\gamma)^4}\\
        &+ \frac{38400(r_b+1)^4\alpha^2}{\tau^2\lambda^{6r_b+4}\mu_{\pi_b,\min}^6\pi_{b,\min}^4\delta_b^6(1-\gamma)^6}\log^4\left(\frac{4(r_b+1)}{\alpha \delta_{b} \lambda^{r_{b}+1} \mu_{\pi_b,\min} \pi_{b,\min}}\right).
\end{align*}
The final result follows from using the definitions of $c_1$, $c_2$, $c_3$, and $c_4$ to simplify the notation.

\subsection{Auxiliary Lemma}

\begin{lemma}\label{le:mu-pi-lipschitz}
    For $\pi_{1}, \pi_{2} \in \Pi$, we have
    \begin{align*}
        \|\Bar{\mu}_{\pi_{1}} - \Bar{\mu}_{\pi_{2}}\|_1\leq 2\left(\frac{\log(\frac{\|\pi_1-\pi_2\|_\infty}{4\bar{C}_c})}{\log(\bar{\rho}_c)}\right)\|\pi_1-\pi_2\|_\infty.
    \end{align*}
\end{lemma}
\begin{proof}[Proof of Lemma \ref{le:mu-pi-lipschitz}]
    Similar results establishing the continuous dependence of the stationary distributions on the policies have been previously obtained in \cite{chen2023finite} and \cite{zhang2023global}, but in different contexts 
and with respect to different norms. We reproduce the proofs for our setting with respect to $\ell_{\infty}$-norm. 

Let $\bar{M}_{\pi_{1}} \in \mathbb{R}^{|\mathcal{S}||\mathcal{A}|\times |\mathcal{S}||\mathcal{A}|}$ be the matrix with $\Bar{\mu}_{\pi_{1}}^{\top}$ as every row. 
Since $\Bar{\mu}_{\pi_{1}}^{\top} = \Bar{\mu}_{\pi_{1}}^{\top} \bar{\mathcal{P}}_{\pi_{1}}^{k}$ and $\Bar{\mu}_{\pi_{2}}^{\top} = \Bar{\mu}_{\pi_{2}}^{\top} \bar{\mathcal{P}}_{\pi_{2}}^{k}$ for any $k\geq 0$, we have
\begin{align}
    \|\Bar{\mu}_{\pi_{1}} - \Bar{\mu}_{\pi_{2}}\|_{1} 
    & = \|(\bar{\mathcal{P}}_{\pi_{1}}^{k})^{\top}\Bar{\mu}_{\pi_{1}} - (\bar{\mathcal{P}}_{\pi_{2}}^{k})^{\top}\Bar{\mu}_{\pi_{2}} \|_{1} \nonumber\\ 
    & \leq \|(\bar{\mathcal{P}}_{\pi_{1}}^{k})^{\top}(\Bar{\mu}_{\pi_{1}} - \Bar{\mu}_{\pi_{2}})\|_{1} + \|(\bar{\mathcal{P}}_{\pi_{1}}^{k} - \bar{\mathcal{P}}_{\pi_{2}}^{k})^{\top} \Bar{\mu}_{\pi_{2}}\|_{1} \nonumber\\
    & = \|(\bar{\mathcal{P}}_{\pi_{1}}^{k} - \bar{M}_{\pi_{1}} + \bar{M}_{\pi_{1}})^{\top}(\Bar{\mu}_{\pi_{1}} - \Bar{\mu}_{\pi_{2}})\|_{1} + \|(\bar{\mathcal{P}}_{\pi_{1}}^{k} - \bar{\mathcal{P}}_{\pi_{2}}^{k})^{\top} \Bar{\mu}_{\pi_{2}}\|_{1} \nonumber\\
    & \leq \|(\bar{\mathcal{P}}_{\pi_{1}}^{k} - \bar{M}_{\pi_{1}})^{\top}(\Bar{\mu}_{\pi_{1}} - \Bar{\mu}_{\pi_{2}})\|_{1} + \|\bar{M}_{\pi_{1}}^{\top}(\Bar{\mu}_{\pi_{1}} - \Bar{\mu}_{\pi_{2}})\|_{1} + \|(\bar{\mathcal{P}}_{\pi_{1}}^{k} - \bar{\mathcal{P}}_{\pi_{2}}^{k})^{\top} \Bar{\mu}_{\pi_{2}}\|_{1} \nonumber\\
    & \leq \|(\bar{\mathcal{P}}_{\pi_{1}}^{k} - \bar{M}_{\pi_{1}})^{\top}\|_{1}\|\Bar{\mu}_{\pi_{1}} - \Bar{\mu}_{\pi_{2}}\|_{1} + \|\bar{M}_{\pi_{1}}^{\top}(\Bar{\mu}_{\pi_{1}} - \Bar{\mu}_{\pi_{2}})\|_{1} + \|(\bar{\mathcal{P}}_{\pi_{1}}^{k} - \bar{\mathcal{P}}_{\pi_{2}}^{k})^{\top}\|_{1} \|\Bar{\mu}_{\pi_{2}}\|_{1} \nonumber\\
    & \leq  2\|\bar{\mathcal{P}}_{\pi_{1}}^{k} - \bar{M}_{\pi_{1}}\|_{\infty} + \|\bar{M}_{\pi_{1}}^{\top}(\Bar{\mu}_{\pi_{1}} - \Bar{\mu}_{\pi_{2}})\|_{1} + \|\bar{\mathcal{P}}_{\pi_{1}}^{k} - \bar{\mathcal{P}}_{\pi_{2}}^{k}\|_{\infty}.\label{eq1:pf:le:mu-pi-lipschitz}
\end{align}
To proceed, observe that
\begin{align}
    \|\bar{\mathcal{P}}_{\pi_{1}}^{k} - \bar{M}_{\pi_{1}}\|_{\infty}
    =\,&\max_{s,a}\sum_{s',a'}|\bar{\mathcal{P}}_{\pi_{1}}^{k}((s,a),(s',a'))-\bar{\mu}_{\pi_1}(s',a')|\nonumber\\
    =\,&2\max_{s,a}\|\Bar{\mathcal{P}}_{\pi_{1}}^{k}((s,a),(\cdot,\cdot)) - \bar{\mu}_{\pi_{1}}(\cdot,\cdot)\|_{\text{TV}}\nonumber\\
    \leq \,&2\bar{C}_1 \bar{\rho}_1^{k},\quad \forall\;k \geq 0.\label{eq2:pf:le:mu-pi-lipschitz}
\end{align}
Moreover, we have
\begin{align}
    \bar{M}_{\pi_{1}}^{\top}(\Bar{\mu}_{\pi_{1}} - \Bar{\mu}_{\pi_{2}})
    =\Bar{\mu}_{\pi_{1}} \mathbf{1}^\top (\Bar{\mu}_{\pi_{1}} - \Bar{\mu}_{\pi_{2}})
    =\Bar{\mu}_{\pi_{1}}-\Bar{\mu}_{\pi_{1}}=0.\label{eq3:pf:le:mu-pi-lipschitz}
\end{align}
and
\begin{align}
    \|\bar{\mathcal{P}}_{\pi_{1}}^{k} - \bar{\mathcal{P}}_{\pi_{2}}^{k}\|_{\infty}
    \leq \,&k\|\bar{\mathcal{P}}_{\pi_{1}} - \bar{\mathcal{P}}_{\pi_{2}}\|_{\infty}\nonumber\\
    =\,&k\max_{s,a}\sum_{s',a'}p(s'|s,a)|\pi_1(a'|s')-\pi_2(a'|s')|\nonumber\\
    \leq \,&k\max_{s'}\sum_{a'}|\pi_1(a'|s')-\pi_2(a'|s')|\nonumber\\
    = \,&k\|\pi_1-\pi_2\|_\infty,\label{eq4:pf:le:mu-pi-lipschitz}
\end{align}
which follows from the same analysis as in the proof of Proposition \ref{prop:MC} (2). Using the inequalities obtained in Eqs. \eqref{eq2:pf:le:mu-pi-lipschitz}, \eqref{eq3:pf:le:mu-pi-lipschitz}, and \eqref{eq4:pf:le:mu-pi-lipschitz} together in Eq. (\ref{eq1:pf:le:mu-pi-lipschitz}), we have
\begin{align*}
    \|\Bar{\mu}_{\pi_{1}} - \Bar{\mu}_{\pi_{2}}\|_1\leq\,& 4 \bar{C}_1 \bar{\rho}_1^k + k\|\pi_1-\pi_2\|_\infty\\
    \leq \,&4 \bar{C}_1 k\bar{\rho}_1^k + k\|\pi_1-\pi_2\|_\infty,\quad \forall\,k\geq 0.
\end{align*}
The final result follows from choosing
\begin{align*}
    k=\frac{\log(\frac{\|\pi_1-\pi_2\|_\infty}{4\bar{C}_c})}{\log(\bar{\rho}_c)}.
\end{align*}
\end{proof}

\end{document}